\documentclass{article}

\usepackage{microtype}
\usepackage{graphicx}
\usepackage{subfigure}
\usepackage{booktabs} %
\usepackage{xcolor}
\usepackage{amsthm}
\usepackage{multirow}
\usepackage{afterpage}
\usepackage{placeins}
\usepackage{wrapfig}
\usepackage{fancyvrb}
\usepackage{comment}

\usepackage{hyperref}

\usepackage{amsmath,amsfonts,bm,amssymb,mathtools,amsthm}

\newtheorem{theorem}{Theorem}
\newtheorem{definition}{Definition}

\newtheorem{lemma}{Lemma}
\newtheorem{assumption}{Assumption}
\newtheorem{example}{Example}

\newcommand{\bb}[1]{{\mathbb{#1}}}

\def\eqref#1{equation~\ref{#1}}

\def\1{\bm{1}}

\newcommand{\defeq}{\vcentcolon=}

\def\vmu{{\bm{\mu}}}
\def\vtheta{{\bm{\theta}}}
\def\va{{\bm{a}}}
\def\vb{{\bm{b}}}

\def\vx{{\bm{x}}}
\def\vy{{\bm{y}}}
\def\vz{{\bm{z}}}

\def\vomega{{\bm{\omega}}}

\def\vtheta{{\bm{\theta}}}
\def\vphi{{\bm{\phi}}}
\def\vmu{{\bm{\mu}}}
\def\Hw{{H_\vomega(\vx)}}
\def\Et{{E_\vtheta(\vx)}}
\def\HE{{\Hw + \Et}}
\newcommand{\at}[2][]{#1|_{#2}}
\def\otheta{{\vtheta^*}}
\def\oomega{{\vomega^*}}
\def\ophi{{\vphi^*}}
\def\vJ{{\bm{J}}}
\def\vK{{\bm{K}}}
\def\vP{{\bm{P}}}
\def\vQ{{\bm{Q}}}
\def\vS{{\bm{S}}}
\def\vU{{\bm{U}}}

\def\vI{{\bm{I}}}

\DeclareMathAlphabet{\mathsfit}{\encodingdefault}{\sfdefault}{m}{sl}
\SetMathAlphabet{\mathsfit}{bold}{\encodingdefault}{\sfdefault}{bx}{n}

\def\gD{{\mathcal{D}}}

\def\gH{{\mathcal{H}}}

\def\gL{{\mathcal{L}}}

\def\gT{{\mathcal{T}}}

\def\gX{{\mathcal{X}}}

\def\gZ{{\mathcal{Z}}}

\newcommand{\R}{\mathbb{R}}

\newcommand{\DF}{D_f}

\DeclareMathOperator*{\argmax}{arg\,max}

\usepackage[accepted]{icml2020}

\icmltitlerunning{$f$-EBM: Training EBMs with $f$-Divergence Minimization}

\begin{document}

\twocolumn[
\icmltitle{%
Training Deep Energy-Based Models with $f$-Divergence Minimization
}

\begin{icmlauthorlist}
\icmlauthor{Lantao Yu}{st}
\icmlauthor{Yang Song}{st}
\icmlauthor{Jiaming Song}{st}
\icmlauthor{Stefano Ermon}{st}
\end{icmlauthorlist}

\icmlaffiliation{st}{Department of Computer Science, Stanford University, Stanford, CA 94305 USA}

\icmlkeywords{Machine Learning, ICML}

\vskip 0.3in
]

\icmlcorrespondingauthor{Lantao Yu}{lantaoyu@cs.stanford.edu}
\icmlcorrespondingauthor{Stefano Ermon}{ermon@cs.stanford.edu}
\printAffiliationsAndNotice{}  %

\begin{abstract}
Deep energy-based models (EBMs) are very flexible in distribution parametrization but computationally challenging because of the intractable partition function. They are typically trained via maximum likelihood, using contrastive divergence to approximate the gradient of the KL divergence between data and model distribution. While KL divergence has many desirable properties, other $f$-divergences have shown advantages in training implicit density generative models such as generative adversarial networks. In this paper, we propose a general variational framework termed $f$-EBM to train EBMs using any desired $f$-divergence. We introduce a corresponding optimization algorithm and prove its local convergence property with non-linear dynamical systems theory. Experimental results demonstrate the superiority of $f$-EBM over contrastive divergence, as well as the benefits of training EBMs using $f$-divergences other than KL.

\comment{
Deep energy-based models (EBMs) have received increasing attention due to their 
flexibility in distribution parametrization. Because of the intractable partition function, 
contrastive divergence has been the predominant approach for training EBMs, which essentially estimates the gradients of KL divergence between data and model distribution through MCMC sampling. 
However, in various application scenarios, users may have different preferences such as sharpness vs. diversity. With KL divergence being a special case, $f$-divergences provide a general family of discrepancy measures between two distributions, which have been extensively explored in conjunction with GANs for training implicit density generative models. Since explicit density generative models offer many advantages, there is an urgent need to investigate how to use other $f$-divergences to train EBMs. In this paper, we propose a general variational framework termed $f$-EBM, which enables us to use any $f$-divergence to train unnormalized EBMs. We further theoretically prove the local convergence property of the proposed algorithm with the theory of non-linear dynamical systems. Experimental results demonstrate the superiority of $f$-EBM over contrastive divergence, as well as the characteristics and benefits of different $f$-divergences for training EBMs.
}
\end{abstract}

\section{Introduction}
Learning deep generative models that can approximate complex distributions over high-dimensional data is an important problem in machine learning, with many applications such as image, speech, natural language generation \cite{radford2015unsupervised,oord2016wavenet,yu2017seqgan} and imitation learning~\cite{ho2016generative}.
To this end, two major branches of generative models have been widely studied: tractable density and implicit density generative models. To enable the use of maximum likelihood training, tractable density models have to use specialized architectures to build a normalized probability model. These include autoregressive models~\cite{larochelle2011neural,germain2015made,oord2016pixel,van2016conditional}, flow-based models~\cite{dinh2014nice,dinh2016density,kingma2018glow} and sum-product networks~\cite{poon2011sum}.
However, such a normalization requirement can hinder the flexibility and expressiveness. For example, 
flow-based models rely on invertible transformations with tractable Jacobian determinant, while sum-product networks rely on special graph structures (with sums and products as internal nodes) to obtain tractable density.
To overcome the limitations caused by the constraint of specifying a normalized explicit density, Generative Adversarial Networks (GANs) \cite{goodfellow2014generative} proposed to learn implicit density generative models within a minimax framework, where the generative model is a direct differentiable mapping from noise space to sample space. So far many variants of GANs have been proposed in order to minimize various discrepancy measures \cite{nowozin2016f,arjovsky2017wasserstein}. In particular, $f$-GANs \cite{nowozin2016f} minimize a variational representation of $f$-divergence between data and model distribution, which is a general family of divergences for measuring the discrepancies between two distributions.
Although appealing, a main drawback of GANs is the lack of explicit density, which is the central goal of density estimation and plays an indispensable role in applications such as anomaly detection, image denoising and inpainting. 

Deep energy-based models (EBMs) are promising to combine the best of both worlds, in the sense that EBMs allow both extremely flexible distribution parametrization and the access to an (unnormalized) explicit density. 
However, because of the intractable partition function (an integral over the sample space), so far within the family of $f$-divergences, only KL divergence proved to be tractable to optimize with methods such as contrastive divergence \cite{hinton2002training}, doubly dual embedding \cite{dai2018kernel} and adversarial dynamics embedding \cite{dai2019exponential}.
Since the progress of generative modeling research was mainly driven by the study on the properties and tractable optimization of various discrepancy measures, there is an urgent need to investigate the possibility of training EBMs with other $f$-divergences. Specifically, the choice of discrepancy measures embodies our preferences and has a significant influence on the learned model distribution (see Figure~\ref{fig:divergences} for illustration of the effect of some popular $f$-divergences on training EBMs in a model misspecification scenario). 
For example, it has been found that minimizing KL (Maximum Likelihood Estimation) is not directly correlated with sample quality \cite{theis2015note}, and in many application scenarios we need more flexibility in trading off mode collapse vs. mode coverage.
In this paper, we propose a new variational framework termed $f$-EBM to enable the use of any $f$-divergence for training EBMs. Our framework also naturally produces a density ratio estimation, which can be used for importance sampling and bias correction \cite{grover2019bias}. For example, we show that in conjunction with rejection sampling, the learned density ratio can be used to recover the data distribution even in a model misspecification scenario (see Figure~\ref{fig:divergences_rs} for illustration), which traditional methods for minimizing KL divergence such as contrastive divergence cannot do. Furthermore, with the theory of non-linear dynamical systems \cite{hassan1996khalil}, we establish a rigorous proof on the local convergence property of the proposed single-step gradient optimization algorithm for $f$-EBM. Finally, experimental results demonstrate the benefits of using various $f$-divergences to train EBMs, and more importantly, with some members in the $f$-divergence family as the training objectives (\emph{e.g.}, Jensen-Shannon, Squared Hellinger and Reverse KL), we are able to achieve significant sample quality improvement compared to contrastive divergence method on some commonly used image datasets.

\begin{figure*}[t]
\centering
\subfigure[KL]{
\includegraphics[height=1.2in]{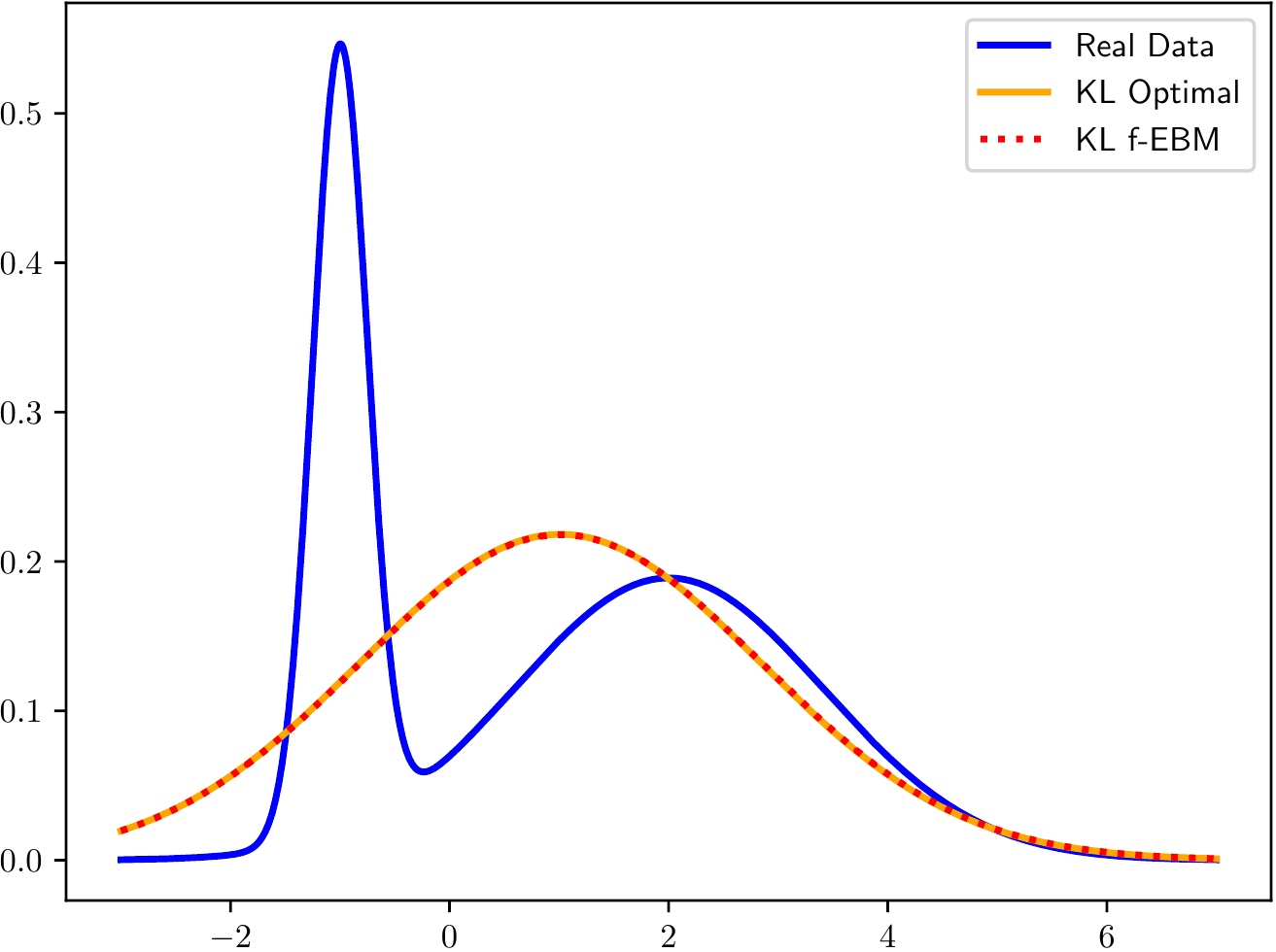}\label{fig:kl}}
\subfigure[Reverse KL]{
\includegraphics[height=1.2in]{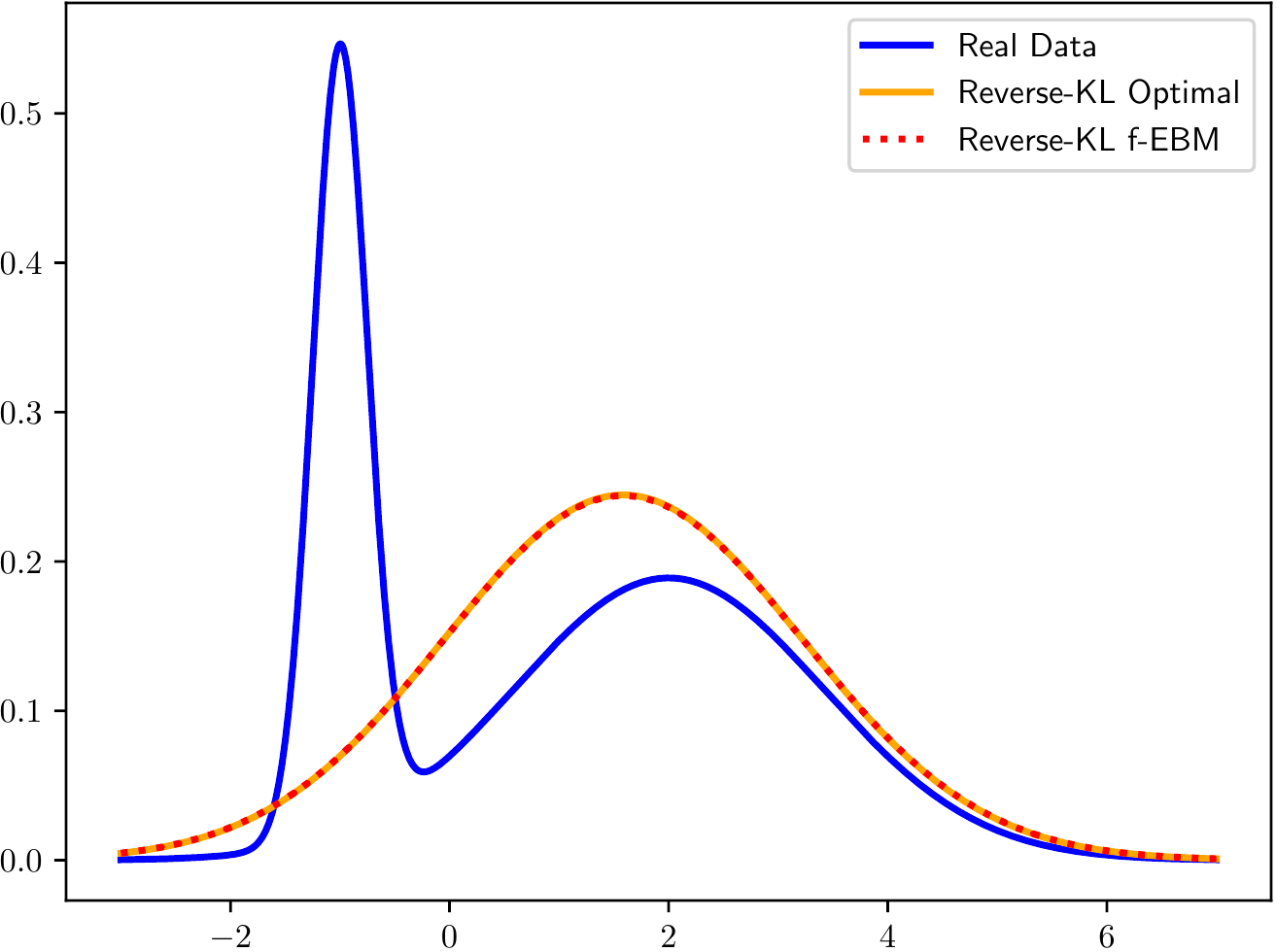}\label{fig:rev-kl}}
\subfigure[Jensen-Shannon]{
\includegraphics[height=1.2in]{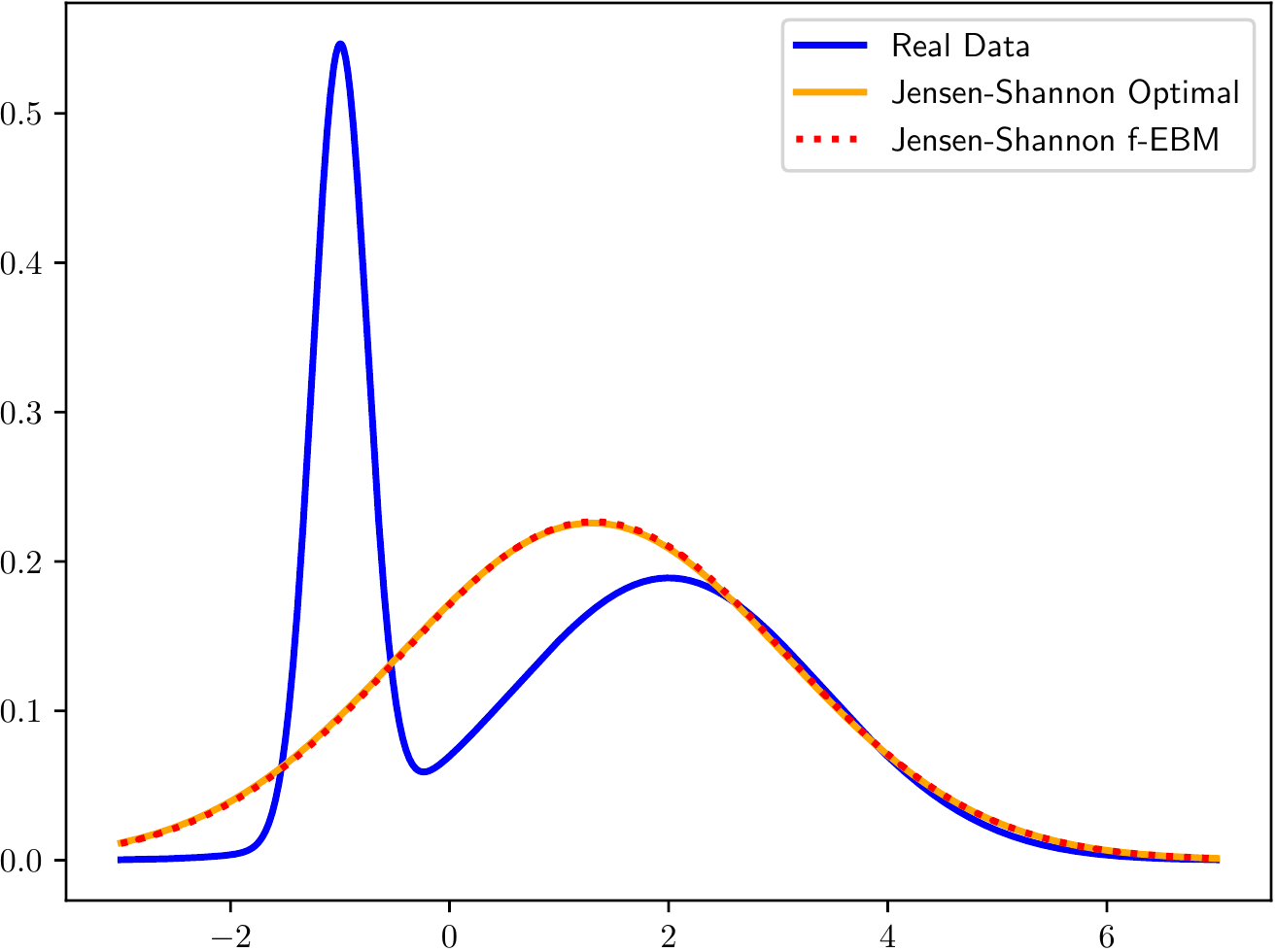}\label{fig:jensen-shannon}}
\subfigure[$\alpha$-Divergence ($\alpha=-1$)]{
\includegraphics[height=1.2in]{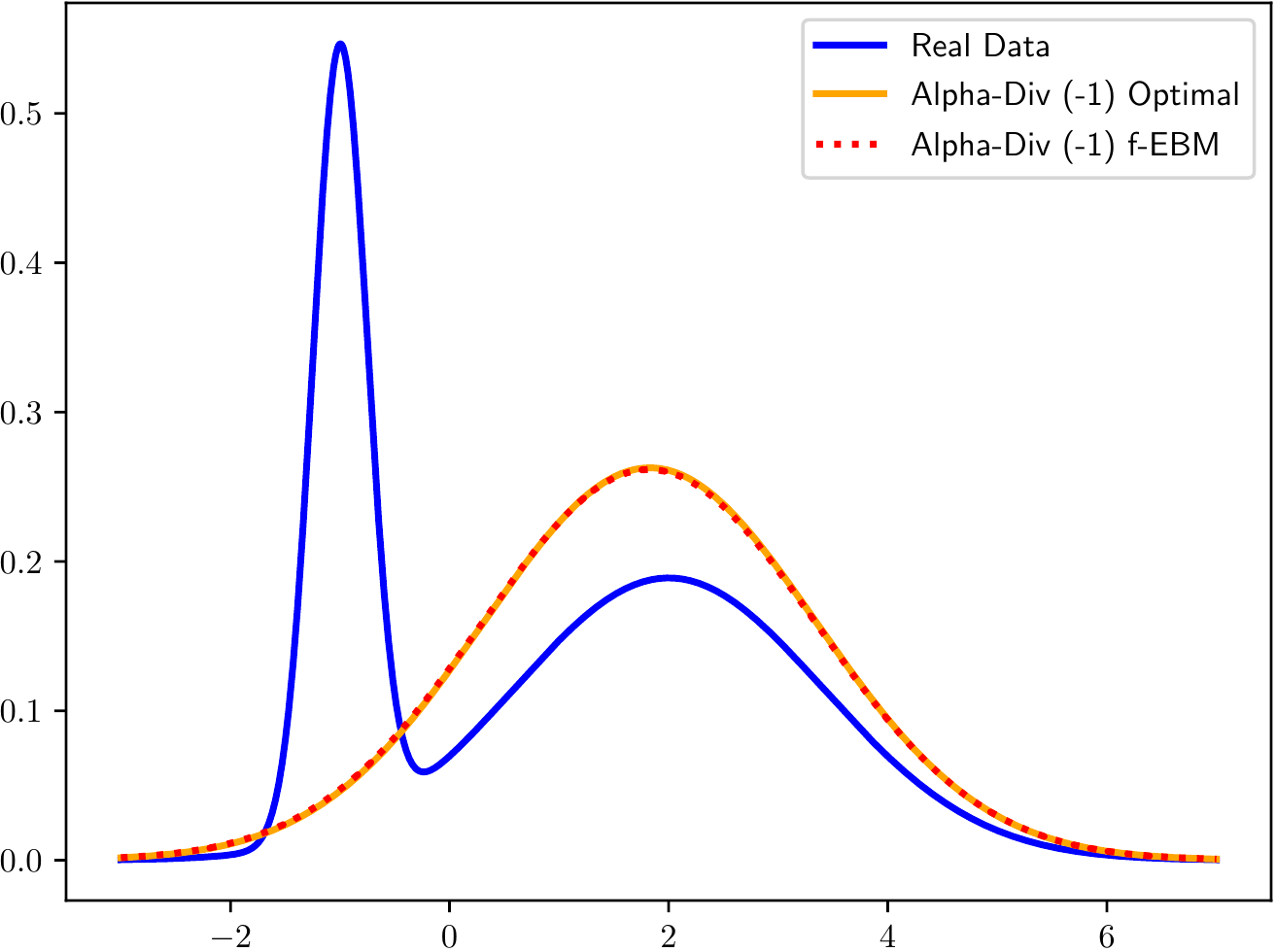}\label{fig:alpha-div}}
\vspace{-8pt}
\caption{The influences of different $f$-divergences on the training of EBMs. The blue solid line represents the real data distribution; the orange solid line represents the optimal model distribution under a certain discrepancy measure obtained by directly solving the minimization problem with the real data density; the red dashed line represents the model distribution learned by $f$-EBM.}
\vspace{-8pt}
\label{fig:divergences}
\end{figure*}
\section{Preliminaries}

\subsection{The $f$-Divergence Family}\label{sec:f-div-family}
Let $P$ and $Q$ denote two probability distributions with density functions $p$ and $q$ with respect to a base measure $\mathrm{d} \vx$ on domain $\mathcal{X}$. Suppose $P$ is absolutely continuous with respect to $Q$, denoted as $P \ll Q$ (\emph{i.e.}, the Radon-Nikodym derivative $\mathrm{d} P/\mathrm{d} Q$ exists). For any convex, lower-semicontinuous function $f: [0, +\infty) \to \mathbb{R}$ satisfying $f(1)=0$, the $f$-divergence between $P$ and $Q$ is defined as:
\begin{align}
    \DF(P\|Q) \defeq \int_\gX q(\vx) f\left(\frac{p(\vx)}{q(\vx)}\right) \mathrm{d} \vx
    = \mathbb{E}_{q(\vx)} \left[f\left(\frac{p(\vx)}{q(\vx)}\right)\right]. \nonumber
\end{align}
Many divergences are special cases obtained by choosing a suitable \emph{generator function} $f$. For example, $f(u) = u \log u$ and $f(u)=-\log u$ correspond to forward KL and reverse KL respectively (see Table 5 in \cite{nowozin2016f} for more examples). In the rest of this paper, we will consider the generator function $f$ to be strictly convex and continuously differentiable, and we will use $f'$ to denote the derivative of $f$.

\begin{definition}[Fenchel Duality]
For any convex, lower-semicontinuous function $f: \gX \to \bb{R} \cup \{+ \infty\}$ defined over a Banach space $\gX$, the Fenchel conjugate function of $f$, $f^*$ 
is defined over the dual space $\gX^*$ as:
\begin{align}
    f^*(\vx^*) \defeq \sup_{\vx \in \gX} \langle \vx^*, \vx \rangle - f(\vx),
\end{align}
where $\langle \cdot, \cdot \rangle$ is the duality paring between $\gX$ and $\gX^*$. For a finite dimensional space $\bb{R}^m$, the dual space is also $\bb{R}^m$ and the duality paring is the usual vector inner product.
\end{definition}
 The function $f^*$ is also convex and lower-semicontinuous and more importantly, we have $f^{**} = f$, which means we can represent $f$ through its conjugate function as:
\begin{align}
    f(\vx) = \sup_{\vx^* \in {\gX}^*} \langle \vx^*, \vx \rangle - f^*(\vx^*).
\end{align}
Based on Fenchel duality, \citet{nguyen2010estimating} proposed a general variational representation of $f$-divergences:

\begin{definition}
Define the subdifferential $\partial f(\vx)$ of a convex function $f: \bb{R}^m \to \bb{R}$ at a point $\vx \in \bb{R}^m$ as the set:
$
    \partial f(\vx) \defeq \{\vz \in \bb{R}^m | f(\vy) \geq f(\vx) + \langle \vz, \vy - \vx \rangle, \forall \vy \in \bb{R}^m\}
$. 
When $f$ is differentiable at $\vx$, $\partial f(\vx) = \{\nabla_\vx f(\vx)\}$. 
\end{definition}

\begin{lemma}[\citet{nguyen2010estimating}]\label{lemma:f-bound}
Let $P$ and $Q$ be two probability measures over the Borel $\sigma$-algebra on domain $\gX$ with densities $p$, $q$ and $P \ll Q$.
For any class of functions 
$\gT =\{T:\gX \rightarrow \bb{R}\}$
such that the subdifferential $\partial f(p/q)$ contains an element of $\gT$, 
we have
\begin{align}
    \DF(P \| Q) = \sup_{T \in \gT} \bb{E}_{p(\vx)} [T(\vx)] - \bb{E}_{q(\vx)}[f^*(T(\vx))],\label{eq:f-bound}
\end{align}
where the supreme is attained at $T^*(\vx) = f'\left(\frac{p(\vx)}{q(\vx)}\right)$.
\end{lemma}

\subsection{Energy-Based Generative Modeling}\label{sec:ebm-intro}
Suppose we are given a dataset consisting of \emph{i.i.d.} samples $\{\vx_i \}_{i=1}^N$ from an unknown data distribution $p(\vx)$, defined over a sample space $\gX \subset \mathbb{R}^m$. Our goal is to find a parametric approximation $q_\vtheta(\vx)$ to the data distribution, such that we can perform downstream inferences (\emph{e.g.}, density evaluation and  sampling). Specifically, we are interested in learning an (unnormalized) energy function $E_\vtheta(\vx)$ that defines the following normalized distribution:
\begin{align}
    q_\vtheta(\vx) = \exp(-E_\vtheta(\vx))/Z_{\vtheta},
\end{align}
where $Z_\vtheta = \int_\gX \exp(-E_\vtheta(\vx)) \mathrm{d} \vx$ 
is the normalization constant (also called the partition function). In this paper, unless otherwise stated, we will always use $p$ to denote the data distribution and $q_\vtheta$ to denote the model distribution. Moreover, we will assume $\exp(-E_\vtheta(\vx))$ is integrable over sample space $\gX$, \emph{i.e.}, $Z_\vtheta$ is finite. Because of the absence of the normalization requirement, energy-based distribution representation offers extreme modeling flexibility in the sense that we can use almost any model architecture that outputs a real number given an input as the energy function. 

While flexible, inference in EBMs is challenging because of the partition function, which is generally intractable to compute exactly. For example, it is non-trivial to sample from an EBM, which usually requires Markov Chain Monte Carlo (MCMC) \cite{robert2013monte} techniques.
As a gradient based MCMC method, Langevin dynamics \cite{neal2011mcmc,welling2011bayesian}
defines an efficient iterative sampling process, which asymptotically can produce samples from an energy-based distribution:
\begin{align}
    \Tilde{\vx}_t = \Tilde{\vx}_{t-1} - \frac{\epsilon}{2} \nabla_\vx E_\vtheta(\Tilde{\vx}_{t-1}) + \sqrt{\epsilon} \vz_t, \label{eq:langevin-dynamics}
\end{align}
where $\vz_t \sim \mathcal{N}(0, I)$. The distribution of $\Tilde{\vx}_T$ converges to the model distribution $q_\vtheta(\vx) \propto \exp(- E_\vtheta(\vx))$ when $\epsilon \to 0$ and $T \to \infty$ under some regularity conditions \cite{welling2011bayesian}. 
Although few bounds on the mixing time are known, in practice researchers found that using a relatively small $\epsilon$ and finite $T$ suffices to produce good samples and
so far various scalable techniques have been developed for the purpose of sampling from an EBM efficiently. 
For example, 
\citet{du2019implicit} proposed to use Langevin dynamics together with a sample replay buffer to reduce mixing time and improve sample diversity, which is a sampling strategy we employ in the experiments.

\comment{
\subsection{Training Implicit Density Generative Models with Variational $f$-Divergence Minimization}\label{sec:variational-f-div}
For the purpose of training implicit density generative models (also known as generative neural samplers which are deterministic and differentiable mappings from noise space to sample space)
, \citet{nowozin2016f} proposed the following $f$-GAN objective to minimize the variational representation of $\DF(P\|Q_\vtheta)$ in Equation~(\ref{eq:f-bound}):
\begin{align}
    \min_\vtheta \max_\vomega \bb{E}_{p(\vx)}[T_\vomega(\vx)] - \bb{E}_{p(\vz)}[f^*(T_\vomega(G_\vtheta(\vz)))],\label{eq:f-gan}
\end{align}
where $p(\vz)$ is a fixed noise distribution (typically Gaussian distribution); $G_\vtheta: \gZ \to \gX$ is a direct differentiable mapping from noise space to sample space and $T_\vomega: \gX \to \bb{R}$ is the variational function. Note that the original GAN objective \cite{goodfellow2014generative} is also a special case with $f(u) = u \log u - (u+1) \log (u + 1)$.
}

\subsection{Tackling Intractable Partition Function with Contrastive Divergence}\label{sec:preliminary-cd}
The predominant approach to training explicit density generative models is to minimize the KL divergence between the (empirical) data distribution and model distribution, which is a specific member within the $f$-divergence family. Minimizing KL is equivalent to the following maximum likelihood estimation (MLE) objective:
\begin{align}
    \min_\vtheta \gL_{\mathrm{MLE}}(\vtheta; p) = \min_\vtheta -\mathbb{E}_{p(\vx)}\left[\log q_\vtheta(\vx)\right]\label{eq:mle}
\end{align}

\begin{lemma}\label{lemma:cd}
Given a $\vtheta$-parametrized energy-based distribution $q_\vtheta(\vx) \propto \exp(- E_\vtheta(\vx))$. Suppose both $\exp(- E_\vtheta(\vx))$ and its partial derivative $\nabla_\vtheta \exp(- E_\vtheta(\vx))$ are continuous w.r.t. $\vtheta$ and $\vx$. With Leibniz integral rule, $\forall \vx \in \gX$, we have:
\vspace{-10pt}
\begin{equation}
    \nabla_\vtheta \log q_\vtheta(\vx) = - \nabla_\vtheta E_\vtheta(\vx) + \bb{E}_{q_\vtheta(\vx)} [\nabla_\vtheta E_\vtheta(\vx)].
    \label{eq:gradloglik}
\end{equation}
\end{lemma}
\begin{proof}
See \citep{turner2005cd} for derivations.
\end{proof}

Because of the intractable partition function, we do not have access to the normalized density $q_\vtheta$ and cannot directly minimize the MLE objective. Fortunately, we can use Lemma~\ref{lemma:cd} to estimate the gradient 
of $\gL_{\mathrm{MLE}}(\vtheta; p)$,
, which gives rise to the contrastive divergence method \cite{hinton2002training}:
\begin{align}
    \nabla_\vtheta \gL_{\mathrm{MLE}}(\vtheta; p) &= - \bb{E}_{p(\vx)} \nabla_\vtheta \log q_\vtheta(\vx) \label{eq:contrastive-divergence}\\
    &= \bb{E}_{p(\vx)}[\nabla_\vtheta E_\vtheta(\vx)] - \bb{E}_{q_\vtheta(\vx)}[\nabla_\vtheta E_\vtheta(\vx)]. \nonumber
\end{align}
Intuitively this amounts to decreasing the energies of ``positive'' samples from $p$ and increasing the energies of ``negative samples'' from $q_\vtheta$.
As mentioned in Section~\ref{sec:ebm-intro}, MCMC techniques such as Langevin dynamics are needed in order to sample from $q_\vtheta$, which gives rise to the following surrogate gradient estimation:
\begin{align}
    \nabla_\vtheta \gL_{\mathrm{CD-}K}(\vtheta; p) = \bb{E}_{p(\vx)}[\nabla_\vtheta E_\vtheta(\vx)] - \bb{E}_{q_\vtheta^K(\vx)}[\nabla_\vtheta E_\vtheta(\vx)],\nonumber
\end{align}
where $q_\vtheta^K$ denotes the distribution after $K$ steps of MCMC transitions from an initial distribution (typically data distribution or uniform distribution), and Equation~(\ref{eq:contrastive-divergence}) corresponds to $\gL_{\mathrm{CD-}\infty}$. 

\section{Method}
In this section, we consider a more general training objective, which is minimizing the $f$-divergence between the (empirical) data distribution and a $\vtheta$-parametrized energy-based distribution:
\begin{align}
    \min_\vtheta \DF(p\|q_\vtheta) = \min_\vtheta \mathbb{E}_{q_\vtheta(\vx)} \left[f\left(\frac{p(\vx)}{q_\vtheta(\vx)}\right)\right]
\end{align}

\subsection{Challenges of Training EBMs with $f$-Divergences}\label{sec:challenges}
\subsubsection{Challenges of The Primal Form}
The reason for the predominance of KL divergence for training EBMs is the convenience and tractability of gradient estimation. Specifically, as discussed in Section~\ref{sec:preliminary-cd}, the gradient of KL divergence reduces to the expected gradient of the log-likelihood, which can be estimated using Lemma~\ref{lemma:cd}.
However, this is not generally applicable to other divergences, and as far as we know, within the $f$-divergence family, KL divergence is the only one that permits such a convenient form. 
For example, reverse KL is another member in the $f$-divergence family with a simple form:
\begin{align}
    D_{\mathrm{RevKL}}(p \| q_\vtheta) = \bb{E}_{q_\vtheta(\vx)} \left[\log \frac{q_\vtheta(\vx)}{p(\vx)}\right].
\end{align}
The gradient can be calculated as:
\begin{align}
    & \nabla_\vtheta D_{\mathrm{RevKL}}(p \| q_\vtheta) \nonumber\\
    =&~\bb{E}_{q_\vtheta(\vx)} \left[\nabla_\vtheta \log q_\vtheta(\vx) \left(\log \frac{q_\vtheta(\vx)}{p(\vx)} + 1\right)\right].\label{eq:revkl-grad}
\end{align}
Although we can still use Equation~(\ref{eq:gradloglik}) in Lemma~\ref{lemma:cd} to evaluate the gradient of the log-likelihood term ($\nabla_\vtheta \log q_\vtheta(\vx)$) in Equation~(\ref{eq:revkl-grad}), it is still infeasible to estimate because we do not know the density ratio $q_\vtheta(\vx) / p(\vx)$ (recall that we only have samples from $p$ and $q_\vtheta$). 

\subsubsection{Challenges of The Dual Form}\label{sec:challenge-f-gan}
The variational characterization from Lemma~\ref{lemma:f-bound} provides us a way to estimate an $f$-divergence only using samples from $p$ and $q_\vtheta$. Therefore we can instead minimize the variational representation of $f$-divergence by solving the following minimax problem:
\begin{align}
    \min_\vtheta \max_\vomega \bb{E}_{p(\vx)}[T_\vomega(\vx)] - \bb{E}_{q_\vtheta(\vx)}[f^*(T_\vomega(\vx))],\label{eq:f-gan}
\end{align}
where $T_\vomega: \gX \to \bb{R}$ is the variational function.
When $q_\vtheta$ is an implicit density generative model defined by a fixed noise distribution $\pi(\vz)$ and
a direct differentiable mapping $G_\vtheta: \gZ \to \gX$ from noise space to sample space, Equation~(\ref{eq:f-gan}) corresponds to the $f$-GAN framework \cite{nowozin2016f}, where the gradient \emph{w.r.t.} $\vtheta$ can be conveniently estimated using reparametrization:
\begin{align}
    - \nabla_\vtheta \bb{E}_{q_\vtheta(\vx)} [f^*(T_\vomega(\vx))] = -\bb{E}_{\pi(\vz)}[\nabla_\vtheta f^*(T_\vomega(G_\vtheta(\vz)))] \label{eqn:fgan}
\end{align}
However, it is challenging to apply this procedure to EBMs.
First of all, it is generally hard to find such a deterministic mapping $G_\vtheta$ to produce exact samples from a given EBM. Consequently, we have to use some approximate sampling process (such as Langevin dynamics in Equation~(\ref{eq:langevin-dynamics})) as a surrogate. 
Differentiating through this sampling process (a sequence of Markov chain transitions), as required for computing Equation~(\ref{eqn:fgan}), is computationally expensive. For example, we often need hundreds of steps in Langevin dynamics to produce a single sample\footnote{Note that when using reparametrization for gradient estimation, we cannot use a sample replay buffer to reduce the number of MCMC transition steps, because the initial distribution of the Markov chain $\pi(\vz)$ cannot depend on the model parameters $\vtheta$.}. Therefore, the same number of backward steps is required for gradient backpropagation, where each backward step further involves computing Hessian matrices whose sizes are proportional to the parameter and data dimension. 
This will lead to hundreds of times more memory consumption compared to using Langevin dynamics only for producing samples. A more detailed discussion is provided in Appendix~\ref{app:differetiate-langevin}.

When reparameterization is not possible, an alternative approach is to use the ``log derivative trick'', similar to the REINFORCE algorithm from reinforcement learning \cite{williams1992simple,sutton2000policy}.
Specifically, when the generative model is an EBM, the following theorem provides an unbiased estimation of the gradient:
\begin{theorem}\label{theorem:direct-febm}
For a $\vtheta$-parametrized energy-based distribution $q_\vtheta(\vx) \propto \exp(- E_\vtheta(\vx))$, under Assumption~\ref{assumption:continuous} in Appendix~\ref{app:proof-direct-febm}, the gradient of the variational representation of $f$-divergence (Equation~(\ref{eq:f-gan})) w.r.t. $\vtheta$ can be written as:
\begin{align*}
    - \nabla_\vtheta \bb{E}_{q_\vtheta(\vx)}[f^*(T_\vomega(\vx))] =  &~\bb{E}_{q_\vtheta}[\nabla_\vtheta E_\vtheta(\vx) \cdot f^*(T_\vomega(\vx))] - \\ 
    &~\bb{E}_{q_\vtheta}[\nabla_\vtheta E_\vtheta(\vx)] \cdot \bb{E}_{q_\vtheta}[f^*(T_\vomega(\vx))]
\end{align*}
When we can obtain i.i.d. samples from $q_\vtheta$, we can get an unbiased estimation of the gradient.
\end{theorem}
\vspace{-5pt}
\begin{proof}
See Appendix~\ref{app:proof-direct-febm}.
\end{proof}
\vspace{-5pt}
Theorem~\ref{theorem:direct-febm} provides us a possible way to train EBMs with any $f$-divergence. 
Based on Equation~(\ref{eq:f-gan}), we can alternatively train $T_\vomega$ using samples from $p$ and $q_\vtheta$, and train $q_\vtheta$ using the gradient estimator in Theorem~\ref{theorem:direct-febm}.

However, in practice we found that this approach %
performs poorly, 
especially for high-dimensional data. Intuitively, this method resembles REINFORCE in the sense that $f^*(T_\vomega(\vx))$ specifies an adaptive reward function to guide the training of a stochastic policy $q_\vtheta$ through trial-and-error. Unlike contrastive divergence, here the energy function is never directly evaluated on the training data in the optimization process, since the expectation over $p(\vx)$ in Equation~(\ref{eq:f-gan}) only involves the variational function $T_\vomega$. 
Thus we cannot directly decrease the energies of the training data and increase the energies of the generated data (\emph{i.e.}, to make training data more likely). 
Therefore the entire supervision signal comes from $f^*(T_\vomega(\vx))$. Taking image domain as an example, starting from a random initialization of the parameters $\vtheta$, all the generated samples obtained by MCMC sampling are random noise.
Since it is unlikely for an EBM to generate a realistic image by random exploration, the reward signals from $f^*(T_\vomega(\vx))$ are always uninformative. More precisely, although the gradient estimator is theoretically unbiased, the variance is extremely large and it would require a prohibitive number of samples to work.
Consequently, in practice EBMs trained by this approach failed to generate reasonable images, as shown in Appendix~\ref{app:baseline-samples}. Even starting from a pretrained model learned by contrastive divergence, we empirically found that this approach still cannot provide effective training.

\subsection{$f$-EBM}\label{sec:f-ebm}
Inspired by the analysis of the challenges in Section~\ref{sec:challenges}, we propose to minimize a new variational representation of $f$-divergence between the data distribution and an energy-based distribution, which can mitigate the issues discussed above and induce an effective training scheme for EBMs.

\begin{theorem}\label{theorem:new-f-bound}
Let $P$ and $Q$ be probability measures over the Borel $\sigma$-algebra on domain $\gX$ with densities $p$, $q$ and $P \ll Q$. Additionally, let $q$ be an energy-based distribution, $q(\vx) = \exp(-E(\vx)) / Z_q$.
For any class of functions $\gH =\{H:\gX \rightarrow \bb{R}\}$ such that $\log\left(p \cdot Z_q\right) \in \gH$ and the expectations in the following equation are finite,
we have:
\begin{align*}
    \DF(P \| Q) = \sup_{H \in \gH} &~\bb{E}_{p(\vx)} [f'(\exp(H(\vx) + E(\vx)))] - \\ 
    &~\bb{E}_{q(\vx)}[f^*(f'(\exp(H(\vx) + E(\vx))))]
\end{align*}
where the supreme is attained at $H^*(\vx) = \log\left(p(\vx) \cdot Z_q\right)$.
\end{theorem}
\vspace{-5pt}
\begin{proof}
See Appendix~\ref{app:proof-new-f-bound}.
\end{proof}
\vspace{-5pt}

As before, let us parametrize the energy function and variational function as $E_\vtheta$ and $H_\vomega$ respectively. Based on the new variational representation of $f$-divergence in Theorem~\ref{theorem:new-f-bound}, we propose the following $f$-EBM framework:
\begin{align}
    \min_\vtheta \max_\vomega &~\gL_{f\text{-EBM}}(\vtheta, \vomega;p)\nonumber\\
    =\min_\vtheta \max_\vomega &~\bb{E}_{p(\vx)} [f'(\exp(H_\vomega(\vx) + E_\vtheta(\vx)))] - \label{eq:minimax-f-EBM}\\
    &~\bb{E}_{q_\vtheta(\vx)}[f^*(f'(\exp(H_\vomega(\vx) + E_\vtheta(\vx))))]\nonumber
\end{align}

For a fixed generative model $q_\vtheta$, from Theorem~\ref{theorem:new-f-bound}, we know that given enough capacity and training time (\emph{i.e.}, in the non-parametric limit), the optimal solution of the inner maximization problem is:
\begin{align}
    H_{\vomega^*}(\vx) 
    = \log(p(\vx) \cdot Z_\vtheta)\label{eq:h-optimal}
\end{align}
Consequently, the density ratio $p(\vx) / q_\vtheta(\vx)$ between data distribution and model distribution can be recovered by\footnote{Different from other variational $f$-divergence minimization frameworks such as $f$-GAN, in $f$-EBM, the density ratio is recovered using both the energy function and the variational function.}:
\begin{align}
    &~\exp(H_{\vomega^*}(\vx) + E_\vtheta(\vx)) \label{eq:recover-density-ratio} \\ 
    =&~\exp(\log(p(\vx)Z_\vtheta)) / \exp(-E_\vtheta(\vx)) = p(\vx) / q_\vtheta(\vx) \nonumber
\end{align}
Plugging the optimal form of $H_\vomega$ into Equation~(\ref{eq:minimax-f-EBM}), from Theorem~\ref{theorem:new-f-bound}, we know that the minimax problem in Equation~(\ref{eq:minimax-f-EBM}) can be reformulated as:
\begin{align}
    \min_\vtheta \gL_{f\text{-EBM}}(\vtheta, \vomega^*; p)
    = \min_\vtheta \DF(p \| q_\vtheta).
\end{align}
Hence, the global optimum of the minimax game is achieved if and only if  $q_\vtheta = p$ and $H_\vomega = \log(p \cdot Z_\vtheta)$.

In this framework, the maximization over $\vomega$ closes the gap between the variational lower bound and the true value of $f$-divergence, while the minimization over $\vtheta$ improves the generative model based on the estimated value of the  $f$-divergence. Different from the baseline method described in Section~\ref{sec:challenges} (which is based on the original variational representation in Lemma~\ref{lemma:f-bound}), here both terms in the new variational representation of $f$-divergence (Equation~(\ref{eq:minimax-f-EBM})) 
have non-zero gradients \emph{w.r.t.} $\vtheta$. In other words, as in contrastive divergence, we are able to evaluate the energy function on both real data from $p$ and synthtetic data from $q_\vtheta$.
Therefore, the energy function can directly receive the supervision signal from the real data, which bypasses the random exploration problem of the previous formulation. 

\subsection{Tractable Optimization of $f$-EBM}
To optimize with respect to $\vomega$ in Equation~(\ref{eq:minimax-f-EBM}), we can simply use \emph{i.i.d.} samples from $p$ and $q_\vtheta$ to obtain an unbiased estimation of $\gL_{f\text{-EBM}}(\vtheta, \vomega;p)$ and then update $\vomega$ with the gradient of the estimated objective using gradient ascent, since the expectations are taken with respect to distributions that do not depend on $\vomega$.

However, it is non-trivial to optimize $\vtheta$ because it is hard to backpropagate gradients through the MCMC sampling process (as discussed in Section~\ref{sec:challenge-f-gan}) and in the second term of Equation~(\ref{eq:minimax-f-EBM}), both the sampling distribution and the function within the expectation depend on $\vtheta$. We derive an unbiased gradient estimator for optimizing $\vtheta$ in the following theorem.
\begin{theorem}\label{theorem:theta-gradient}
For a fixed $\vomega$, under Assumptions~\ref{assumption:continuous-1} and \ref{assumption:continuous-2} in Appendix~\ref{app:proof-theta-gradient}, the gradient of $\gL_{f\text{-EBM}}(\vtheta, \vomega;p)$ with respect to $\vtheta$ can be written as:
\begin{align}
    &\nabla_\vtheta \gL_{f\text{-EBM}}(\vtheta, \vomega;p) = \nonumber\\ 
    &~\bb{E}_{p(\vx)}[\nabla_\vtheta f'(\exp(f_\vomega(\vx) + E_\vtheta(x)))] + \label{eq:theta-gradient}\\
    &~ \bb{E}_{q_\vtheta(\vx)}[F_{\vtheta, \vomega}(\vx) \nabla_\vtheta E_\vtheta(\vx)] - \bb{E}_{q_\vtheta(\vx)}[\nabla_\vtheta F_{\vtheta, \vomega}(\vx)] - \nonumber\\
    &~ \bb{E}_{q_\vtheta(\vx)}[\nabla_\vtheta E_\vtheta(\vx)] \cdot \bb{E}_{q_\vtheta(\vx)}[F_{\vtheta, \vomega}(\vx)]\nonumber
\end{align}
where $F_{\vtheta,\vomega}(\vx) \defeq f^*(f'(\exp(H_\vomega(\vx) + E_\vtheta(\vx))))$. When we can obtain i.i.d. samples from $p$ and $q_\vtheta$, we can get an unbiased estimation of the gradient.
\end{theorem}
\vspace{-5pt}
\begin{proof}
See Appendix~\ref{app:proof-theta-gradient}.
\end{proof}
\vspace{-5pt}
Like adversarial training for implicit density generative models \cite{goodfellow2014generative,nowozin2016f}, in practice, optimizing $H_\vomega$ to completion in every inner loop of training is computationally prohibitive and would result in overfitting given finite samples. Motivated by the success of single-step alternative gradient updates for training generative adversarial networks, we propose a practical $f$-EBM algorithm, presented in Algorithm~\ref{alg:f-ebm}. In Appendix~\ref{app:implementation}, we also provide a simple PyTorch \cite{paszke2019pytorch} implementation of Steps 6-9 of Algorithm~\ref{alg:f-ebm}.

As in the training of GANs, although using single-step gradient updates is not generally guaranteed to solve the minimax problem, we empirically observed that such a practical approach enjoys good distribution matching performance and convergence properties (see Appendix~\ref{app:optimization-trajectory} for visualization of the optimization trajectories in a simple setting). This motivates us to conduct a theoretical study on the local convergence property of single-step $f$-EBM in the next section, and we leave the study of global convergence to future work.

\begin{algorithm}[t]
   \caption{Single-Step $f$-EBM. See code implementation in Appendix~\ref{app:implementation}.}
   \label{alg:f-ebm}
\begin{algorithmic}[1]
   \STATE {\bfseries Input:} Empirical data distribution $p_\text{data}$.
   \STATE Initialize energy function $E_\vtheta$ and variational function $H_\vomega$. 
   \REPEAT
   \STATE Draw a minibatch of samples $\gD_p$ from $p_\text{data}$.
   \STATE Draw a minibatch of samples $\gD_q$ from $q_\vtheta$ (\emph{e.g.}, using Langevin dynamics with a sample replay buffer).
   \STATE Estimate $\gL_{f\text{-EBM}}(\vtheta, \vomega;p)$ with $\gD_p$ and $\gD_q$.
   \STATE Perform SGD over $\vomega$ with $\nabla_\vomega \hat{\gL}_{f\text{-EBM}}(\vtheta, \vomega)$.
   \STATE Estimate Equation~(\ref{eq:theta-gradient}) with $\gD_p$ and $\gD_q$.
   \STATE Perform SGD over $\vtheta$ with $-\nabla_\vtheta \hat{\gL}_{f\text{-EBM}}(\vtheta, \vomega)$
   \UNTIL{Convergence}
   \STATE {\bfseries Output:} Learned energy-based model $q_\vtheta \propto \exp(-E_\vtheta)$ and density ratio estimator $\exp(H_\vomega + E_\vtheta)$.
\end{algorithmic}
\end{algorithm}

\section{Theoretical Analysis}
In this section, we conduct a theoretical study on the local convergence property of $f$-EBM, which shows that under proper conditions, the single-step $f$-EBM algorithm
(simultaneous gradient updates and alternative gradient updates) is locally exponentially stable around good equilibrium points. For readability, we defer rigorous statements of assumptions, theorems and proofs to Appendix~\ref{app:proof-local-convergence}.

The main technical tool we use is the non-linear dynamical systems theory \cite{hassan1996khalil}, which have been used to establish the local convergence property of GANs \cite{nagarajan2017gradient,mescheder2018training}. Specifically, the linearization theorem (Theorem~\ref{theorem:equilibrium-hurwitz} in Appendix~\ref{sec:non-linear-system}) states that the local convergence property can be analyzed by examining the spectrum of the Jacobian of the dynamical system $\dot \vphi = v(\vphi)$: if the Jacobian $\vJ = \partial v(\vphi)/ \partial \vphi \at{\vphi = \ophi}$ evaluated at an equilibrium point $\ophi$ is a Hurwitz matrix (\emph{i.e.}, all the eigenvalues of $\vJ$ have strictly negative real parts), then the system will converge to the equilibrium with a linear convergence rate within some neighborhood of $\ophi$. In other words, the equilibrium is locally exponentially stable (Definition~\ref{def:stability} in Appendix~\ref{sec:non-linear-system}).

To begin with, we first derive the Jacobian of the following differential equations\footnote{Following \cite{nagarajan2017gradient}, we focus on the analysis of continuous time ordinary differential equations, which implies similar results for discrete time updates when the learning rate is sufficiently small.}:
\begin{align}\label{eq:ode}
    \begin{pmatrix}
    \dot \vtheta\\
    \dot \vomega
    \end{pmatrix} = 
    \begin{pmatrix}
    - \nabla_\vtheta V(\vtheta, \vomega) \\
    \nabla_\vomega V(\vtheta, \vomega)
    \end{pmatrix}
\end{align}
where $V(\vtheta, \vomega)$ is a generalized definition of the minimax objective in Equation~(\ref{eq:minimax-f-EBM}) (see Equation~(\ref{eq:new-minimax-game}) in Appendix~\ref{sec:notation-setup}).
\begin{theorem}
For the dynamical system defined in Equation~(\ref{eq:new-minimax-game}) and the updates defined in Equation~(\ref{eq:ode}), under Assumption~\ref{assumption:realizability} in Appendix~\ref{sec:notation-setup}, the Jacobian at an equilibrium point ($\otheta,\oomega$) is:
\begin{align*}
    \vJ =&  
    \begin{bmatrix}
    - f''(1) \vK_{EE} & f''(1) \vK_{EH}\\
    - f''(1) \vK_{EH}^\top & -f''(1) \vK_{HH}
    \end{bmatrix}\\
    \vK_{EE} \defeq & (\bb{E}_{p(\vx)}[\nabla_\vtheta \Et \nabla_\vtheta^\top \Et] - \\
    &~~2\bb{E}_{p(\vx)}[\nabla_\vtheta \Et] \bb{E}_{p(\vx)}[\nabla_\vtheta^\top \Et])\at{\otheta}\\
    \vK_{EH} \defeq & (\bb{E}_{p(\vx)} [\nabla_\vtheta \Et]\bb{E}_{p(\vx)}[\nabla_\vomega^\top \Hw])\at{(\otheta, \oomega)}\\
    \vK_{HH} \defeq & (\bb{E}_{p(\vx)}[\nabla_\vomega \Hw \nabla_\vomega^\top \Hw])\at{\oomega}
\end{align*}
\end{theorem}
Now, we present the main theoretical result:
\begin{theorem}\label{theorem:convergence}
The dynamical system defined in Equation~(\ref{eq:new-minimax-game}) and the updates defined in Equation~(\ref{eq:ode}) is locally exponentially stable with respect to an equilibrium point ($\otheta, \oomega$) when the Assumptions~\ref{assumption:realizability},~\ref{assumption:energy-function},~\ref{assumption:full-rank} hold for ($\otheta, \oomega$). Let $\overline{\lambda}(\cdot)$ and $\underline{\lambda}(\cdot)$ denote the largest and smallest eigenvalues of a non-zero positive semi-definite matrix. The rate of convergence is governed only by the eigenvalues $\lambda$ of the Jacobian $\vJ$ of the system at the equilibrium point, with a strictly negative real part upper bounded as:
\begin{itemize}
\small
    \item When $\mathrm{Im}(\lambda) = 0$, 
    \begin{align*}
    &\mathrm{Re}(\lambda) <  \frac{-f''(1) \underline{\lambda}(\vK_{HH}) \underline{\lambda}(\vK_{EH} \vK_{EH}^\top)}{\underline{\lambda}(\vK_{HH}) \overline{\lambda}(\vK_{EE}, \vK_{HH}) + \underline{\lambda}(\vK_{EH} \vK_{EH}^\top)} < 0\\
    &\text{where}~\overline{\lambda}(\vK_{EE}, \vK_{HH}) \defeq \max(\overline{\lambda}(\vK_{EE}), \overline{\lambda}(\vK_{HH}))
    \end{align*}
    \item When $\mathrm{Im}(\lambda) \neq 0$,
    \begin{align*}
    \mathrm{Re}(\lambda) \leq -\frac{f''(1)}{2} (\underline{\lambda}(\vK_{EE}) +\underline{\lambda}(\vK_{HH})) < 0
    \end{align*}
\end{itemize}
\end{theorem}
\vspace{-5pt}
Note that these theoretical results hold for any $f$-divergence with strictly convex and continuously differentiable $f$. All the proofs for this section can be found in Appendix~\ref{app:proof-local-convergence}.
\section{Related Work}
As a generalization to maximum likelihood estimation (MLE), methods based on variational $f$-divergence minimization have been proposed for training implicit generative models \cite{nowozin2016f, mohamed2016learning} and latent variable models \cite{zhang2018training}. Although using general $f$-divergences to train deep EBMs is novel, 
researchers have developed related techniques for learning unnormalized statistical models. \citet{hinton2002training} proposed contrastive divergence (CD) to enable maximum likelihood training for EBMs. Inspired by persistent CD \cite{tieleman2008training} which propogates Markov chains throughout training, recently \citet{du2019implicit} proposed techniques to scale CD to high-dimensional data domains, while \citet{nijkamp2019learning} proposed to learn non-convergent non-persistent short-run MCMC as an efficient alternative to CD for maximum likelihood training. 

Based on the primal-dual view of MLE, doubly dual embedding \cite{dai2018kernel} and adversarial dynamics embedding \cite{dai2019exponential} were proposed, which introduce a dual sampler to avoid the computation of the partition function.
\citet{li2019relieve} proposed a black-box algorithm to perform maximum likelihood training for Markov random field (MRF), which employs two variational distributions to approximately infer the latent variables and estimate the partition function of an MRF.
Noise contrastive estimation (NCE) \cite{gutmann2012noise} train EBMs by performing non-linear logistic regression to classify real data and artificial data from a noise distribution.
Theoretically, \citet{riou2018noise} proved that when the number of artificial data points approaches infinity, NCE is asymptotically equivalent to MLE. Similar to EBMs, diffusion probabilistic models proposed by \cite{sohl2015deep} define the model distribution as a result of a parametric diffusion process starting from a simple known distribution, which is also trained using MLE. 

In this work, to generalize maximum likelihood training of EBMs, we propose a new variational framework, which enables us to train EBMs using any desired $f$-divergence.
\section{Experiments}
\subsection{Fitting Univariate Mixture of Gaussians}\label{sec:experiment-gaussian}
\paragraph{Setup.} 
To understand the effects of different divergences for training EBMs, inspired by \cite{minka2005divergence,nowozin2016f}, we first investigate a model misspecification scenario. Specifically, we use an EBM with a quadratic energy function (corresponding to a Gaussian distribution):
\begin{align}
    E_{\mu, \sigma}(x) = \frac{(x - \mu)^2}{2\sigma^2},~q_{\mu, \sigma}(x) \propto \exp(-E_{\mu, \sigma}(x))\label{eq:gaussian}
\end{align}
to fit a mixture of Gaussians (see Figure~\ref{fig:divergences} for illustration). Here $\mu$ and $\sigma$ are two trainable parameters. Although for energy function in Equation~(\ref{eq:gaussian}), the partition function is actually tractable ($Z_{\mu, \sigma} = \sqrt{2 \pi \sigma^2}$), for experimental purposes, we will treat it as a general EBM and do not leverage this tractability. For the implementation of the variational function $H_\vomega$, we use a fully-connected neural network with two hidden layers (each having 64 hidden units) and $\texttt{tanh}$ activation functions. To optimize the $f$-EBM objective, we use single-step gradient updates (Algorithm~\ref{alg:f-ebm}) to alternatively train the energy function and the variational function, with a learning rate of $0.01$ and a batch size of $1000$.

\paragraph{Parameter Learning Results.}
Because we know the real data density, we can numerically solve the $f$-divergence minimization problem and get the ground-truth solution ($\mu^*, \sigma^*$) induced by the chosen $f$-divergence. Then, we test whether single-step $f$-EBM algorithm can learn good parameters ($\hat{\mu}, \hat{\sigma}$) that are close to the desired solution. As shown in Figure~\ref{fig:divergences}, and Table~\ref{tab:gaussian-result} in Appendix~\ref{app:gaussian-parameter-learning}, single step $f$-EBM is capable of learning model distributions that closely match the desired solutions for various $f$-divergences.
Furthermore, although all $f$-divergences are valid objectives (since they are minimized if and only if the model distribution exactly matches the data distribution), in practice the choice of discrepancy measures has a significant influence on the learned distribution because of model misspecification.

\begin{figure*}[t]
\centering
\subfigure[KL]{
\includegraphics[height=1.2in]{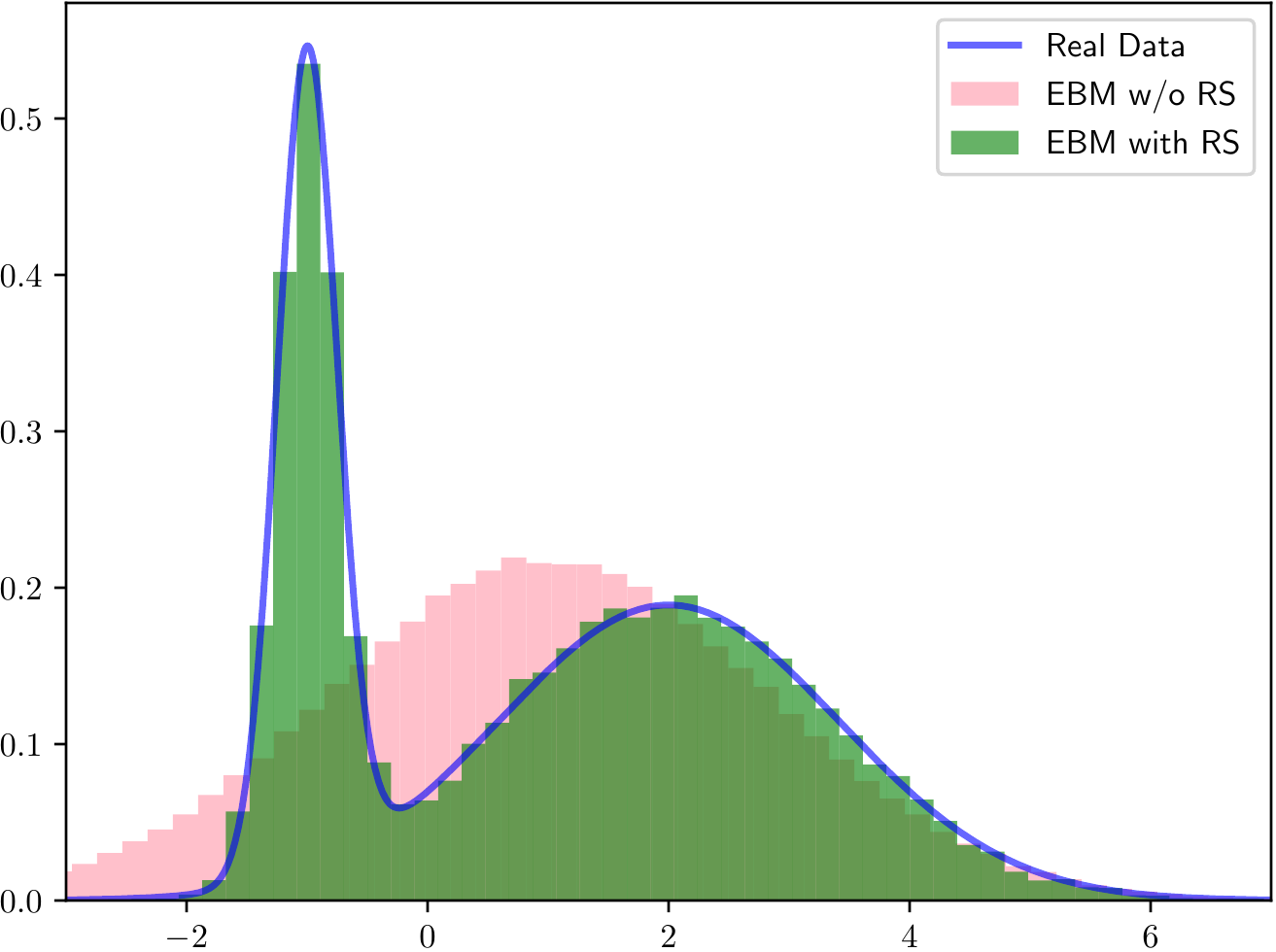}\label{fig:kl_rs}}
\subfigure[Reverse KL]{
\includegraphics[height=1.2in]{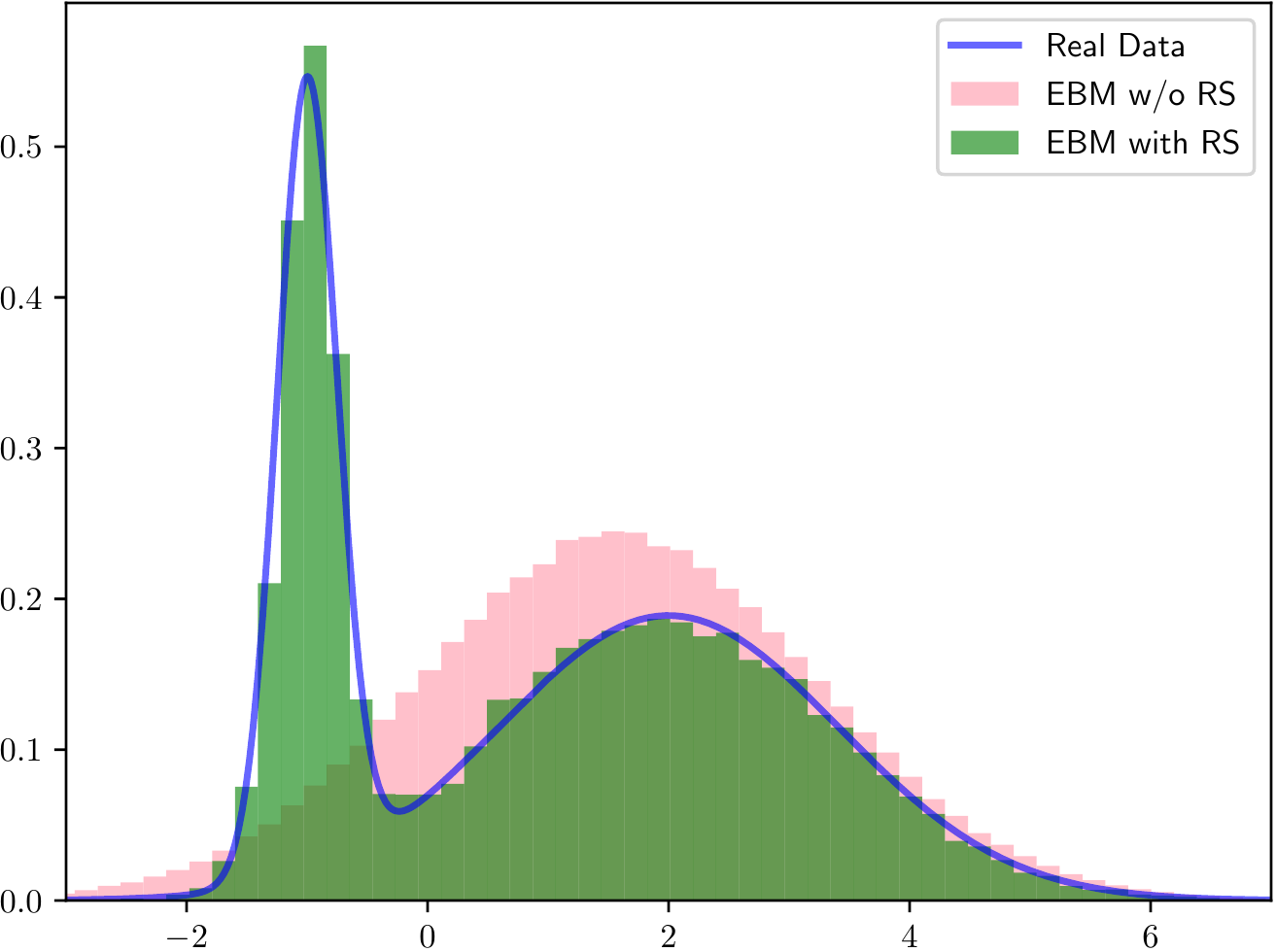}\label{fig:rev-kl_rs}}
\subfigure[Jensen-Shannon]{
\includegraphics[height=1.2in]{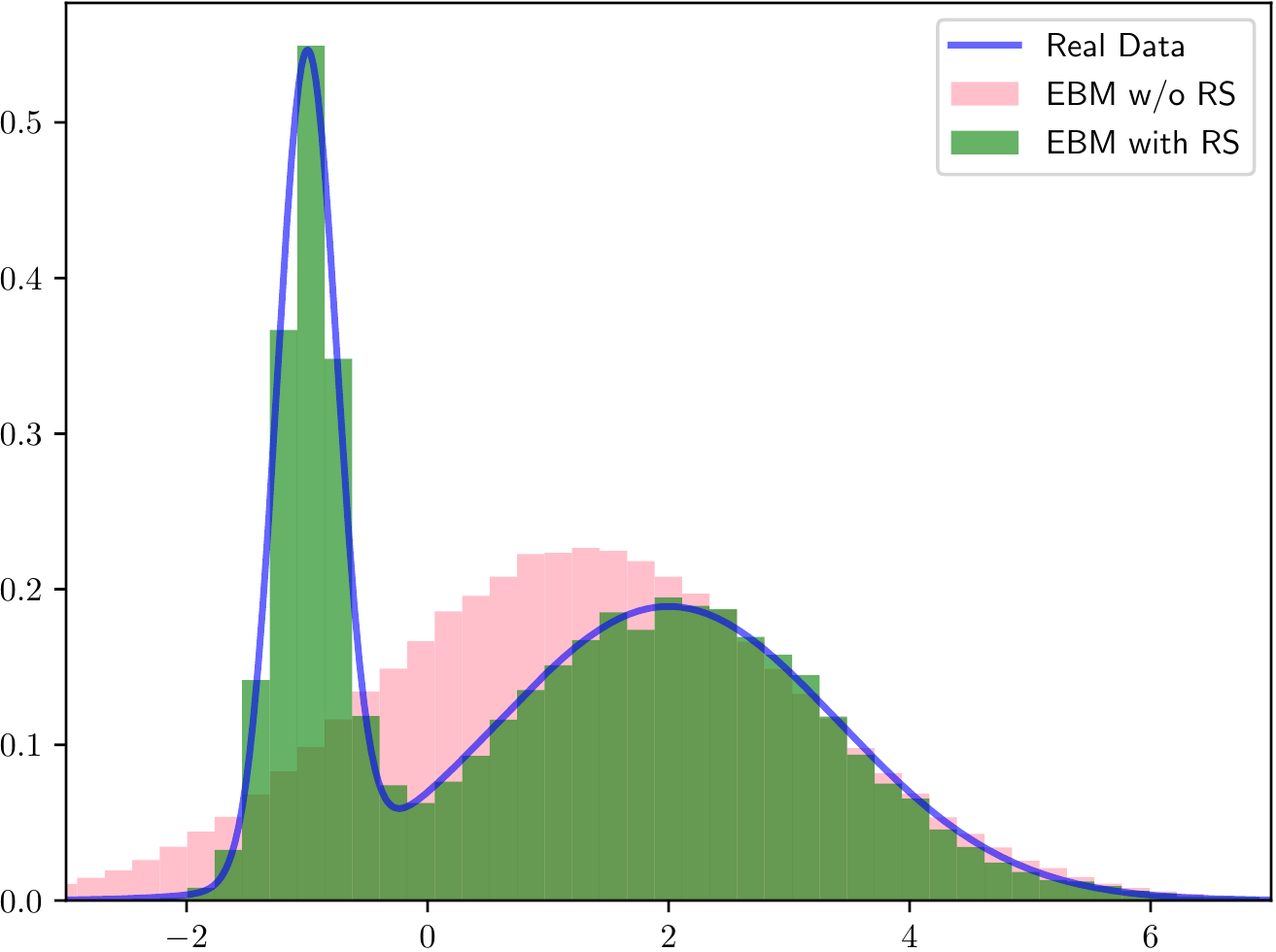}\label{fig:jensen-shannon_rs}}
\subfigure[$\alpha$-Divergence ($\alpha=-1$)]{
\includegraphics[height=1.2in]{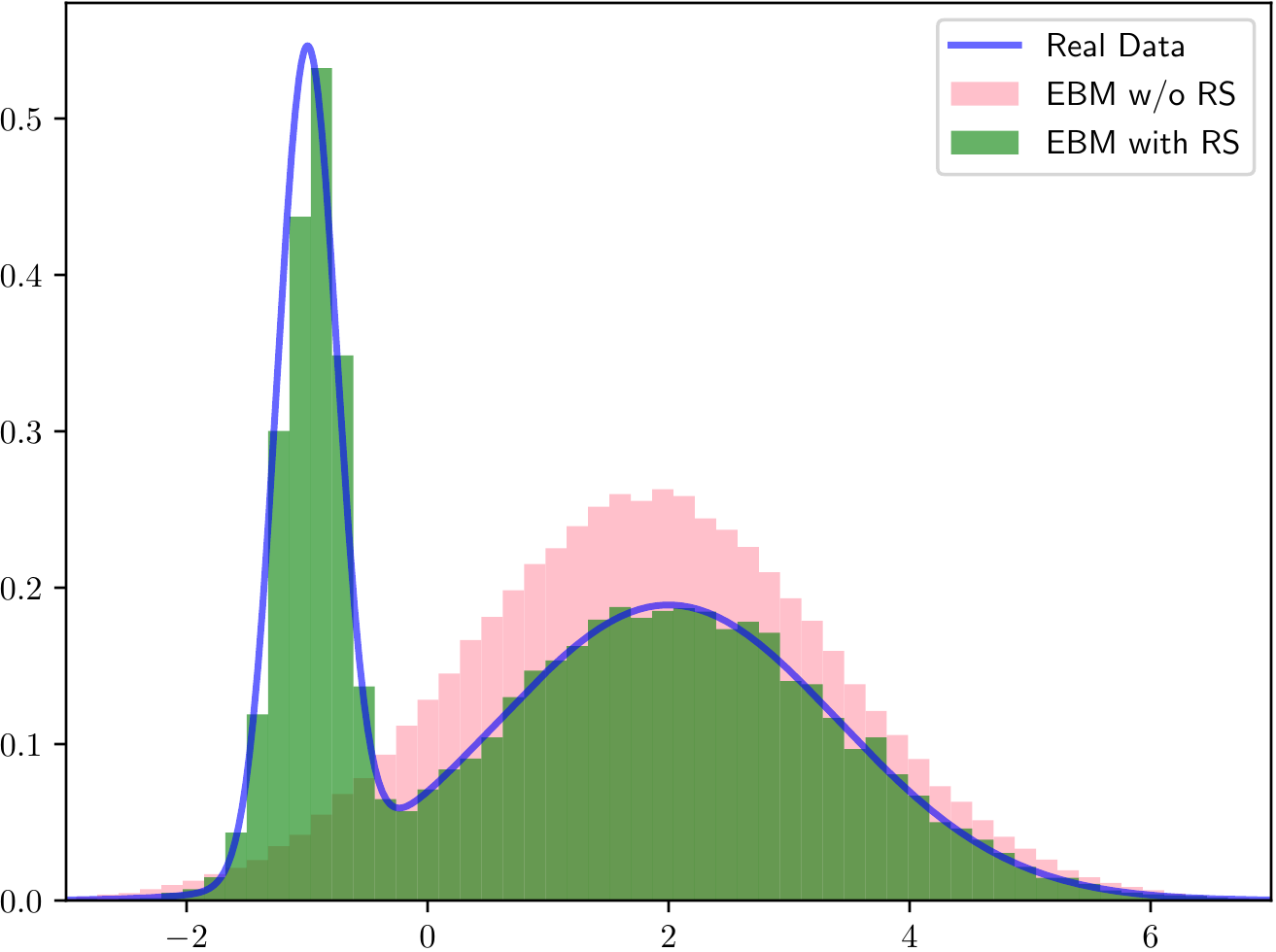}\label{fig:alpha-div_rs}}
\vspace{-5pt}
\caption{The effect of rejection sampling using learned density ratio ($\exp(H_\vomega(x) + E_\vtheta(x))$). The blue line represents the desired real data distribution. The green and pink histograms represent the generated samples with and without rejection sampling respectively.}
\vspace{-5pt}
\label{fig:divergences_rs}
\end{figure*}
\paragraph{Density Ratio Estimation \& Bias Correction.} As shown in Equation~(\ref{eq:recover-density-ratio}), our framework additionally produces a density ratio estimator $\exp(H_\vomega + E_\vtheta)$. First, we visualize the density ratio estimation results in Figure~\ref{fig:divergences_dr} in Appendix~\ref{app:density-ratio}, from which we can see that the estimated density ratio is accurate in most areas of the data support, except in the low-density regions where very few training data points come from.
Furthermore, as discussed in \cite{grover2019bias}, 
when the learned generative model produces biased statistics relative to the real data distribution, it is possible to
correct the bias using estimated density ratio in conjunction with importance sampling or rejection sampling. With the goal of sampling from real data distribution $p$, we employ rejection sampling \cite{halton1970retrospective,azadi2018discriminator}, where we accept samples from the EBM with a probability proportional to the density ratio. As shown in Figure~\ref{fig:divergences_rs}, using the learned density ratio, we are able to recover the real data distribution even in such a model misspecification scenario.

\paragraph{Convergence Results.} From the optimization trajectories shown in Figure~\ref{fig:divergences_grad} in Appendix~\ref{app:optimization-trajectory}, we empirically found that, using single-step alternative gradient updates, $f$-EBM is able to converge to the desired solution starting from different random initializations.

\subsection{Modeling Natural Images}
In this section, we demonstrate that $f$-EBMs can produce high-quality images for various choices of $f$-divergences\footnote{Our implementation of $f$-EBM can be found at: \url{https://github.com/ermongroup/f-EBM}}. Furthermore, we show that energy functions trained by $f$-EBMs can learn meaningful image representations for image denoising and inpainting tasks.

\paragraph{Setup.} We conduct experiments with two commonly used image datasets, CelebA \cite{liu2015faceattributes} and CIFAR-10 \cite{krizhevsky2009learning}. Since the performance is sensitive to the model architectures, for fair comparisons, we use the same architecture and training hyper-parameters for $f$-EBMs and the contrastive divergence baseline \cite{du2019implicit}. See Appendix~\ref{app:training-details} for more details. 

\paragraph{Image Generation.} For qualitative evaluation, we show uncurated samples from $f$-EBMs for CIFAR-10 and CelebA in Appendix~\ref{app:cifar10-samples} and \ref{app:celeba-samples}. These high quality samples demonstrate that the $f$-EBM framework is effective for training EBMs under a variety of discrepancy measures. For quantitative evaluation, we report Inception \cite{salimans2016improved} and FID \cite{heusel2017gans} scores for CIFAR-10 in Table~\ref{tab:cifar10-generation}. For CelebA dataset, since different preprocessing may lead to numbers that are not directly comparable, and \citet{heusel2017gans} show that Inception score is not a valid metric for CelebA, we only report the FID score for EBMs used in our experiments in Table~\ref{tab:celeba-generation}. Table~\ref{tab:cifar10-generation} and \ref{tab:celeba-generation} suggest that we can achieve a significant sample quality improvement with $f$-EBMs using $f$-divergences such as Jensen-Shannon, Squared Hellinger and Reverse KL. Moreover, the performance of $f$-EBM with KL divergence is similar to contrastive divergence, whose underlying discrepancy measure is also KL. To illustrate the Langevin dynamics sampling process, we show how samples evolve from random noise in Appendix~\ref{app:langevin-cifar10} and \ref{app:langevin-celeba}. For generalization, we show nearest neighbor images of our samples in the training set in Appendix~\ref{app:nearest-neighbor}, from which we know that our model is not simply memorizing training images.

\paragraph{Image Denoising \& Inpainting.} In Appendix~\ref{app:inpainting} and \ref{app:denoising}, we show that energy functions learned by $f$-EBMs can be used for image inpainting and denoising tasks, where we use Langevin dynamics with the gradients of the learned energy functions to refine the images.

\begin{table}[htb]
\vspace{-5pt}
\caption{Inception and FID scores for CIFAR-10 conditional generation. We compare with results reported by 
DCGAN \cite{radford2015unsupervised,wang2016learning}, 
Improved GAN \cite{salimans2016improved}, 
Fisher GAN \cite{mroueh2017fisher}, 
ACGAN \cite{odena2017conditional}, 
WGAN-GP \cite{gulrajani2017improved}, 
SNGAN \cite{miyato2018spectral}, 
Contrastive Divergence \cite{du2019implicit}.
}
\begin{center}
\scalebox{0.9}{
\begin{tabular}{l|lcc}
\toprule
 & Method & Inception & FID\\
\midrule
\multirow{6}{*}{GANs} 
& DCGAN & 6.58 & -\\
& Improved GAN & 8.09 & -\\
& Fisher GAN & 8.16 & -\\
& ACGAN & 8.25 & -\\
& WGAN-GP & 8.42 & -\\
& SNGAN & 8.60 & 17.50\\
\midrule
\multirow{5}{*}{EBMs} & Contrastive Divergence (KL) & 8.30 & 37.90\\
& $f$-EBM (KL) & 8.11 $\pm$ .06 & 37.36\\
& $f$-EBM (Reverse KL) & 8.49 $\pm$ .09 & 33.25\\
& $f$-EBM (Squared Hellinger) & 8.57 $\pm$ .08 & 32.19\\
& $f$-EBM (Jensen Shannon) & \textbf{8.61} $\pm$ .06 & \textbf{30.86}\\
\bottomrule
\end{tabular}
}
\end{center}
\vspace{-5pt}
\label{tab:cifar10-generation}
\end{table}

\begin{table}[htb]
\vspace{-10pt}
\caption{FID for CelebA ($32 \times 32$) unconditional generation.
}
\begin{center}
\scalebox{0.9}{
\begin{tabular}{l|lcc}
\toprule
 & Method & FID\\
\midrule
\multirow{5}{*}{EBMs} & Contrastive Divergence (KL) & 35.97\\
& $f$-EBM (KL) & 37.41\\
& $f$-EBM (Reverse KL) & 26.75\\
& $f$-EBM (Jensen Shannon) & 26.53\\
& $f$-EBM (Squared Hellinger) & \textbf{23.67}\\
\bottomrule
\end{tabular}
}
\end{center}
\vspace{-15pt}
\label{tab:celeba-generation}
\end{table}

\section{Discussion and Future Work}
In this paper, based on a new variational representation of $f$-divergences, we propose a general framework termed $f$-EBM, which enables us to train deep energy-based models using various descrepancy measures within the $f$-divergence family. We demonstrate the effectiveness of $f$-EBM both theoretically and empirically. Experimental results on CIFAR-10 and CelebA datasets show that our framework can achieve significant sample quality improvement compared to the predominant contrastive divergence method.
For future works, there are some other possible ways that may be potentially useful for training EBMs with $f$-divergences. First, to directly tackle the computational issues of the gradient reparametrization method in Section~\ref{sec:challenge-f-gan}, we may redefine the generative process as the combination of a proposal distribution and MCMC sampling so that we can jointly train the proposal model and the energy function, and we only need to backpropagate through a few MCMC steps \cite{lawson2019energy}. We may also define the generative process as MCMC sampling in the latent space together with a decoder that transforms latent codes to samples \cite{girolami2011riemann,kumar2019maximum}, such that the dimension of the latent space is much smaller than the sample space and computing Hessians is more efficient. Since these methods rely on careful design of additional components such as a proposal or a decoder, we leave them as interesting avenues for future works.

\section*{Acknowledgments}
Toyota Research Institute (``TRI'') provided funds to assist the authors with their research but this article solely reflects the opinions and conclusions of its authors and not TRI or any other Toyota entity. This research was also supported by NSF (\#1651565, \#1522054, \#1733686), ONR (N00014-19-1-2145), AFOSR (FA9550-19-1-0024).
\bibliography{ref}

\begin{thebibliography}{58}
\providecommand{\natexlab}[1]{#1}
\providecommand{\url}[1]{\texttt{#1}}
\expandafter\ifx\csname urlstyle\endcsname\relax
  \providecommand{\doi}[1]{doi: #1}\else
  \providecommand{\doi}{doi: \begingroup \urlstyle{rm}\Url}\fi

\bibitem[Abadi et~al.(2016)Abadi, Barham, Chen, Chen, Davis, Dean, Devin,
  Ghemawat, Irving, Isard, et~al.]{abadi2016tensorflow}
Abadi, M., Barham, P., Chen, J., Chen, Z., Davis, A., Dean, J., Devin, M.,
  Ghemawat, S., Irving, G., Isard, M., et~al.
\newblock Tensorflow: A system for large-scale machine learning.
\newblock In \emph{12th $\{$USENIX$\}$ Symposium on Operating Systems Design
  and Implementation ($\{$OSDI$\}$ 16)}, pp.\  265--283, 2016.

\bibitem[Arjovsky et~al.(2017)Arjovsky, Chintala, and
  Bottou]{arjovsky2017wasserstein}
Arjovsky, M., Chintala, S., and Bottou, L.
\newblock Wasserstein gan.
\newblock \emph{arXiv preprint arXiv:1701.07875}, 2017.

\bibitem[Azadi et~al.(2018)Azadi, Olsson, Darrell, Goodfellow, and
  Odena]{azadi2018discriminator}
Azadi, S., Olsson, C., Darrell, T., Goodfellow, I., and Odena, A.
\newblock Discriminator rejection sampling.
\newblock \emph{arXiv preprint arXiv:1810.06758}, 2018.

\bibitem[Dai et~al.(2018)Dai, Dai, Gretton, Song, Schuurmans, and
  He]{dai2018kernel}
Dai, B., Dai, H., Gretton, A., Song, L., Schuurmans, D., and He, N.
\newblock Kernel exponential family estimation via doubly dual embedding.
\newblock \emph{arXiv preprint arXiv:1811.02228}, 2018.

\bibitem[Dai et~al.(2019)Dai, Liu, Dai, He, Gretton, Song, and
  Schuurmans]{dai2019exponential}
Dai, B., Liu, Z., Dai, H., He, N., Gretton, A., Song, L., and Schuurmans, D.
\newblock Exponential family estimation via adversarial dynamics embedding.
\newblock \emph{arXiv preprint arXiv:1904.12083}, 2019.

\bibitem[Dinh et~al.(2014)Dinh, Krueger, and Bengio]{dinh2014nice}
Dinh, L., Krueger, D., and Bengio, Y.
\newblock Nice: Non-linear independent components estimation.
\newblock \emph{arXiv preprint arXiv:1410.8516}, 2014.

\bibitem[Dinh et~al.(2016)Dinh, Sohl-Dickstein, and Bengio]{dinh2016density}
Dinh, L., Sohl-Dickstein, J., and Bengio, S.
\newblock Density estimation using real nvp.
\newblock \emph{arXiv preprint arXiv:1605.08803}, 2016.

\bibitem[Du \& Mordatch(2019)Du and Mordatch]{du2019implicit}
Du, Y. and Mordatch, I.
\newblock Implicit generation and modeling with energy based models.
\newblock In \emph{Advances in Neural Information Processing Systems 32}, pp.\
  3603--3613, 2019.

\bibitem[Germain et~al.(2015)Germain, Gregor, Murray, and
  Larochelle]{germain2015made}
Germain, M., Gregor, K., Murray, I., and Larochelle, H.
\newblock Made: Masked autoencoder for distribution estimation.
\newblock In \emph{International Conference on Machine Learning}, pp.\
  881--889, 2015.

\bibitem[Girolami \& Calderhead(2011)Girolami and
  Calderhead]{girolami2011riemann}
Girolami, M. and Calderhead, B.
\newblock Riemann manifold langevin and hamiltonian monte carlo methods.
\newblock \emph{Journal of the Royal Statistical Society: Series B (Statistical
  Methodology)}, 73\penalty0 (2):\penalty0 123--214, 2011.

\bibitem[Goodfellow et~al.(2014)Goodfellow, Pouget-Abadie, Mirza, Xu,
  Warde-Farley, Ozair, Courville, and Bengio]{goodfellow2014generative}
Goodfellow, I., Pouget-Abadie, J., Mirza, M., Xu, B., Warde-Farley, D., Ozair,
  S., Courville, A., and Bengio, Y.
\newblock Generative adversarial nets.
\newblock In \emph{Advances in neural information processing systems}, pp.\
  2672--2680, 2014.

\bibitem[Goodfellow et~al.(2016)Goodfellow, Bengio, and
  Courville]{Goodfellow-et-al-2016}
Goodfellow, I., Bengio, Y., and Courville, A.
\newblock \emph{Deep Learning}.
\newblock MIT Press, 2016.
\newblock \url{http://www.deeplearningbook.org}.

\bibitem[Grover et~al.(2019)Grover, Song, Agarwal, Tran, Kapoor, Horvitz, and
  Ermon]{grover2019bias}
Grover, A., Song, J., Agarwal, A., Tran, K., Kapoor, A., Horvitz, E., and
  Ermon, S.
\newblock Bias correction of learned generative models using likelihood-free
  importance weighting.
\newblock \emph{arXiv preprint arXiv:1906.09531}, 2019.

\bibitem[Gulrajani et~al.(2017)Gulrajani, Ahmed, Arjovsky, Dumoulin, and
  Courville]{gulrajani2017improved}
Gulrajani, I., Ahmed, F., Arjovsky, M., Dumoulin, V., and Courville, A.~C.
\newblock Improved training of wasserstein gans.
\newblock In \emph{Advances in neural information processing systems}, pp.\
  5767--5777, 2017.

\bibitem[Gutmann \& Hyv{\"a}rinen(2012)Gutmann and
  Hyv{\"a}rinen]{gutmann2012noise}
Gutmann, M.~U. and Hyv{\"a}rinen, A.
\newblock Noise-contrastive estimation of unnormalized statistical models, with
  applications to natural image statistics.
\newblock \emph{Journal of Machine Learning Research}, 13\penalty0
  (Feb):\penalty0 307--361, 2012.

\bibitem[Halton(1970)]{halton1970retrospective}
Halton, J.~H.
\newblock A retrospective and prospective survey of the monte carlo method.
\newblock \emph{Siam review}, 12\penalty0 (1):\penalty0 1--63, 1970.

\bibitem[Hassan(1996)]{hassan1996khalil}
Hassan, K.
\newblock Khalil, nonlinear systems.
\newblock \emph{Prentice-Hall, Inc., New Jersey}, 1996.

\bibitem[He et~al.(2016)He, Zhang, Ren, and Sun]{he2016deep}
He, K., Zhang, X., Ren, S., and Sun, J.
\newblock Deep residual learning for image recognition.
\newblock In \emph{Proceedings of the IEEE conference on computer vision and
  pattern recognition}, pp.\  770--778, 2016.

\bibitem[Heusel et~al.(2017)Heusel, Ramsauer, Unterthiner, Nessler, and
  Hochreiter]{heusel2017gans}
Heusel, M., Ramsauer, H., Unterthiner, T., Nessler, B., and Hochreiter, S.
\newblock Gans trained by a two time-scale update rule converge to a local nash
  equilibrium.
\newblock In \emph{Advances in neural information processing systems}, pp.\
  6626--6637, 2017.

\bibitem[Hinton(2002)]{hinton2002training}
Hinton, G.~E.
\newblock Training products of experts by minimizing contrastive divergence.
\newblock \emph{Neural computation}, 14\penalty0 (8):\penalty0 1771--1800,
  2002.

\bibitem[Ho \& Ermon(2016)Ho and Ermon]{ho2016generative}
Ho, J. and Ermon, S.
\newblock Generative adversarial imitation learning.
\newblock In \emph{Advances in neural information processing systems}, pp.\
  4565--4573, 2016.

\bibitem[Kingma \& Dhariwal(2018)Kingma and Dhariwal]{kingma2018glow}
Kingma, D.~P. and Dhariwal, P.
\newblock Glow: Generative flow with invertible 1x1 convolutions.
\newblock In \emph{Advances in Neural Information Processing Systems}, pp.\
  10215--10224, 2018.

\bibitem[Krizhevsky et~al.(2009)Krizhevsky, Hinton,
  et~al.]{krizhevsky2009learning}
Krizhevsky, A., Hinton, G., et~al.
\newblock Learning multiple layers of features from tiny images.
\newblock 2009.

\bibitem[Kumar et~al.(2019)Kumar, Goyal, Courville, and
  Bengio]{kumar2019maximum}
Kumar, R., Goyal, A., Courville, A., and Bengio, Y.
\newblock Maximum entropy generators for energy-based models.
\newblock \emph{arXiv preprint arXiv:1901.08508}, 2019.

\bibitem[Larochelle \& Murray(2011)Larochelle and Murray]{larochelle2011neural}
Larochelle, H. and Murray, I.
\newblock The neural autoregressive distribution estimator.
\newblock In \emph{Proceedings of the Fourteenth International Conference on
  Artificial Intelligence and Statistics}, pp.\  29--37, 2011.

\bibitem[Lawson et~al.(2019)Lawson, Tucker, Dai, and
  Ranganath]{lawson2019energy}
Lawson, J., Tucker, G., Dai, B., and Ranganath, R.
\newblock Energy-inspired models: Learning with sampler-induced distributions.
\newblock In \emph{Advances in Neural Information Processing Systems}, pp.\
  8499--8511, 2019.

\bibitem[Li et~al.(2019)Li, Du, Xu, Welling, Zhu, and Zhang]{li2019relieve}
Li, C., Du, C., Xu, K., Welling, M., Zhu, J., and Zhang, B.
\newblock To relieve your headache of training an mrf, take advil.
\newblock \emph{arXiv preprint arXiv:1901.08400}, 2019.

\bibitem[Liu et~al.(2015)Liu, Luo, Wang, and Tang]{liu2015faceattributes}
Liu, Z., Luo, P., Wang, X., and Tang, X.
\newblock Deep learning face attributes in the wild.
\newblock In \emph{Proceedings of International Conference on Computer Vision
  (ICCV)}, December 2015.

\bibitem[Mescheder et~al.(2018)Mescheder, Geiger, and
  Nowozin]{mescheder2018training}
Mescheder, L., Geiger, A., and Nowozin, S.
\newblock Which training methods for gans do actually converge?
\newblock \emph{arXiv preprint arXiv:1801.04406}, 2018.

\bibitem[Minka et~al.(2005)]{minka2005divergence}
Minka, T. et~al.
\newblock Divergence measures and message passing.
\newblock Technical report, Technical report, Microsoft Research, 2005.

\bibitem[Miyato et~al.(2018)Miyato, Kataoka, Koyama, and
  Yoshida]{miyato2018spectral}
Miyato, T., Kataoka, T., Koyama, M., and Yoshida, Y.
\newblock Spectral normalization for generative adversarial networks.
\newblock \emph{arXiv preprint arXiv:1802.05957}, 2018.

\bibitem[Mohamed \& Lakshminarayanan(2016)Mohamed and
  Lakshminarayanan]{mohamed2016learning}
Mohamed, S. and Lakshminarayanan, B.
\newblock Learning in implicit generative models.
\newblock \emph{arXiv preprint arXiv:1610.03483}, 2016.

\bibitem[Mroueh \& Sercu(2017)Mroueh and Sercu]{mroueh2017fisher}
Mroueh, Y. and Sercu, T.
\newblock Fisher gan.
\newblock In \emph{Advances in Neural Information Processing Systems}, pp.\
  2513--2523, 2017.

\bibitem[Nagarajan \& Kolter(2017)Nagarajan and Kolter]{nagarajan2017gradient}
Nagarajan, V. and Kolter, J.~Z.
\newblock Gradient descent gan optimization is locally stable.
\newblock In \emph{Advances in Neural Information Processing Systems}, pp.\
  5585--5595, 2017.

\bibitem[Neal et~al.(2011)]{neal2011mcmc}
Neal, R.~M. et~al.
\newblock Mcmc using hamiltonian dynamics.
\newblock \emph{Handbook of markov chain monte carlo}, 2\penalty0
  (11):\penalty0 2, 2011.

\bibitem[Nguyen et~al.(2010)Nguyen, Wainwright, and
  Jordan]{nguyen2010estimating}
Nguyen, X., Wainwright, M.~J., and Jordan, M.~I.
\newblock Estimating divergence functionals and the likelihood ratio by convex
  risk minimization.
\newblock \emph{IEEE Transactions on Information Theory}, 56\penalty0
  (11):\penalty0 5847--5861, 2010.

\bibitem[Nijkamp et~al.(2019)Nijkamp, Hill, Zhu, and Wu]{nijkamp2019learning}
Nijkamp, E., Hill, M., Zhu, S.-C., and Wu, Y.~N.
\newblock Learning non-convergent non-persistent short-run mcmc toward
  energy-based model.
\newblock In \emph{Advances in Neural Information Processing Systems}, pp.\
  5233--5243, 2019.

\bibitem[Nowozin et~al.(2016)Nowozin, Cseke, and Tomioka]{nowozin2016f}
Nowozin, S., Cseke, B., and Tomioka, R.
\newblock f-gan: Training generative neural samplers using variational
  divergence minimization.
\newblock In \emph{Advances in neural information processing systems}, pp.\
  271--279, 2016.

\bibitem[Odena et~al.(2017)Odena, Olah, and Shlens]{odena2017conditional}
Odena, A., Olah, C., and Shlens, J.
\newblock Conditional image synthesis with auxiliary classifier gans.
\newblock In \emph{Proceedings of the 34th International Conference on Machine
  Learning-Volume 70}, pp.\  2642--2651. JMLR. org, 2017.

\bibitem[Oord et~al.(2016{\natexlab{a}})Oord, Dieleman, Zen, Simonyan, Vinyals,
  Graves, Kalchbrenner, Senior, and Kavukcuoglu]{oord2016wavenet}
Oord, A. v.~d., Dieleman, S., Zen, H., Simonyan, K., Vinyals, O., Graves, A.,
  Kalchbrenner, N., Senior, A., and Kavukcuoglu, K.
\newblock Wavenet: A generative model for raw audio.
\newblock \emph{arXiv preprint arXiv:1609.03499}, 2016{\natexlab{a}}.

\bibitem[Oord et~al.(2016{\natexlab{b}})Oord, Kalchbrenner, and
  Kavukcuoglu]{oord2016pixel}
Oord, A. v.~d., Kalchbrenner, N., and Kavukcuoglu, K.
\newblock Pixel recurrent neural networks.
\newblock \emph{arXiv preprint arXiv:1601.06759}, 2016{\natexlab{b}}.

\bibitem[Paszke et~al.(2019)Paszke, Gross, Massa, Lerer, Bradbury, Chanan,
  Killeen, Lin, Gimelshein, Antiga, et~al.]{paszke2019pytorch}
Paszke, A., Gross, S., Massa, F., Lerer, A., Bradbury, J., Chanan, G., Killeen,
  T., Lin, Z., Gimelshein, N., Antiga, L., et~al.
\newblock Pytorch: An imperative style, high-performance deep learning library.
\newblock In \emph{Advances in Neural Information Processing Systems}, pp.\
  8024--8035, 2019.

\bibitem[Poon \& Domingos(2011)Poon and Domingos]{poon2011sum}
Poon, H. and Domingos, P.
\newblock Sum-product networks: A new deep architecture.
\newblock In \emph{2011 IEEE International Conference on Computer Vision
  Workshops (ICCV Workshops)}, pp.\  689--690. IEEE, 2011.

\bibitem[Radford et~al.(2015)Radford, Metz, and
  Chintala]{radford2015unsupervised}
Radford, A., Metz, L., and Chintala, S.
\newblock Unsupervised representation learning with deep convolutional
  generative adversarial networks.
\newblock \emph{arXiv preprint arXiv:1511.06434}, 2015.

\bibitem[Riou-Durand \& Chopin(2018)Riou-Durand and Chopin]{riou2018noise}
Riou-Durand, L. and Chopin, N.
\newblock Noise contrastive estimation: asymptotics, comparison with mc-mle.
\newblock \emph{arXiv preprint arXiv:1801.10381}, 2018.

\bibitem[Robert \& Casella(2013)Robert and Casella]{robert2013monte}
Robert, C. and Casella, G.
\newblock \emph{Monte Carlo statistical methods}.
\newblock Springer Science \& Business Media, 2013.

\bibitem[Salimans et~al.(2016)Salimans, Goodfellow, Zaremba, Cheung, Radford,
  and Chen]{salimans2016improved}
Salimans, T., Goodfellow, I., Zaremba, W., Cheung, V., Radford, A., and Chen,
  X.
\newblock Improved techniques for training gans.
\newblock In \emph{Advances in neural information processing systems}, pp.\
  2234--2242, 2016.

\bibitem[Sohl-Dickstein et~al.(2015)Sohl-Dickstein, Weiss, Maheswaranathan, and
  Ganguli]{sohl2015deep}
Sohl-Dickstein, J., Weiss, E.~A., Maheswaranathan, N., and Ganguli, S.
\newblock Deep unsupervised learning using nonequilibrium thermodynamics.
\newblock \emph{arXiv preprint arXiv:1503.03585}, 2015.

\bibitem[Sutton et~al.(2000)Sutton, McAllester, Singh, and
  Mansour]{sutton2000policy}
Sutton, R.~S., McAllester, D.~A., Singh, S.~P., and Mansour, Y.
\newblock Policy gradient methods for reinforcement learning with function
  approximation.
\newblock In \emph{Advances in neural information processing systems}, pp.\
  1057--1063, 2000.

\bibitem[Theis et~al.(2015)Theis, Oord, and Bethge]{theis2015note}
Theis, L., Oord, A. v.~d., and Bethge, M.
\newblock A note on the evaluation of generative models.
\newblock \emph{arXiv preprint arXiv:1511.01844}, 2015.

\bibitem[Tieleman(2008)]{tieleman2008training}
Tieleman, T.
\newblock Training restricted boltzmann machines using approximations to the
  likelihood gradient.
\newblock In \emph{Proceedings of the 25th international conference on Machine
  learning}, pp.\  1064--1071, 2008.

\bibitem[Turner(2005)]{turner2005cd}
Turner, R.
\newblock Cd notes.
\newblock 2005.

\bibitem[Van~den Oord et~al.(2016)Van~den Oord, Kalchbrenner, Espeholt,
  Vinyals, Graves, et~al.]{van2016conditional}
Van~den Oord, A., Kalchbrenner, N., Espeholt, L., Vinyals, O., Graves, A.,
  et~al.
\newblock Conditional image generation with pixelcnn decoders.
\newblock In \emph{Advances in neural information processing systems}, pp.\
  4790--4798, 2016.

\bibitem[Wang \& Liu(2016)Wang and Liu]{wang2016learning}
Wang, D. and Liu, Q.
\newblock Learning to draw samples: With application to amortized mle for
  generative adversarial learning.
\newblock \emph{arXiv preprint arXiv:1611.01722}, 2016.

\bibitem[Welling \& Teh(2011)Welling and Teh]{welling2011bayesian}
Welling, M. and Teh, Y.~W.
\newblock Bayesian learning via stochastic gradient langevin dynamics.
\newblock In \emph{Proceedings of the 28th international conference on machine
  learning (ICML-11)}, pp.\  681--688, 2011.

\bibitem[Williams(1992)]{williams1992simple}
Williams, R.~J.
\newblock Simple statistical gradient-following algorithms for connectionist
  reinforcement learning.
\newblock \emph{Machine learning}, 8\penalty0 (3-4):\penalty0 229--256, 1992.

\bibitem[Yu et~al.(2017)Yu, Zhang, Wang, and Yu]{yu2017seqgan}
Yu, L., Zhang, W., Wang, J., and Yu, Y.
\newblock Seqgan: Sequence generative adversarial nets with policy gradient.
\newblock In \emph{Thirty-First AAAI Conference on Artificial Intelligence},
  2017.

\bibitem[Zhang et~al.(2018)Zhang, Bird, Habib, Xu, and
  Barber]{zhang2018training}
Zhang, M., Bird, T., Habib, R., Xu, T., and Barber, D.
\newblock Training generative latent models by variational f-divergence
  minimization.
\newblock 2018.

\end{thebibliography}
\bibliographystyle{icml2020}

\appendix
\onecolumn
\section{Proof of Theorem~\ref{theorem:direct-febm}}\label{app:proof-direct-febm}
\addtocounter{theorem}{-5}
\begin{assumption}\label{assumption:continuous}
The function $q_\vtheta(\vx) f^*(T_\vomega(\vx))$ and its partial derivative $\nabla_\vtheta q_\vtheta(\vx) f^*(T_\vomega(\vx))$ are continuous w.r.t. $\vtheta$ and $\vx$.
\end{assumption}
\begin{theorem}
For a $\vtheta$-parametrized energy-based model $q_\vtheta(\vx) = \frac{\exp(- E_\vtheta(\vx))}{Z_\vtheta}$, under Assumption~\ref{assumption:continuous}, the gradient of the variational representation of $f$-divergence (Equation~(\ref{eq:f-bound})) w.r.t. $\vtheta$ can be written as:
\begin{align*}
    - \nabla_\vtheta \bb{E}_{q_\vtheta(\vx)}[f^*(T_\vomega(\vx))] =  \bb{E}_{q_\vtheta(\vx)}[\nabla_\vtheta E_\vtheta(\vx) f^*(T_\vomega(\vx))] - \bb{E}_{q_\vtheta(\vx)}[\nabla_\vtheta E_\vtheta(\vx)] \cdot \bb{E}_{q_\vtheta(\vx)}[f^*(T_\vomega(\vx))]
\end{align*}
When we can obtain i.i.d. samples from $q_\vtheta$, we can get an unbiased estimation of the gradient.
\end{theorem}
\begin{proof}
Under Assumption~\ref{assumption:continuous}, with Leibniz integral rule, we have:
\begin{align*}
    - \nabla_\vtheta \bb{E}_{q_\vtheta(\vx)}[f^*(T_\vomega(\vx))] &=
    - \int \nabla_\vtheta q_\vtheta(\vx) f^*(T_\vomega(\vx)) \mathrm{d} \vx\\
    &= - \int q_\vtheta(\vx) \frac{\nabla_\vtheta q_\vtheta(\vx)}{q_\vtheta(\vx)} f^*(T_\vomega(\vx)) \mathrm{d} \vx\\
    &= - \int q_\vtheta(\vx) \nabla_\vtheta \log q_\vtheta(\vx) f^*(T_\vomega(\vx)) \mathrm{d} \vx\\
    &= - \int q_\vtheta(\vx) \left[- \nabla_\vtheta E_\vtheta(\vx) - \nabla_\vtheta \log Z_\vtheta \right] f^*(T_\vomega(\vx)) \mathrm{d} \vx\\
    &= \int q_\vtheta(\vx) \nabla_\vtheta E_\vtheta(\vx) f^*(T_\vomega(\vx)) \mathrm{d} \vx + \nabla_\vtheta \log Z_\vtheta \cdot  \int q_\vtheta(\vx) f^*(T_\vomega(\vx)) \mathrm{d} \vx\\
    &= \int q_\vtheta(\vx) \nabla_\vtheta E_\vtheta(\vx) f^*(T_\vomega(\vx)) \mathrm{d} \vx + \frac{\int \exp(- E_\vtheta(\vx)) (- \nabla_\vtheta E_\vtheta(\vx)) \mathrm{d} \vx}{Z_\vtheta} \cdot  \int q_\vtheta(\vx) f^*(T_\vomega(\vx)) \mathrm{d} \vx\\
    &= \bb{E}_{q_\vtheta(\vx)}[\nabla_\vtheta E_\vtheta(\vx) f^*(T_\vomega(\vx))] - \bb{E}_{q_\vtheta(\vx)}[\nabla_\vtheta E_\vtheta(\vx)] \cdot \bb{E}_{q_\vtheta(\vx)}[f^*(T_\vomega(\vx))]
\end{align*}
We can simply use $i.i.d.$ samples from $q_\vtheta$ to get an unbiased estimation of the first term, and we introduce the following lemma for estimating the second term $\bb{E}_{q_\vtheta(\vx)}[\nabla_\vtheta E_\vtheta(\vx)] \cdot \bb{E}_{q_\vtheta(\vx)}[f^*(T_\vomega(\vx))]$.
\begin{lemma}\label{lemma:product-of-expectations}
Let $P$ denote a probability distribution over sample space $\gX$ and $g, h: \gX \to \bb{R}$ are functions. The following empirical estimator is unbiased with respect to $\bb{E}_P[g(X)] \cdot \bb{E}_P[h(X)]$:
\begin{align}
    \frac{1}{nm} \sum_{i=1}^{n} g(x_i) \cdot \sum_{i=1}^{m} h(y_j)
\end{align}
where $\vx_{1:n} \sim P$, $\vy_{1:m} \sim P$ are two independent sets of \textit{i.i.d.} samples from $P$.
\end{lemma}
\begin{proof}
Let $X_{1:n}$ and $Y_{1:m}$ denote the random variables corresponding to the sampling process for $\vx_{1:n}$, $\vy_{1:m}$ respectively. Let $P_n$ and $P_m$ denote the probability distribution for $X_{1:n}$ and $Y_{1:m}$. Then we have: 
\begin{align}
   & \ \bb{E}_{P_n} \bb{E}_{P_m}\left[\frac{1}{nm} \sum_{i=1}^{n} g(X_i) \cdot \sum_{i=1}^{m} h(Y_j)\right] \\
  = & \ \bb{E}_{P_n} \left[\frac{1}{n} \sum_{i=1}^{n} g(X_i) \cdot \bb{E}_{P_m}\left[\frac{1}{m} \sum_{i=1}^{m} h(Y_j)\right]\right] \\
  = & \ \bb{E}_{P_n} \left[\frac{1}{n} \sum_{i=1}^{n} g(X_i) \cdot \bb{E}_P[h(X)]\right] \\
  = & \ \bb{E}_P[g(X)] \cdot \bb{E}_P[h(X)]
\end{align}
where the first equality comes from the law of total expectation and the fact that $X_{1:n}$ and $Y_{1:m}$ are independent. Therefore, the estimator is unbiased.
\end{proof}
Lemma~\ref{lemma:product-of-expectations} shows that we can use $i.i.d.$ samples from $q_\vtheta$ to get an unbiased estimation of the second term.
\end{proof}

\section{Proof of Theorem~\ref{theorem:new-f-bound}}\label{app:proof-new-f-bound}
\begin{theorem}
Let $P$ and $Q$ be two probability measures over the Borel $\sigma$-algebra on domain $\gX$ with densities $p$, $q$ and $P \ll Q$. Additionally, let $q$ be an energy-based distribution, $q(\vx) = \exp(-E(\vx)) / Z_q$.
For any class of functions $\gH$ mapping from $\gX$ to $\bb{R}$ such that the expectations in the following equation are finite and $\forall \vx \in \gX$,
$\gH$ contains an element $H(\vx) = \log\left(p(\vx) Z_q\right)$, then we have
\begin{align*}
    \DF(P \| Q) = \sup_{H \in \gH} &~\bb{E}_{p(\vx)} [f'(\exp(H(\vx) + E(\vx)))] - \\ 
    &~\bb{E}_{q(\vx)}[f^*(f'(\exp(H(\vx) + E(\vx))))]
\end{align*}
where the supreme is attained at $H^*(\vx) = \log\left(p(\vx) Z_q\right)$.
\end{theorem}
\begin{proof}
First, we define a function class $\hat{\gT} = \{f'(r) | r: \gX \to \R_{+}\}$. Since $\hat{\gT}$ is a subset of $\gT$ (in Lemma~\ref{lemma:f-bound}) and the optimal $T^\star = f'(p / q)$ is an element of $\hat{\gT}$, from Lemma~\ref{lemma:f-bound}, we have:
\begin{align}
    \DF(P \| Q) = \sup_{T \in \hat{\gT}} \bb{E}_{p(\vx)} [T(\vx)] - \bb{E}_{q(\vx)}[f^*(T(\vx))].
\end{align}
Without loss of generality, we can reparametrize all functions $T \in \hat{\gT}$ with $H$ such that $T(\vx) = f'(\exp(H(\vx) + E(\vx)))$ for all $\vx \in \gX$, since the transformations $f'(\cdot)$, $\exp(\cdot)$ and addition by $E(\vx)$ are all bijections (note that $f$ is strictly convex and differentiable). The optimal $H^\star(\vx)$ then satisfies:
\begin{align*}
    \forall \vx \in \gX,~f'(\exp(H^\star(\vx) + E(\vx))) = T^\star(\vx) = f'(p(\vx) / q(\vx))
\end{align*}
From the bijectivity of $f'(\cdot)$, we have:
\begin{align*}
     H^\star(\vx) & = \log p(\vx) - \log q(\vx) - E(\vx) \\
     & = \log p(\vx) + E(\vx) + \log Z_q - E(\vx) = \log (p(\vx) Z_q),
\end{align*}
\end{proof}

\section{Proof of Theorem~\ref{theorem:theta-gradient}}\label{app:proof-theta-gradient}
\begin{assumption}\label{assumption:continuous-1}
The function $p(\vx) f'(\exp(H_\vomega(\vx) + E_\vtheta(\vx)))$ and its partial derivative $\nabla_\vtheta f'(\exp(H_\vomega(\vx) + E_\vtheta(\vx)))$ are continuous w.r.t. $\vtheta$ and $\vx$.
\end{assumption}
\begin{assumption}\label{assumption:continuous-2}
The function $q_\vtheta(\vx) f^*(f'(\exp(H_\vomega(\vx) + E_\vtheta(\vx))))$ and its partial derivative $\nabla_\vtheta q_\vtheta(\vx) f^*(f'(\exp(H_\vomega(\vx) + E_\vtheta(\vx))))$ are continuous w.r.t. $\vtheta$ and $\vx$.
\end{assumption}

\begin{theorem}
For a $\vtheta$-parametrized energy-based model $q_\vtheta(\vx) = \frac{\exp(- E_\vtheta(\vx))}{Z_\vtheta}$ and a fixed $\vomega$, under Assumptions~\ref{assumption:continuous-1} and \ref{assumption:continuous-2}, the gradient of $\gL_{f\text{-EBM}}(\vtheta, \vomega)$ (the objective in Equation~(\ref{eq:minimax-f-EBM})) with respect to $\vtheta$ can be written as:
\begin{align*}
    \nabla_\vtheta \gL_{f\text{-EBM}}(\vtheta, \vomega) = &~\bb{E}_{p(\vx)}[\nabla_\vtheta f'(\exp(f_\vomega(\vx) + E_\vtheta(x)))] + \bb{E}_{q_\vtheta(\vx)}[F_{\vtheta, \vomega}(\vx) \nabla_\vtheta E_\vtheta(\vx)] - \\
    &~\bb{E}_{q_\vtheta(\vx)}[\nabla_\vtheta F_{\vtheta, \vomega}(\vx)] - \bb{E}_{q_\vtheta(\vx)}[\nabla_\vtheta E_\vtheta(\vx)] \cdot \bb{E}_{q_\vtheta(\vx)}[F_{\vtheta, \vomega}(\vx)]
\end{align*}
where $F_{\vtheta,\vomega}(\vx)=f^*(f'(\exp(H_\vomega(\vx) + E_\vtheta(\vx))))$. When we can obtain i.i.d. samples from $p$ and $q_\vtheta$, we can get an unbiased estimation of the gradient.
\end{theorem}
\begin{proof}
Under Assumption \ref{assumption:continuous-1}, with Leibniz integral rule, for the first term in Equation~(\ref{eq:minimax-f-EBM}), we have:
\begin{align*}
    \nabla_\vtheta \bb{E}_{p(\vx)}[f'(\exp(H_\vomega(\vx) + E_\vtheta(\vx)))] = \int p(\vx) \nabla_\vtheta f'(\exp(H_\vomega(\vx) + E_\vtheta(\vx))) \mathrm{d} \vx = \bb{E}_{p(\vx)}[\nabla_\vtheta f'(\exp(f_\vomega(\vx) + E_\vtheta(x)))]
\end{align*}
For notational simplicity, let us use $F_{\vtheta,\vomega}(\vx)$ to denote $f^*(f'(\exp(H_\vomega(\vx) + E_\vtheta(\vx))))$. Under Assumption \ref{assumption:continuous-2}, with Leibniz integral rule, for the second term, we have:
\begin{align*}
    &- \nabla_\vtheta \bb{E}_{q_\vtheta(\vx)} [f^*(f'(\exp(H_\vomega(\vx) + E_\vtheta(\vx))))]\\
    =&- \int \nabla_\vtheta (q_\vtheta(\vx) F_{\vtheta, \vomega}(\vx)) \mathrm{d} \vx\\
    =&- \int F_{\vtheta, \vomega}(\vx) \nabla_\vtheta q_\vtheta(\vx) \mathrm{d} \vx - \int q_\vtheta(\vx) \nabla_\vtheta F_{\vtheta, \vomega}(\vx) \mathrm{d} \vx\\
    =&- \int q_\vtheta(\vx) F_{\vtheta, \vomega}(\vx) \frac{\nabla_\vtheta q_\vtheta(\vx)}{q_\vtheta(\vx)} \mathrm{d} \vx - \int q_\vtheta(\vx) \nabla_\vtheta F_{\vtheta, \vomega}(\vx) \mathrm{d} \vx\\
    =&- \int q_\vtheta(\vx) F_{\vtheta,\vomega}(\vx) \nabla_\vtheta \log q_\vtheta(\vx) \mathrm{d} \vx - \int q_\vtheta(\vx) \nabla_\vtheta F_{\vtheta, \vomega}(\vx) \mathrm{d} \vx\\
    =&- \int q_\vtheta(\vx) F_{\vtheta,\vomega}(\vx) \left[- \nabla_\vtheta E_\vtheta(\vx) - \nabla_\vtheta \log Z_\vtheta \right] \mathrm{d} \vx - \int q_\vtheta(\vx) \nabla_\vtheta F_{\vtheta, \vomega}(\vx) \mathrm{d} \vx\\
    =& \int q_\vtheta(\vx) F_{\vtheta, \vomega}(\vx) \nabla_\vtheta E_\vtheta(\vx) \mathrm{d} \vx - \int q_\vtheta(\vx) \nabla_\vtheta F_{\vtheta, \vomega}(\vx) \mathrm{d} \vx + \nabla_\vtheta \log Z_\vtheta \cdot \int q_\vtheta(\vx) F_{\vtheta, \vomega}(\vx) \mathrm{d} \vx\\
    =& \int q_\vtheta(\vx) F_{\vtheta, \vomega}(\vx) \nabla_\vtheta E_\vtheta(\vx) \mathrm{d} \vx - \int q_\vtheta(\vx) \nabla_\vtheta F_{\vtheta, \vomega}(\vx) \mathrm{d} \vx + \frac{\int \exp(- E_\vtheta(\vx)) (- \nabla_\vtheta E_\vtheta(\vx)) \mathrm{d} \vx}{Z_\vtheta} \cdot \int q_\vtheta(\vx) F_{\vtheta, \vomega}(\vx) \mathrm{d} \vx\\
    =&~\bb{E}_{q_\vtheta(\vx)}[F_{\vtheta, \vomega}(\vx) \nabla_\vtheta E_\vtheta(\vx)] - \bb{E}_{q_\vtheta(\vx)}[\nabla_\vtheta F_{\vtheta, \vomega}(\vx)] - \bb{E}_{q_\vtheta(\vx)}[\nabla_\vtheta E_\vtheta(\vx)] \cdot \bb{E}_{q_\vtheta(\vx)}[F_{\vtheta, \vomega}(\vx)]
\end{align*}
Similar to the proof in Appendix~\ref{app:proof-direct-febm}, we can use \emph{i.i.d.} samples from $p$ and $q_\vtheta$ to get an unbiased estimation of the gradient.
\end{proof}

\section{Proof for the Local Convergence of $f$-EBM}\label{app:proof-local-convergence}
\subsection{Non-linear Dynamical Systems}\label{sec:non-linear-system}
In this section, we present a brief introduction of non-linear dynamical system theory \cite{hassan1996khalil}. For a comprehensive description, please refer to \cite{hassan1996khalil,nagarajan2017gradient}. We further generalize some of the theories which will be useful for establishing the local convergence property of $f$-EBM later.

Consider a system consisting of variables $\vphi \in \Phi \subseteq \bb{R}^n$ whose time derivative is defined by the vector field $v(\vphi)$:
\begin{align}
    \dot \vphi = v(\vphi) \label{eq:system}
\end{align}
where $v: \Phi \to \bb{R}^n$ is a locally Lipschitz mapping from a domain $\Phi$ into $\bb{R}^n$.

Suppose $\ophi$ is an equilibrium point of the system in Equation~(\ref{eq:system}), \emph{i.e.}, $v(\ophi) = 0$. Let $\vphi_t$ denote the state of the system at time $t$. To begin with, we introduce the following definition to characterize the stability of $\ophi$:
\begin{definition}[Definition 4.1 in \cite{hassan1996khalil}] \label{def:stability}
The equilibrium point $\ophi$ for the system defined in Equation~(\ref{eq:system}) is
\begin{itemize}
    \item stable if for each $\epsilon > 0$, there is $\delta = \delta(\epsilon) > 0$ such that $$\|\vphi_0 - \ophi\| < \delta \Longrightarrow \forall t \geq 0,~ \|\vphi_t - \ophi\| < \epsilon.$$
    \item unstable if not stable.
    \item asymptotically stable if it is stable and $\delta > 0$ can be chosen such that $$\|\vphi_0 - \ophi\| < \delta \Longrightarrow \lim_{t \to \infty} \vphi_t = \ophi.$$
    \item exponentially stable if it is asymptotically stable and $\delta, k, \lambda > 0$ can be chosen such that:
    $$\|\vphi_0 - \ophi\| < \delta \Longrightarrow \|\vphi_t\| \leq k \|\vphi_0\| \exp(-\lambda t).$$
\end{itemize}
\end{definition}
The system is stable if for any value of $\epsilon$, we can find a value of $\delta$ (possibly dependent on $\epsilon$), such that a trajectory starting in a $\delta$ neighborhood of the equilibrium point will never leave the $\epsilon$ neighborhood of the equilibrium point. However, such a system may either converge to the equilibrium point or orbit within the $\epsilon$ ball. By contrast, asymptotic stability is a stronger notion of stability in the sense that trajectories starting in a $\delta$ neighborhood of the equilibrium will converge to the equilibrium point in the limit $t \to \infty$. Moreover, if $\ophi$ is asymptotically stable, we call the algorithm obtained by iteratively applying the updates in Equation~(\ref{eq:system}) \emph{locally convergent} to $\ophi$. If $\ophi$ is exponentially stable, we call the corresponding algorithm \emph{linearly convergent} to $\ophi$.

Now we introduce the following theorem, which is central for studying the asymptotic stability of a system:
\begin{theorem}[Theorem 4.15 in \cite{hassan1996khalil}]\label{theorem:equilibrium-hurwitz}
Let $\ophi$ be an equilibrium point for the non-linear system
\begin{align}
    \dot \vphi = v(\vphi)\label{eq:theorem-system}
\end{align}
where $v: \Phi \to \bb{R}^n$ is continuously differentiable and $\Phi$ is a neighborhood of $\ophi$. Let $\vJ$ be the Jacobian of the system in Equation~(\ref{eq:theorem-system}) at the equilibrium point:
\begin{align}
    \vJ = \frac{\partial v(\vphi)}{\partial \vphi}\at[\Bigg]{\vphi = \ophi}
\end{align}
Then, we have:
\begin{itemize}
    \item The equilibrium point $\ophi$ is asymptotically stable and  exponentially stable if $\vJ$ is a Hurwitz matrix, i.e., $\mathrm{Re}(\lambda) < 0$ for all eigenvalues $\lambda$ of $\vJ$.
    \item The equilibrium point $\ophi$ is unstable if $\mathrm{Re}(\lambda) > 0$ for one or more of the eigenvalues of $\vJ$. 
\end{itemize}
\end{theorem}
\begin{proof}
See proof for Theorem 4.7, Theorem 4.15 and Corollary 4.3 in \cite{hassan1996khalil}.
\end{proof}
With Theorem~\ref{theorem:equilibrium-hurwitz}, the stability of an equilibrium point can be analyzed by examining if all the eigenvalues of the Jacobian $v'(\vphi)\at{\vphi=\ophi}$ have strictly negative real part.

In the following, we will use $\lambda_\mathrm{max}(\cdot)$ and $\lambda_\mathrm{min}(\cdot)$ to denote the largest and smallest eigenvalues of a non-zero positive semi-definite matrix. 
Now we introduce the following theorem to upper bound the real part of the eigenvalues of a matrix with a specific form:
\begin{theorem}\label{theorem:eigenvalue-bound}
Suppose $\vJ \in \bb{R}^{(m+n)\times(m+n)}$ is of the following form:
\begin{align*}
    \vJ = 
    \begin{bmatrix}
    \bm{0} & \vP\\
    -\vP^\top & -\vQ
    \end{bmatrix}
\end{align*}
where $\vQ \in \bb{R}^{n \times n}$ is a symmetric real positive definite matrix and $\vP^\top \in \bb{R}^{n \times m}$ is a full column rank matrix. Then, for every eigenvalue $\lambda$ of $\vJ$, $\mathrm{Re}(\lambda) < 0$. More precisely, we have:
\begin{itemize}
    \item When $\mathrm{Im}(\lambda) = 0$, 
    $$\mathrm{Re}(\lambda) \leq -\frac{\lambda_\mathrm{min}(\vQ) \lambda_\mathrm{min}(\vP \vP^\top)}{\lambda_\mathrm{min}(\vQ)\lambda_\mathrm{max}(\vQ) + \lambda_\mathrm{min} (\vP \vP^\top)}$$
    \item When $\mathrm{Im}(\lambda) \neq 0$, $$\mathrm{Re}(\lambda) \leq - \frac{\lambda_\mathrm{min}(\vQ)}{2}$$
\end{itemize}
\end{theorem}
\begin{proof}
See Lemma G.2 in \cite{nagarajan2017gradient}.
\end{proof}
This theorem is useful for proving the local convergence of GANs. To prove the stability of $f$-EBM in the following sections, we need the following generalized theorem:
\begin{theorem}\label{theorem:new-eigenvalue-bound}
Suppose $\vJ \in \bb{R}^{(m+n)\times(m+n)}$ is of the following form:
\begin{align*}
    \vJ = 
    \begin{bmatrix}
    -\vS & \vP\\
    -\vP^\top & -\vQ
    \end{bmatrix}
\end{align*}
where $\vS \in \bb{R}^{m \times m}$ is a symmetric real positive semi-definite matrix, $\vQ \in \bb{R}^{n \times n}$ is a symmetric real positive definite matrix, $\vP^\top \in \bb{R}^{n \times m}$ is a full column rank matrix. Then, for every eigenvalue $\lambda$ of $\vJ$, $\mathrm{Re}(\lambda) < 0$. More precisely, we have:
\begin{itemize}
    \item When $\mathrm{Im}(\lambda) = 0$,
    \begin{align*}
    \lambda_1 < - \frac{\lambda_\mathrm{min}(\vQ) \lambda_\mathrm{min}(\vP \vP^\top)}{\lambda_\mathrm{min}(\vQ) \lambda_\mathrm{max}(\vS, \vQ) + \lambda_\mathrm{min}(\vP \vP^\top)} < 0
    \end{align*}
    where $\lambda_\mathrm{max}(\vS, \vQ) = \max(\lambda_\mathrm{max}(\vS), \lambda_\mathrm{max}(\vQ))$.
    \item When $\mathrm{Im}(\lambda) \neq 0$,
    \begin{align*}
    \lambda_1 \leq -\frac{\lambda_\mathrm{min}(\vS) + \lambda_\mathrm{min}(\vQ)}{2} < 0
    \end{align*}
\end{itemize}
\end{theorem}
\begin{proof}
We prove this theorem in a similar way to the proof for Theorem~\ref{theorem:eigenvalue-bound}.
Consider the following generic eigenvector equation:
\begin{align*}
    \begin{bmatrix}
    - \vS & \vP\\
    - \vP^\top &- \vQ
    \end{bmatrix}
    \begin{bmatrix}
    \va_1 + i \va_2\\
    \vb_1 + i \vb_2
    \end{bmatrix}
    = (\lambda_1 + i \lambda_2)
    \begin{bmatrix}
    \va_1 + i \va_2\\
    \vb_1 + i \vb_2
    \end{bmatrix}
\end{align*}
where $\va_i, \vb_i, \lambda_i$ are all real valued and the vector is normalized, \emph{i.e.}, $\|\va_1\|^2 + \|\va_2\|^2 + \|\vb_1\|^2 + \|\vb_2\|^2 = 1$. The above equation can be rewritten as:
\begin{align*}
    \begin{bmatrix}
        - \vS \va_1 + \vP \vb_1 + i (- \vS \va_2 + \vP \vb_2)\\
        - \vP^\top \va_1 - \vQ \vb_1 + i (- \vP^\top \va_2 - \vQ \vb_2)
\end{bmatrix} = 
\begin{bmatrix}
\lambda_1 \va_1 - \lambda_2 \va_2 + i (\lambda_1 \va_2 + \lambda_2 \va_1)\\
\lambda_1 \vb_1 - \lambda_2 \vb_2 + i (\lambda_1 \vb_2 + \lambda_2 \vb_1)
\end{bmatrix}
\end{align*}
By equating the real and complex parts, we get:
\begin{align}
    - \vS \va_1 + \vP \vb_1 =& \lambda_1 \va_1 - \lambda_2 \va_2 \label{eq:bound-eq-1}\\
    - \vP^\top \va_1 - \vQ \vb_1 =& \lambda_1 \vb_1 - \lambda_2 \vb_2 \label{eq:bound-eq-2}\\
    - \vS \va_2 + \vP \vb_2 =& \lambda_1 \va_2 + \lambda_2 \va_1 \label{eq:bound-eq-3}\\
    - \vP^\top \va_2 - \vQ \vb_2 =& \lambda_1 \vb_2 + \lambda_2 \vb_1 \label{eq:bound-eq-4}
\end{align}
By multiplying the Equations~(\ref{eq:bound-eq-1}), (\ref{eq:bound-eq-2}), (\ref{eq:bound-eq-3}), (\ref{eq:bound-eq-4}) by $\va_1^\top, \vb_1^\top, \va_2^\top, \vb_2^\top$ and adding them together, we get:
\begin{align*}
    & - \va_1^\top \vS \va_1 + \va_1^\top \vP \vb_1 - \vb_1^\top \vP^\top \va_1 - \vb_1^\top \vQ \vb_1 - \va_2^\top \vS \va_2 + \va_2^\top \vP \vb_2 - \vb_2^\top \vP^\top \va_2 - \vb_2^\top \vQ \vb_2\\
    = & \lambda_1 \va_1^\top \va_1 - \lambda_2 \va_1^\top \va_2 + \lambda_1 \vb_1^\top \vb_1 - \lambda_2 \vb_1^\top \vb_2 + \lambda_1 \va_2^\top \va_2 + \lambda_2 \va_2^\top \va_1 + \lambda_1 \vb_2^\top \vb_2 + \lambda_2 \vb_2^\top \vb_1
\end{align*}
which simplifies to
\begin{align*}
    - \va_1^\top \vS \va_1 - \va_2^\top \vS \va_2 - \vb_1^\top \vQ \vb_1  - \vb_2^\top \vQ \vb_2 = \lambda_1 (\va_1^\top \va_1 + \vb_1^\top \vb_1 + \va_2^\top \va_2 + \vb_2^\top \vb_2) = \lambda_1
\end{align*}
Because $S \succeq \bm{0}, Q \succ \bm{0}$, $- \va_1^\top \vS \va_1 - \va_2^\top \vS \va_2 - \vb_1^\top \vQ \vb_1  - \vb_2^\top \vQ \vb_2 \leq 0$, where the equality holds only if $\vb_1 = \bm{0}$, $\vb_2 = \bm{0}$ (as well as $- \va_1^\top \vS \va_1 - \va_2^\top \vS \va_2 = 0$). Next we show that $\vb_1 = \bm{0}$, $\vb_2 = \bm{0}$ is contradictory with the condition that $\vP^\top$ is a full column rank matrix. First, when $- \va_1^\top \vS \va_1 - \va_2^\top \vS \va_2 - \vb_1^\top \vQ \vb_1  - \vb_2^\top \vQ \vb_2 = 0$, we have $\lambda_1 = 0$. Applying these to Equations~(\ref{eq:bound-eq-2}) and (\ref{eq:bound-eq-4}), we get $\vP^\top \va_1 = 0$ and $\vP^\top \va_2 = 0$. Since one of $\va_1$ and $\va_2$ is non-zero (otherwise the eigenvector is zero), this implies $\vP^\top$ is not a full column rank matrix, which is contradictory with the condition that $\vP^\top$ has full column rank. Consequently, $- \va_1^\top \vS \va_1 - \va_2^\top \vS \va_2 - \vb_1^\top \vQ \vb_1  - \vb_2^\top \vQ \vb_2 = \lambda_1 < 0$.

Now we proceed to get a tighter upper bound on the real part of the eigenvalue $\lambda_1$. By multiplying Equations~(\ref{eq:bound-eq-1}) and (\ref{eq:bound-eq-3}) by $- \va_2^\top$ and $\va_1^\top$ and adding them together, we get:
\begin{align*}
    \va_2^\top \vS \va_1 - \va_2^\top \vP \vb_1 - \va_1^\top \vS \va_2 + \va_1^\top \vP \vb_2 = - \lambda_1 \va_2^\top \va_1 + \lambda_2 \va_2^\top \va_2 + \lambda_1 \va_1^\top \va_2 + \lambda_2 \va_1^\top \va_1
\end{align*}
which simplifies to:
\begin{align}
     - \va_2^\top \vP \vb_1 + \va_1^\top \vP \vb_2 = \lambda_2 \va_2^\top \va_2 + \lambda_2 \va_1^\top \va_1 \label{eq:bound-eq-5}
\end{align}
Similarly, by multiplying Equations~(\ref{eq:bound-eq-2}) and (\ref{eq:bound-eq-4}) by $- \vb_2^\top$ and $\vb_1^\top$ and adding them together, we get:
\begin{align}
    \vb_2^\top \vP^\top \va_1 + \vb_2^\top \vQ \vb_1 - \vb_1^\top \vP^\top \va_2 - \vb_1^\top \vQ \vb_2 = - \lambda_1 \vb_2^\top \vb_1 + \lambda_2 \vb_2^\top \vb_2 + \lambda_1 \vb_1^\top \vb_2 + \lambda_2 \vb_1^\top \vb_1 
\end{align}
which simplifies to:
\begin{align}
    \vb_2^\top \vP^\top \va_1  - \vb_1^\top \vP^\top \va_2 = \lambda_2 \vb_2^\top \vb_2 + \lambda_2 \vb_1^\top \vb_1 \label{eq:bound-eq-6}
\end{align}
With Equation~(\ref{eq:bound-eq-5}) and Equation~(\ref{eq:bound-eq-6}), we have:
\begin{align*}
    \lambda_2 (\|\va_1\|^2 + \|\va_2\|^2) = \lambda_2 (\|\vb_1\|^2 + \|\vb_2\|^2)
\end{align*}
which implies that either $\|\va_1\|^2 + \|\va_2\|^2 = \|\vb_1\|^2 + \|\vb_2\|^2 = 1/2$ or $\lambda_2 = 0$.

In the first case ($\lambda_2 \neq 0$ and $\|\va_1\|^2 + \|\va_2\|^2 = \|\vb_1\|^2 + \|\vb_2\|^2 = 1/2$), from $- \va_1^\top \vS \va_1 - \va_2^\top \vS \va_2 - \vb_1^\top \vQ \vb_1  - \vb_2^\top \vQ \vb_2 = \lambda_1$, we get an upper bound:
\begin{align*}
    \lambda_1 \leq -\frac{\lambda_\mathrm{min}(\vS) + \lambda_\mathrm{min}(\vQ)}{2}
\end{align*}
This upper bound is strictly negative since $\lambda_\mathrm{min}(\vS) \geq 0$ and $\lambda_\mathrm{min}(\vQ) > 0$.

Now we introduce the following lemma which is useful for deriving the upper bound of $\lambda_1$ in the second case:
\begin{lemma}\label{lemma:eigenvalue-bound}
Let $\vS \succeq \bm{0}$ and $\vQ \succeq \bm{0}$ be two real symmetric matrices. If $\va^\top \vS \va + \vb^\top \vQ \vb = c$, then $\va^\top \vS^\top \vS \va + \vb^\top \vQ^\top \vQ \vb \in [c \cdot \min(\lambda_\mathrm{min}(\vS), \lambda_\mathrm{min}(\vQ)), c \cdot \max(\lambda_\mathrm{max}(\vS), \lambda_\mathrm{max}(\vQ))]$.
\end{lemma}
\begin{proof}
Let $\vS = \vU_\vS \bm{\Lambda}_\vS \vU_\vS^\top$ and $\vQ = \vU_\vQ \bm{\Lambda}_\vQ \vU_\vQ^\top$ be the eigenvalue decompositions of $\vS$ and $\vQ$. Then, we have
\begin{align*}
    c = \va^\top \vS \va + \vb^\top \vQ \vb =
    \va^\top \vU_\vS \bm{\Lambda}_\vS \vU_\vS^\top \va + \vb^\top \vU_\vQ \bm{\Lambda}_\vQ \vU_\vQ^\top \vb 
    = \vx^\top \bm{\Lambda}_\vS \vx + \vy^\top \bm{\Lambda}_\vQ \vy
\end{align*}
where $\vx = \vU_\vS^\top \va$ and $\vy = \vU_\vQ^\top \vb$. Therefore, we have:
\begin{align*}
    c = \sum_i x_i^2 \lambda_\vS^i + \sum_j y_j^2 \lambda_\vQ^j
\end{align*}
Similarly, we have:
\begin{align*}
    \va^\top \vS^\top \vS \va + \vb^\top \vQ^\top \vQ \vb =& \va^\top \vU_\vS \bm{\Lambda}_\vS \vU_\vS^\top \vU_\vS \bm{\Lambda}_\vS \vU_\vS^\top \va + 
    \vb^\top \vU_\vQ \bm{\Lambda}_\vQ \vU_\vQ^\top \vU_\vQ \bm{\Lambda}_\vQ \vU_\vQ^\top \vb\\
    =& \sum_i x_i^2 (\lambda_\vS^i)^2 + \sum_j y_j^2 (\lambda_\vQ^j)^2
\end{align*}
which differs from $c$ by a multiplicative factor within $[\min(\lambda_\mathrm{min}(\vS), \lambda_\mathrm{min}(\vQ)), \max(\lambda_\mathrm{max}(\vS), \lambda_\mathrm{max}(\vQ))]$.
\end{proof}

In the second case, the imaginary part of the eigenvalue is zero ($\lambda_2 = 0$), which implies the imaginary part of the eigenvector must also be zero, $\va_2 = \vb_2 = \bm{0}$. Applying this to Equations~(\ref{eq:bound-eq-1}), (\ref{eq:bound-eq-2}), (\ref{eq:bound-eq-3}), (\ref{eq:bound-eq-4}), we get:
\begin{align*}
    - \vS \va_1 + \vP \vb_1 =& \lambda_1 \va_1 \\
    - \vP^\top \va_1 - \vQ \vb_1 =& \lambda_1 \vb_1
\end{align*}
Rearranging the above equations, we get:
\begin{align*}
    \vP \vb_1 =& (\lambda_1 \vI + \vS) \va_1\\
    - \vP^\top \va_1 =& (\lambda_1 \vI + \vQ) \vb_1
\end{align*}
Squaring both sides of the equations, we get:
\begin{align*}
    \vb_1^\top \vP^\top \vP \vb_1 = \va_1^\top(\lambda_1^2 \vI + 2 \lambda_1 \vS + \vS^\top \vS) \va_1 = \lambda_1^2 \|\va_1\|^2 + 2\lambda_1 \va_1^\top \vS \va_1 + \va_1^\top \vS^\top \vS \va_1 \\
    \va_1^\top \vP \vP^\top \va_1 = \vb_1^\top (\lambda_1^2 \vI + 2 \lambda_1 \vQ + \vQ^\top \vQ) \vb_1 = \lambda_1^2 \|\vb_1\|^2 + 2\lambda_1 \vb_1^\top \vQ \vb_1 + \vb_1^\top \vQ^\top \vQ \vb_1
\end{align*}
Summing these two equations together, using the fact that $- \va_1^\top \vS \va_1 - \vb_1^\top \vQ \vb_1 = \lambda_1$ and $\|\va_1\|^2 + \|\vb_1\|^2 = 1$, we get:
\begin{align*}
    \vb_1^\top \vP^\top \vP \vb_1 + \va_1^\top \vP \vP^\top \va_1 =& \lambda_1^2 (\|\va_1\|^2 + \|\vb_1\|^2) - 2 \lambda_1^2 + \va_1^\top \vS^\top \vS \va_1 + \vb_1^\top \vQ^\top \vQ \vb_1 \\
    =& - \lambda_1^2 + \va_1^\top \vS^\top \vS \va_1 + \vb_1^\top \vQ^\top \vQ \vb_1
\end{align*}
Let $\lambda_\mathrm{max}(\vS, \vQ)$ denote $\max(\lambda_\mathrm{max}(\vS), \lambda_\mathrm{max}(\vQ))$. With Lemma~\ref{lemma:eigenvalue-bound}, we have:
\begin{align}\label{eq:bound-ineq-1}
    \vb_1^\top \vP^\top \vP \vb_1 + \va_1^\top \vP \vP^\top \va_1 \leq - \lambda_1^2 - \lambda_1 \lambda_\mathrm{max}(\vS, \vQ)
\end{align}
Furthermore, we have:
\begin{align*}
    - \lambda_1 = \va_1^\top \vS \va_1 + \vb_1^\top \vQ \vb_1 \geq \lambda_\mathrm{min}(\vQ)\|\vb_1\|^2 = \lambda_\mathrm{min}(\vQ)(1 - \|\va_1\|^2)
    \Longrightarrow \|\va_1\|^2 \geq 1 + \frac{\lambda_1}{\lambda_\mathrm{min}(\vQ)}
\end{align*}
Note that $\lambda_1$ can either satisfy $\lambda_1 \leq - \lambda_\mathrm{min}(\vQ)$ or $- \lambda_\mathrm{min}(\vQ) < \lambda_1 < 0$. In the first scenario, we already obtain an upper bound: $\lambda_1 \leq - \lambda_\mathrm{min}(\vQ)$. So we focus on deriving an upper bound for the second scenario. From Equation~(\ref{eq:bound-ineq-1}), we have:
\begin{align}
    & - \lambda_1^2 - \lambda_1 \lambda_\mathrm{max}(\vS, \vQ)
    \geq 
    \va_1^\top \vP \vP^\top \va_1
    \geq
    \lambda_\mathrm{min}(\vP \vP^\top)\|\va_1\|^2 
    \geq
    \lambda_\mathrm{min}(\vP \vP^\top)\left(1 + \frac{\lambda_1}{\lambda_\mathrm{min}(\vQ)}\right) \nonumber\\
    \Longrightarrow
    & - \lambda_1 \left(\lambda_1 + \lambda_\mathrm{max}(\vS, \vQ) +  \frac{\lambda_\mathrm{min}(\vP \vP^\top)}{\lambda_\mathrm{min}(\vQ)}\right) \geq \lambda_\mathrm{min}(\vP \vP^\top) \nonumber\\
    \Longrightarrow
    & - \lambda_1 \left(\lambda_\mathrm{max}(\vS, \vQ) +  \frac{\lambda_\mathrm{min}(\vP \vP^\top)}{\lambda_\mathrm{min}(\vQ)}\right) > \lambda_\mathrm{min}(\vP \vP^\top) \label{eq:bound-ineq-2}
\end{align}
where the last implication uses the fact that $\lambda_1 < 0$. From Equation~(\ref{eq:bound-ineq-2}), we get an upper bound for $\lambda_1$:
\begin{align*}
    \lambda_1 < - \lambda_\mathrm{min}(\vQ) \frac{ \lambda_\mathrm{min}(\vP \vP^\top)}{\lambda_\mathrm{min}(\vQ) \lambda_\mathrm{max}(\vS, \vQ) + \lambda_\mathrm{min}(\vP \vP^\top)}
\end{align*}
Since the fraction in above equation lies in $(0, 1)$, we will use this as the upper bound for the second case ($\lambda_2 = 0$). Also note that this upper bound is strictly negative, since $\lambda_\mathrm{max}(\vS, \vQ) \geq \lambda_\mathrm{min}(\vQ) > 0$ and $\lambda_\mathrm{min}(\vP \vP^\top) > 0$.

\end{proof}

\subsection{Notations and Setup}\label{sec:notation-setup}
First, we can reformulate the minimax game defined in Equation~(\ref{eq:minimax-f-EBM}) as:
\begin{align}
    \min_\vtheta \max_\vomega V(\vtheta, \vomega) = \min_\vtheta \max_\vomega \bb{E}_{p(\vx)} [A(H_\vomega(\vx) + E_\vtheta(\vx))] - \bb{E}_{q_\vtheta(\vx)}[B(H_\vomega(\vx) + E_\vtheta(\vx))]\label{eq:new-minimax-game}
\end{align}
where the functions $A(u)$ and $B(u)$ are defined as:
\begin{align}
    A(u) = f'(\exp(u)),~B(u) = f^*(f'(\exp(u))) - f^*(f'(1)) \label{app:AB-def}
\end{align}
with $B(0) = 0$. Since $f^*(f'(1))$ is a constant, and $\bb{E}_{q_\vtheta(\vx)}[B(H_\vomega(\vx) + E_\vtheta(\vx))] = \bb{E}_{q_\vtheta(\vx)}[f^*(f'(\exp(H_\vomega(\vx) + E_\vtheta(\vx))))] - f^*(f'(1))$, the above formulation is equivalent to the original $f$-EBM minimax game in Equation~(\ref{eq:minimax-f-EBM}) (up to a constant that does not depend on $\vtheta$ and $\vomega$).

We have the following theorem to characterize the properties of functions $A(u)$ and $B(u)$:
\begin{theorem}\label{theorem:AB}
For any $f$-divergence with closed, strictly convex and third-order differentiable generator function $f$, functions $A(u), B(u)$ defined as:
$$A(u) = f'(\exp(u)),~B(u) = f^*(f'(\exp(u))) - f^*(f'(1))$$
satisfy the following properties:
\begin{itemize}
    \item $A'(0) = B'(0) = f''(1) > 0$
    \item $A''(0) - B''(0) = - f''(1) < 0$
    \item $A''(0) - B''(0) + 2 B'(0) = f''(1) = B'(0) > 0$
\end{itemize}
where $A'$, $B'$ are first-order derivative of $A$, $B$; $A''$, $B''$ are second-order derivative of $A$, $B$.
\end{theorem}
\begin{proof}
First, we introduce the following lemma for convex conjugates and subgradients:
\begin{lemma}\label{lemma:conjugate-subgradient}
If $f$ is closed and convex, then we have:
$$y \in \partial f(x) \Longleftrightarrow x \in \partial f^*(y)$$
\end{lemma}
\begin{proof}
If $y \in \partial f(x)$, then $f^*(y) = \sup_u (yu - f(u)) = yx - f(x)$. Therefore, we have:
\begin{align*}
    f^*(v) =& \sup_u vu - f(u) \\
    \geq & vx - f(x)\\
    =& x(v - y) - f(x) + xy\\
    =& f^*(y) + x(v-y)
\end{align*}
Because this holds for all $v$, we have $x \in \partial f^*(y)$. The reverse implication $x \in \partial f^*(y) \Longrightarrow y \in \partial f(x)$ follows from $f^{**} = f$.
\end{proof}
Let us use $g(y)$ to denote the conjugate function, $g(y) = f^*(y)$. Next we have the following lemma for the gradient of conjugate:
\begin{lemma}\label{lemma:gradient-conjugate}
If $f$ is strictly convex and differentiable, then we have:
$g'(y) = \argmax_x~(yx - f(x)),~g'(f'(u)) = u$.
\end{lemma}
\begin{proof}
Since $f$ is strictly convex, $x$ maximizes $yx - f(x)$ if and only if $y \in \partial f(x)$. With Lemma~\ref{lemma:conjugate-subgradient}, we know that $$y \in \partial f(x) \Longleftrightarrow x \in \partial f^*(y) = \{g'(y)\}$$
Therefore $g'(y) = \argmax_x~(yx - f(x))$. Then we have:
\begin{align*}
    g'(f'(u)) = \argmax_x~f'(u) x - f(x) \defeq \argmax_x~h(x)
\end{align*}
Since $h'(x) = f'(u) - f'(x), h''(x) = - f''(x) < 0$, we have $\argmax_x~f'(u) x - f(x) = u$.
\end{proof}
Now we are ready to proof the properties of functions $A(u)$ and $B(u)$. Recall that $A(u) = f'(\exp(u)),~B(u) = f^*(f'(\exp(u))) - f^*(f'(1))$, with Lemma~\ref{lemma:gradient-conjugate}, we have:
\begin{align*}
    A'(u) =& f''(\exp(u))\exp(u)\\
    A''(u) =& f'''(\exp(u))\exp(2u) + f''(\exp(u))\exp(u)\\
    B'(u) =& g'(f'(\exp(u)))f''(\exp(u))\exp(u)\\
    =& f''(\exp(u))\exp(2u)\\
    B''(u) =& f'''(\exp(u))\exp(3u) + 2f''(\exp(u))\exp(2u)
\end{align*}
Then we have:
\begin{align*}
    A'(0) =& f''(1)\\
    A''(0) =& f'''(1) + f''(1)\\
    B'(0) =& f''(1)\\
    B''(0) =& f'''(1) + 2 f''(1)
\end{align*}
Since $f$ is strictly convex with $\forall x \in \mathrm{dom}(f), f''(x)>0$, we have:
\begin{align*}
    &A'(0) = B'(0) = f''(1) > 0\\
    &A''(0) - B''(0) = - f''(1) < 0\\
    &A''(0) - B''(0) + 2 B'(0) = f''(1) = B'(0) > 0
\end{align*}
\end{proof}
Some examples of Theorem~\ref{theorem:AB} can be found in Table~\ref{tab:f-divergences-AB}.

\begin{table}[h!bt]
\begin{center}
\scalebox{1.}{%
\begin{tabular}{lllll}
\toprule
Name & $A'(0)$ & $B'(0)$ & $A''(0)$ & $B''(0)$\\ \midrule
Kullback-Leibler
& $1$
& $1$
& $0$
& $1$\\
Reverse Kullback-Leibler
& $1$
& $1$
& $-1$
& $0$\\
Pearson $\chi^2$
& $2$
& $2$
& $2$
& $4$\\
Neyman $\chi^2$
& $2$
& $2$
& $-4$
& $-2$\\
Squared Hellinger
& $\frac{1}{2}$
& $\frac{1}{2}$
& $-\frac{1}{4}$
& $\frac{1}{4}$\\
Jensen-Shannon
& $\frac{1}{2}$
& $\frac{1}{2}$
& $-\frac{1}{4}$
& $\frac{1}{4}$\\
$\alpha$-divergence ($\alpha \notin \{0,1\}$)
& $1$
& $1$
& $\alpha - 1$
& $\alpha$
\\
\bottomrule
\end{tabular}
}%
\end{center}
\caption{Some examples of $f$-divergences and the corresponding first- and second-order gradient values of functions $A(u)$ and $B(u)$ (defined in Equation~(\ref{app:AB-def})) at the equilibrium point.
}
\label{tab:f-divergences-AB}
\end{table}

Let us use $(\vtheta^*, \vomega^*)$ to denote the equilibrium point and we have the following realizability assumption:
\begin{assumption}[Realizability]\label{assumption:realizability}
$\exists \otheta,\oomega$, such that $\forall \vx \in \mathrm{supp}(p)$, $q_{\vtheta^*}(\vx) = p(\vx)$ and $H_{\vomega^*}(\vx) = \log(p(\vx) Z_{\vtheta^*})$.
\end{assumption}
That is we assume the energy function and the variational functions are powerful enough such that at the equilibrium point, the model distribution $q_{\vtheta^*}$ matches the true data distribution $p$ and the variational function achieves the optimal form in Theorem~\ref{theorem:new-f-bound}. Note that this also implies:
\begin{align*}
    \forall \vx \in \mathrm{supp}(p), \Hw + \Et  &= \log(p(\vx) Z_{\vtheta^*}) - \log(\exp(-E_{\vtheta^*}(\vx)))\\
    &= \log (p(\vx)/q_{\vtheta^*}(\vx)) = 0
\end{align*}

Furthermore, we have the following assumption on the energy function:
\begin{assumption}\label{assumption:energy-function} $(\bb{E}_{p(\vx)}[\nabla_\vtheta \Et \nabla_\vtheta^\top \Et] - 2\bb{E}_{p(\vx)}[\nabla_\vtheta \Et] \bb{E}_{p(\vx)}[\nabla_\vtheta^\top \Et])\at{\otheta}$ is positive semi-definite.
\end{assumption}
In one-dimension case, this assumption is saying: $\bb{E}_{p(x)}[(\nabla_\theta E_\theta(x))^2] - 2(\bb{E}_{p(x)}[\nabla_\theta E_\theta(x)])^2 \geq 0$. Intuitively, this implies the sum of energies over the entire sample space changes smoothly such that the average change of energies is small compared to the sum of the squared change of energies. For a better understanding of this assumption and its rationality, in the following, we will introduce two simple examples that satisfy this assumption.
\begin{example}
Consider the following quadratic energy function:
\begin{align*}
    E_\vtheta(x) = \frac{(x-\mu)^2}{2\sigma^2},~\vtheta = [\mu, \sigma]^\top,~p=q_\vtheta = \frac{\exp(-\frac{(x-\mu)^2}{2\sigma^2})}{\sqrt{2 \pi} \sigma}
\end{align*}
which corresponds to a Gaussian distribution.

The first-order gradient is:
\begin{align*}
    \frac{\partial E_\vtheta(x)}{\partial \mu} = \frac{x-\mu}{\sigma^2},~\frac{\partial E_\vtheta(x)}{\partial \sigma} = -\frac{(x-\mu)^2}{\sigma^3}
\end{align*}
Then, we have:
\begin{align*}
    \int_{- \infty}^{+ \infty} p(x) \frac{\partial E_\vtheta(x)}{\partial \mu} dx =& \int_{- \infty}^{+ \infty} \frac{\exp(-\frac{(x-\mu)^2}{2\sigma^2})}{\sqrt{2 \pi} \sigma} \frac{x-\mu}{\sigma^2} d x = 0\\
    \int_{- \infty}^{+ \infty} p(x) \frac{\partial E_\vtheta(x)}{\partial \sigma} dx =& -\int_{- \infty}^{+ \infty} \frac{\exp(-\frac{(x-\mu)^2}{2\sigma^2})}{\sqrt{2 \pi} \sigma} \frac{(x-\mu)^2}{\sigma^3} d x\\
    =& \frac{1}{\sqrt{2\pi} \sigma^4} \int_{- \infty}^{+ \infty} \exp(-\frac{x^2}{2\sigma^2}) x^2 dx = \frac{1}{\sigma}
\end{align*}
\begin{align*}
    \int_{- \infty}^{+ \infty} p(x) \nabla_\vtheta E_\vtheta(x) dx = 
    \int_{- \infty}^{+ \infty} p(x)
    \begin{bmatrix}
    \frac{\partial E_\vtheta(x)}{\partial \mu}\\
    \frac{\partial E_\vtheta(x)}{\partial \sigma}
    \end{bmatrix} dx =
    \begin{bmatrix}
    0\\
    \frac{1}{\sigma}
    \end{bmatrix}
\end{align*}
\begin{align*}
    \int_{- \infty}^{+ \infty} p(x) \left(\frac{\partial E_\vtheta(x)}{\partial \mu}\right)^2 d x =& \int_{- \infty}^{+ \infty} \frac{\exp(-\frac{(x-\mu)^2}{2\sigma^2})}{\sqrt{2 \pi} \sigma} \frac{(x-\mu)^2}{\sigma^4} d x\\
    =& \frac{1}{\sqrt{2\pi} \sigma^5} \int_{- \infty}^{+ \infty} \exp(-\frac{x^2}{2\sigma^2}) x^2 dx = \frac{1}{\sigma^2}\\
    \int_{- \infty}^{+ \infty} p(x) \frac{\partial E_\vtheta(x)}{\partial \mu}\frac{\partial E_\vtheta(x)}{\partial \sigma} d x =& -\int_{- \infty}^{+ \infty} \frac{\exp(-\frac{(x-\mu)^2}{2\sigma^2})}{\sqrt{2 \pi} \sigma} \frac{(x-\mu)^3}{\sigma^5} d x = 0\\
    \int_{- \infty}^{+ \infty} p(x) \left(\frac{\partial E_\vtheta(x)}{\partial \sigma}\right)^2 d x =& \int_{- \infty}^{+ \infty} \frac{\exp(-\frac{(x-\mu)^2}{2\sigma^2})}{\sqrt{2 \pi} \sigma} \frac{(x-\mu)^4}{\sigma^6} d x = 0\\
    =& \frac{1}{\sqrt{2\pi} \sigma^7} \int_{- \infty}^{+ \infty} \exp(-\frac{x^2}{2\sigma^2}) x^4 dx = \frac{3}{\sigma^2}
\end{align*}
\begin{align*}
    \int_{- \infty}^{+ \infty} p(x) \nabla_\vtheta E_\vtheta(x) \nabla_\vtheta^\top E_\vtheta(x) dx = \int_{- \infty}^{+ \infty} p(x)
    \begin{bmatrix}
    \left(\frac{\partial E_\vtheta(x)}{\partial \mu}\right)^2 & \frac{\partial E_\vtheta(x)}{\partial \mu}\frac{\partial E_\vtheta(x)}{\partial \sigma}\\
    \frac{\partial E_\vtheta(x)}{\partial \mu}\frac{\partial E_\vtheta(x)}{\partial \sigma} & \left(\frac{\partial E_\vtheta(x)}{\partial \sigma}\right)^2
    \end{bmatrix} dx = 
    \begin{bmatrix}
    \frac{1}{\sigma^2} & 0\\
    0 & \frac{3}{\sigma^2}
    \end{bmatrix}
\end{align*}
Therefore, Assumption~\ref{assumption:energy-function} holds:
\begin{align*}
    \bb{E}_{p(\vx)}[\nabla_\vtheta \Et \nabla_\vtheta^\top \Et] - 2\bb{E}_{p(\vx)}[\nabla_\vtheta \Et] \bb{E}_{p(\vx)}[\nabla_\vtheta^\top \Et] = 
    \begin{bmatrix}
    \frac{1}{\sigma^2} & 0\\
    0 & \frac{1}{\sigma^2}
    \end{bmatrix} \succ \bm{0}
\end{align*}
\end{example}

\begin{example}
We consider a more powerful energy function:
\begin{align*}
    E_\vtheta(\vx) = - \log\left(\sum_{k=1}^K \pi_k \frac{\exp(-\frac{1}{2}(\vx - \vmu)^\top \Sigma_k^{-1}(\vx - \vmu_k))}{\sqrt{(2\pi)^n |\Sigma_k|}}\right),
    ~\sum_{k=1}^K \pi_k = 1, ~\vx \in \bb{R}^n,~\vmu_k \in \bb{R}^n,~\Sigma_k \in S_{++}^n
\end{align*}
where $\pi_1, \ldots, \pi_K,\vmu_1, \ldots, \vmu_K, \Sigma_1, \ldots, \Sigma_K$ are the learnable parameters $\vtheta$; $S_{++}^n$ denotes the space of symmetric positive definite $n \times n$ matrices.

The partition function is:
\begin{align*}
    Z_\vtheta = \int \exp(-E_\vtheta(\vx)) \mathrm{d} \vx = \int \sum_{k=1}^K \pi_k \frac{\exp(-\frac{1}{2}(\vx - \vmu)^\top \Sigma_k^{-1}(\vx - \vmu_k))}{\sqrt{(2\pi)^n |\Sigma_k|}} \mathrm{d} \vx = \sum_{k=1}^K \pi_k = 1
\end{align*}
Therefore, Assumption~\ref{assumption:energy-function} holds:
\begin{align*}
    &\bb{E}_{p(\vx)}[\nabla_\vtheta \Et \nabla_\vtheta^\top \Et] - 2\bb{E}_{p(\vx)}[\nabla_\vtheta \Et] \bb{E}_{p(\vx)}[\nabla_\vtheta^\top \Et]\\
    =& \bb{E}_{p(\vx)}[\nabla_\vtheta \Et \nabla_\vtheta^\top \Et] - 2\nabla_\vtheta \log Z_\vtheta \nabla_\vtheta^\top \log Z_\vtheta\\
    =& \bb{E}_{p(\vx)}[\nabla_\vtheta \Et \nabla_\vtheta^\top \Et]
\end{align*}
which is a positive semi-definite moment matrix. 

Note that the model distribution induced by the energy function correspond to a Gaussian mixture model (GMM). With the universal approximation theorem of GMM 
(i.e., any smooth density can be approximated with any specific nonzero amount of error by a GMM with enough components
\cite{Goodfellow-et-al-2016}), the above EBM can be used to fit any probability distribution when $K$ is large enough.
\end{example}
Finally, the last assumption we need is:
\begin{assumption}\label{assumption:full-rank}
$(\bb{E}_{p(\vx)} [\nabla_\vomega \Hw \nabla_\vomega^\top \Hw])\at{\oomega}$ and $(\bb{E}_{p(\vx)}[\nabla_\vomega \Hw] \bb{E}_{p(\vx)} [\nabla^\top_\vtheta \Et])\at{(\otheta, \oomega)}$ are full column rank.
\end{assumption}
This is similar to the assumption used in the local convergence analysis of GANs \cite{nagarajan2017gradient,mescheder2018training}. Alternatively, we can replace it with another assumption that the rank deficiencies of the above matrices, if any, correspond to equivalent equilibria. In the following, we will use Assumption~\ref{assumption:full-rank}.

\subsection{Local Convergence of Single-Step $f$-EBM}
In the following, we will establish a theoretical proof of the local convergence property of Single-Step $f$-EBM. More specifically, instead of assuming the variational function $H_\vomega$ is optimal at every update of the energy function, we consider a more realistic setting, where both the variational function $H_\vomega$ and energy function $E_\vtheta$ take simultaneous gradient steps, with time derivatives defined as:
\begin{align}
    \begin{pmatrix}
    \dot \vtheta\\
    \dot \vomega
    \end{pmatrix} = 
    \begin{pmatrix}
    - \nabla_\vtheta V(\vtheta, \vomega) \\
    \nabla_\vomega V(\vtheta, \vomega)
    \end{pmatrix} \label{eq:theta-omega-updates}
\end{align}
That is we use gradient descent to update $\vtheta$ and gradient ascent to update $\vomega$ with respect to $V(\vtheta, \vomega)$ at the same frequency. 
In practice we can also update $\vtheta$ and $\vomega$ alternatively as presented in Algorithm~{\ref{alg:f-ebm}}. 
Note that our theoretical analysis also holds for the alternative gradient methods, as the ordinary differential equations of both simultaneous gradient methods and alternative gradient methods have the same Jacobians at the equilibrium point.

Throughout this paper we will use the notation $\nabla^\top(\cdot)$ to denote the row vector corresponding to the gradient that is being compute. First, we derive the Jacobian at equilibrium.
\begin{theorem}\label{theorem:jacobian}
For the dynamical system defined in Equation~(\ref{eq:new-minimax-game}) and the updates defined in Equation~(\ref{eq:theta-omega-updates}), under Assumption~\ref{assumption:realizability}, the Jacobian at an equilibrium point ($\otheta,\oomega$) is:
\begin{align*}
    \vJ =  
    \begin{bmatrix}
    -\nabla_\vtheta^2 V & -\nabla_{\vomega} \nabla_{\vtheta} V\\
    \nabla_{\vtheta} \nabla_{\vomega} V & \nabla_{\vomega}^2 V
    \end{bmatrix} = 
    \begin{bmatrix}
    - f''(1) \vK_{EE} & f''(1) \vK_{EH}\\
    - f''(1) \vK_{EH}^\top & -f''(1) \vK_{HH}
    \end{bmatrix}
\end{align*}
where
\begin{align*}
    \vK_{EE} \defeq & (\bb{E}_{p(\vx)}[\nabla_\vtheta \Et \nabla_\vtheta^\top \Et] - 2\bb{E}_{p(\vx)}[\nabla_\vtheta \Et] \bb{E}_{p(\vx)}[\nabla_\vtheta^\top \Et])\at{\otheta}\\
    \vK_{EH} \defeq & (\bb{E}_{p(\vx)} [\nabla_\vtheta \Et]\bb{E}_{p(\vx)}[\nabla_\vomega^\top \Hw])\at{(\otheta, \oomega)}\\
    \vK_{HH} \defeq & (\bb{E}_{p(\vx)}[\nabla_\vomega \Hw \nabla_\vomega^\top \Hw])\at{\oomega}
\end{align*}

\end{theorem}
\begin{proof}
First, let us derive the first- and second-order derivatives of $V(\vtheta, \vomega)$ with respect to $\vtheta$:
\begin{align*}
    \nabla_\vtheta V(\vtheta, \vomega) = & \int p(\vx) A'(\Hw + \Et) \nabla_\vtheta \Et \mathrm{d} \vx - \\
    &\int B(\Hw + \Et) \nabla_\vtheta q_\vtheta(\vx) \mathrm{d} \vx - \\ 
    &\int q_\vtheta(\vx) B'(\Hw + \Et) \nabla_\vtheta \Et \mathrm{d} \vx
\end{align*}
\begin{align*}
    \nabla^2_\vtheta V(\vtheta, \vomega) = & \int p(\vx) A''(\Hw + \Et) \nabla_\vtheta \Et \nabla_\vtheta^\top \Et \mathrm{d} \vx + \\
    & \int p(\vx) A'(\Hw + \Et) \nabla_\vtheta^2 \Et \mathrm{d} \vx - \\
    & \int B'(\Hw + \Et) \nabla_\vtheta q_\vtheta(\vx) \nabla_\vtheta^\top \Et \mathrm{d} \vx - \\
    & \int B(\Hw + \Et) \nabla_\vtheta^2 q_\vtheta(\vx) \mathrm{d} \vx - \\
    & \int B'(\HE) \nabla_\vtheta \Et \nabla_\vtheta^\top q_\vtheta(\vx) \mathrm{d} \vx - \\
    & \int q_\vtheta(\vx) B''(\HE) \nabla_\vtheta \Et \nabla_\vtheta^\top \Et \mathrm{d} \vx - \\
    & \int q_\vtheta(\vx) B'(\HE) \nabla_\vtheta^2 \Et \mathrm{d} \vx
\end{align*}
At the equilibrium point $(\vtheta^*, \vomega^*)$, we have:
\begin{align*}
    \nabla_\vtheta^2 V(\vtheta, \vomega)\at{(\vtheta^*, \vomega^*)} = & \left(\int p(\vx) A''(0) \nabla_\vtheta \Et \nabla_\vtheta^\top \Et \mathrm{d} \vx + \right.\\
    & ~~\int p(\vx) A'(0) \nabla_\vtheta^2 \Et \mathrm{d} \vx - \\
    & ~~\int B'(0) \nabla_\vtheta q_\vtheta(\vx) \nabla_\vtheta^\top \Et \mathrm{d} \vx - \\
    & ~~\int B(0) \nabla_\vtheta^2 q_\vtheta(\vx) \mathrm{d} \vx - \\
    & ~~\int B'(0) \nabla_\vtheta \Et \nabla_\vtheta^\top q_\vtheta(\vx) \mathrm{d} \vx - \\
    & ~~\int q_\vtheta(\vx) B''(0) \nabla_\vtheta \Et \nabla_\vtheta^\top \Et \mathrm{d} \vx - \\
    &~~\left. \int q_\vtheta(\vx) B'(0) \nabla_\vtheta^2 \Et \mathrm{d} \vx\right)\at[\Bigg]{\otheta}
\end{align*}
With Theorem~\ref{theorem:AB} and Assumption~\ref{assumption:realizability}, as well as by definition of function $B(u)$ ($B(0) = 0$) we have:
\begin{align}
    \nabla_\vtheta^2 V(\vtheta, \vomega)\at{(\vtheta^*, \vomega^*)} = & \left(\int p(\vx) A''(0) \nabla_\vtheta \Et \nabla_\vtheta^\top \Et \mathrm{d} \vx - \right.\nonumber\\
    & ~~\int B'(0) \nabla_\vtheta q_\vtheta(\vx) \nabla_\vtheta^\top \Et \mathrm{d} \vx - \nonumber\\
    & ~~\int B'(0) \nabla_\vtheta \Et \nabla_\vtheta^\top q_\vtheta(\vx) \mathrm{d} \vx - \nonumber\\
    &~~\left. \int q_\vtheta(\vx) B''(0) \nabla_\vtheta \Et \nabla_\vtheta^\top \Et \mathrm{d} \vx \right)\at[\Bigg]{\otheta}\label{eq:theta-hessian}
\end{align}

Next,  for a $\vtheta$-parametrized energy based model $q_\vtheta(\vx) = \frac{\exp(-E_\vtheta(\vx))}{Z_\vtheta}$, we observe that
\begin{align}
    \nabla_\vtheta q_\vtheta(\vx) 
    &= q_\vtheta(\vx) \nabla_\vtheta \log q_\vtheta(\vx) 
    = q_\vtheta(\vx) (-\nabla_\vtheta E_\vtheta(\vx) - \nabla_\vtheta \log Z_\vtheta)\\
    \nabla_\vtheta \log Z_\vtheta &= \frac{\int \exp(-\Et) (- \nabla_\vtheta \Et) \mathrm{d} \vx}{Z_\vtheta} = - \int q_\vtheta(\vx) \nabla_\vtheta \Et \mathrm{d} \vx
\end{align}

With Assumption~\ref{assumption:realizability} ($q_{\vtheta^*} = p$) and above observation, Equation~(\ref{eq:theta-hessian}) can be written as:
\begin{align*}
    \nabla_\vtheta^2 V(\vtheta, \vomega)\at{(\vtheta^*, \vomega^*)} = & \left(\int p(\vx) (A''(0) - B''(0)) \nabla_\vtheta \Et \nabla_\vtheta^\top \Et \mathrm{d} \vx - \right.\nonumber\\
    & ~~\int B'(0) \nabla_\vtheta q_\vtheta(\vx) \nabla_\vtheta^\top \Et \mathrm{d} \vx - \nonumber\\
    &~~\left. \int B'(0) \nabla_\vtheta \Et \nabla_\vtheta^\top q_\vtheta(\vx) \mathrm{d} \vx \right)\at[\Bigg]{\otheta}\\
    = & \left((A''(0) - B''(0)) \int p(\vx) \nabla_\vtheta \Et \nabla_\vtheta^\top \Et \mathrm{d} \vx - \right.\nonumber\\
    & ~~\int B'(0) q_\vtheta(\vx) (-\nabla_\vtheta \Et - \nabla_\vtheta \log Z_\vtheta) \nabla_\vtheta^\top \Et \mathrm{d} \vx - \nonumber\\
    &~~\left. \int B'(0) q_\vtheta(\vx) \nabla_\vtheta \Et (-\nabla^\top_\vtheta \Et - \nabla^\top_\vtheta \log Z_\vtheta) \mathrm{d} \vx \right)\at[\Bigg]{\otheta}\\
    = & \left((A''(0) - B''(0) + 2B'(0)) \int p(\vx) \nabla_\vtheta \Et \nabla_\vtheta^\top \Et \mathrm{d} \vx + \right.\nonumber\\
    &~~~ B'(0) \nabla_\vtheta \log Z_\vtheta \cdot \int q_\vtheta(\vx) \nabla_\vtheta^\top \Et \mathrm{d} \vx + \\
    &~~\left. B'(0) \int q_\vtheta(\vx) \nabla_\vtheta \Et \mathrm{d} \vx \cdot \nabla_\vtheta^\top \log Z_\vtheta \right)\at[\Bigg]{\otheta}\\
    = & \left((A''(0) - B''(0) + 2B'(0)) \int p(\vx) \nabla_\vtheta \Et \nabla_\vtheta^\top \Et \mathrm{d} \vx - \right.\nonumber\\
    &~~\left. 2B'(0) \nabla_\vtheta \log Z_\vtheta \nabla_\vtheta^\top \log Z_\vtheta \right)\at[\Bigg]{\otheta}
\end{align*}

With Theorem~\ref{theorem:AB}, we have:
\begin{align*}
    \nabla_\vtheta^2 V(\vtheta, \vomega)\at{(\vtheta^*, \vomega^*)} 
    =& f''(1) \left(\int p(\vx) \nabla_\vtheta \Et \nabla_\vtheta^\top \Et \mathrm{d} \vx - 2\nabla_\vtheta \log Z_\vtheta \nabla_\vtheta^\top \log Z_\vtheta\right)\at[\Bigg]{\otheta}\\
    =& f''(1) \left(\int p(\vx) \nabla_\vtheta \Et \nabla_\vtheta^\top \Et \mathrm{d} \vx - 2\int p(\vx) \nabla_\vtheta \Et \mathrm{d} \vx  \int p(\vx) \nabla_\vtheta^\top \Et \mathrm{d} \vx \right)\at[\Bigg]{\otheta}\\
    =& f''(1)(\bb{E}_{p(\vx)}[\nabla_\vtheta \Et \nabla_\vtheta^\top \Et] - 2\bb{E}_{p(\vx)}[\nabla_\vtheta \Et] \bb{E}_{p(\vx)}[\nabla_\vtheta^\top \Et])\at{\otheta}
\end{align*}

Now let us derive the first- and second-order derivatives of $V(\vtheta, \vomega)$ with respect to $\vomega$:
\begin{align*}
    \nabla_\vomega V(\vtheta, \vomega) = & \int p(\vx) A'(\Hw + \Et) \nabla_\vomega \Hw \mathrm{d} \vx - \int q_\vtheta(\vx) B'(\Hw + \Et) \nabla_\vomega \Hw \mathrm{d} \vx
\end{align*}
\begin{align*}
    \nabla_\vomega^2 V(\vtheta, \vomega) = & \int p(\vx) A''(\HE) \nabla_\vomega \Hw \nabla_\vomega^\top \Hw \mathrm{d} \vx + \\
    & \int p(\vx) A'(\HE) \nabla_\vomega^2 \Hw \mathrm{d} \vx - \\
    & \int q_\vtheta(\vx) B''(\HE) \nabla_\vomega \Hw \nabla_\vomega^\top \Hw \mathrm{d} \vx - \\
    & \int q_\vtheta(\vx) B'(\HE) \nabla_\vomega^2 \Hw \mathrm{d} \vx
\end{align*}

At the equilibrium point $(\otheta, \oomega)$, under Assumption~\ref{assumption:realizability}, we have:
\begin{align*}
    \nabla_\vomega^2 V(\vtheta, \vomega)\at{(\vtheta^*, \vomega^*)}  = \left((A''(0) - B''(0)) \int p(\vx)  \nabla_\vomega \Hw \nabla_\vomega^\top \Hw \mathrm{d} \vx + 
    (A'(0) - B'(0)) \int p(\vx)  \nabla_\vomega^2 \Hw \mathrm{d} \vx\right)\at[\Bigg]{\oomega}
\end{align*}
With Theorem~\ref{theorem:AB}, we have:
\begin{align*}
    \nabla_\vomega^2 V(\vtheta, \vomega)\at{(\vtheta^*, \vomega^*)}  =& - f''(1) \left(\int p(\vx)  \nabla_\vomega \Hw \nabla_\vomega^\top \Hw \mathrm{d} \vx\right)\at[\Bigg]{\oomega}\\
    =& -f''(1) (\bb{E}_{p(\vx)}[\nabla_\vomega \Hw \nabla_\vomega^\top \Hw])\at{\oomega}
\end{align*}

Finally, let us derive $\nabla_\vomega \nabla_\vtheta V(\vtheta, \vomega)$:
\begin{align*}
    \nabla_\vomega \nabla_\vtheta V(\vtheta, \vomega) =& \int p(\vx) A''(\HE) \nabla_\vtheta \Et \nabla_\vomega^\top \Hw \mathrm{d} \vx - \\
    & \int B'(\HE) \nabla_\vtheta q_\vtheta(\vx) \nabla_\vomega^\top \Hw \mathrm{d} \vx - \\
    & \int q_\vtheta(\vx) B''(\HE) \nabla_\vtheta \Et \nabla_\vomega^\top \Hw \mathrm{d} \vx
\end{align*}

At the equilibrium point $(\otheta, \oomega)$, under Assumption~\ref{assumption:realizability}, we have:
\begin{align*}
    \nabla_\vomega \nabla_\vtheta V(\vtheta, \vomega)\at{(\vtheta^*, \vomega^*)} =& \left((A''(0) - B''(0))\int p(\vx) \nabla_\vtheta \Et \nabla_\vomega^\top \Hw \mathrm{d} \vx - \right.\\
    &~~ \left. B'(0) \int  q_\vtheta(\vx)(- \nabla_\vtheta \Et - \nabla_\vtheta \log Z_\vtheta) \nabla_\vomega^\top \Hw \mathrm{d} \vx\right)\at[\Bigg]{(\otheta, \oomega)} \\
    =& \left((A''(0) - B''(0) + B'(0))\int p(\vx) \nabla_\vtheta \Et \nabla_\vomega^\top \Hw \mathrm{d} \vx + \right.\\
    &~~ \left. B'(0) \nabla_\vtheta \log Z_\vtheta \int  p(\vx) \nabla_\vomega^\top \Hw \mathrm{d} \vx\right)\at[\Bigg]{(\otheta, \oomega)} \\
    =& \left((A''(0) - B''(0) + B'(0))\int p(\vx) \nabla_\vtheta \Et \nabla_\vomega^\top \Hw \mathrm{d} \vx - \right.\\
    &~~ \left. B'(0) \int p(\vx) \nabla_\vtheta \Et \mathrm{d} \vx \int  p(\vx) \nabla_\vomega^\top \Hw \mathrm{d} \vx\right)\at[\Bigg]{(\otheta, \oomega)}
\end{align*}
With Theorem~\ref{theorem:AB}, we have:
\begin{align*}
    \nabla_\vomega \nabla_\vtheta V(\vtheta, \vomega)\at{(\vtheta^*, \vomega^*)} = - f''(1) (\bb{E}_{p(\vx)} [\nabla_\vtheta \Et]\bb{E}_{p(\vx)}[\nabla_\vomega^\top \Hw])\at{(\otheta, \oomega)}
\end{align*}
Similarly, for $\nabla_\vtheta \nabla_\vomega V(\vtheta, \vomega)$, we have:
\begin{align*}
    \nabla_\vtheta \nabla_\vomega V(\vtheta, \vomega)\at{(\vtheta^*, \vomega^*)} = - f''(1) (\bb{E}_{p(\vx)} [\nabla_\vomega \Hw]\bb{E}_{p(\vx)}[\nabla_\vtheta^\top \Et])\at{(\otheta, \oomega)}
\end{align*}
\end{proof}

Now we are ready to present the main theorem:
\begin{theorem}
The dynamical system defined in Equation~(\ref{eq:new-minimax-game}) and the updates defined in Equation~(\ref{eq:theta-omega-updates}) is locally exponentially stable with respect to an equilibrium point ($\otheta, \oomega$) when the Assumptions~\ref{assumption:realizability},~\ref{assumption:energy-function},~\ref{assumption:full-rank} hold for ($\otheta, \oomega$). Let $\lambda_\mathrm{max}(\cdot)$ and $\lambda_\mathrm{min}(\cdot)$ denote the largest and smallest eigenvalues of a non-zero positive semi-definite matrix. The rate of convergence is governed only by the eigenvalues $\lambda$ of the Jacobian $\vJ$ of the system at the equilibrium point, with a strictly negative real part upper bounded as:
\begin{itemize}
    \item When $\mathrm{Im}(\lambda) = 0$, 
    \begin{align*}
    \mathrm{Re}(\lambda) < -f''(1) \frac{\lambda_\mathrm{min}(\vK_{HH}) \lambda_\mathrm{min}(\vK_{EH} \vK_{EH}^\top)}{\lambda_\mathrm{min}(\vK_{HH}) \lambda_\mathrm{max}(\vK_{EE}, \vK_{HH}) + \lambda_\mathrm{min}(\vK_{EH} \vK_{EH}^\top)} < 0
    \end{align*}
    where $\lambda_\mathrm{max}(\vK_{EE}, \vK_{HH}) = \max(\lambda_\mathrm{max}(\vK_{EE}), \lambda_\mathrm{max}(\vK_{HH}))$.
    \item When $\mathrm{Im}(\lambda) \neq 0$,
    \begin{align*}
    \mathrm{Re}(\lambda) \leq -\frac{f''(1)}{2} (\lambda_\mathrm{min}(\vK_{EE}) + \lambda_\mathrm{min}(\vK_{HH})) < 0
    \end{align*}
\end{itemize}
\end{theorem}
\begin{proof}
In Theorem~\ref{theorem:jacobian}, for any $f$-divergences with strictly convex and differentiable generator function $f$, we derived the Jacobian of the system $\vJ$ under Assumption~\ref{assumption:realizability}. 
With Assumptions~\ref{assumption:energy-function} and \ref{assumption:full-rank}, and the strict convexity of the function $f$, we know that $f''(1) \vK_{EE}$ is positive semi-definite, $f''(1) \vK_{HH}$ is positive definite (a full rank moment matrix with a positive multiplicative factor), and $f''(1)\vK_{EH}^\top$ is full column rank.
Therefore, with Theorem~\ref{theorem:new-eigenvalue-bound}, we know that the Jacobian $\vJ$ is a Hurwitz matrix, (\emph{i.e.}, all the eigenvalues of $\vJ$ have strictly negative real parts). Furthermore, with Theorem~\ref{theorem:new-eigenvalue-bound}, we can obtain an upper bound of the real parts of the eigenvalues. Finally, with Theorem~\ref{theorem:equilibrium-hurwitz}, we can conclude that the system is locally exponentially stable.
\end{proof}

\section{Implementation Details}
\subsection{Implementation of $f$-EBM Algorithm in PyTorch}\label{app:implementation}
\begin{Verbatim}[numbers=left, xleftmargin=5mm]
def update_H(real_x, fake_x, model_h, model_e, optim_h, grad_exp, conjugate_grad_exp):
    // Step 6-7 in Algorithm 1. Update the parameter of the variational function.
    // - real_x and fake_x are samples from the data distribution and EBM respectively.
    // - model_h is the neural network for the variational function $H_\vomega$.
    // - model_e is the neural network for the energy function $E_\vtheta$.
    // - optim_h is the optimizer for model_h, e.g. torch.optim.SGD(model_h.parameters())
    // - grad_exp and conjugate_grad_exp are functions defined by the used f-divergence.
    real_e, fake_e = model_e(real_x), model_e(fake_x)
    real_h, fake_h = model_h(real_x), model_h(fake_x)
    loss_h = -(grad_exp(real_h+real_e) - conjugate_grad_exp(fake_h+fake_e)).mean()
    optim_h.zero_grad()
    loss_h.backward()
    optim_h.step()

def update_E(real_x, fake_x, model_h, model_e, optim_e, grad_exp, conjugate_grad_exp):
    // Step 8-9 in Algorithm 1. Update the parameter of the energy function.
    // - optim_e is the optimizer for model_e, e.g. torch.optim.SGD(model_e.parameters())
    real_e, fake_e = model_e(real_x), model_e(fake_x)
    real_h, fake_h = model_h(real_x), model_h(fake_x)
    loss_e = torch.mean(grad_exp(real_h+real_e)) + \
             torch.mean(conjugate_grad_exp(fake_h+fake_e).detach() * fake_e) - \
             torch.mean(conjugate_grad_exp(fake_h+fake_e)) - \
             torch.mean(fake_e) * torch.mean(conjugate_grad_exp(fake_h+fake_e)).detach()
    optim_e.zero_grad()
    loss_e.backward()
    optim_e.step()
\end{Verbatim}
Note that according to Lemma~\ref{lemma:product-of-expectations}, technically we should use two independent batches of data to estimate the two expectations in the product $\bb{E}_{q_\vtheta(\vx)}[\nabla_\vtheta E_\vtheta(\vx)] \cdot \bb{E}_{q_\vtheta(\vx)}[f^*(f'(\exp(E_\vtheta(\vx) + H_\vomega(\vx))))]$. Empirically we found that using the same set of samples also works well, since it is an asymptotically consistent estimator.

To implement $f$-EBM in Tensorflow \cite{abadi2016tensorflow}, the main change is to use $\texttt{tf.stop\_gradient(X)}$ to replace \texttt{X.detach()}. Functions \texttt{grad\_exp} and \texttt{conjugate\_grad\_exp} correspond to $f'(\exp(u))$ and $f^*(f'(\exp(u)))$, which can be derived based on the definitions of $f$-divergences (see Table 5 and Table 6 in \cite{nowozin2016f} for reference). Some examples of $f'(\exp(u))$ and $f^*(f'(\exp(u)))$ can be found in Table~\ref{tab:f-divergences}.

\begin{table}[h!bt]
\begin{center}
\scalebox{0.79}{%
\begin{tabular}{llll}
\toprule
Name & $D_f(P\|Q)$ & $f'(\exp(u))$ & $f^*(f'(\exp(u)))$\\ \midrule
Kullback-Leibler
& $\int p(\vx) \log \frac{p(\vx)}{q(\vx)} \, \mathrm{d} \vx$
& $1 + u$
& $\exp(u)$\\
Reverse Kullback-Leibler
& $\int q(\vx) \log \frac{q(\vx)}{p(\vx)} \, \mathrm{d} \vx$
& $-\exp(- u)$
& $-1 + u$\\
Pearson $\chi^2$
& $\int \frac{(q(\vx)-p(\vx))^2}{p(\vx)}\, \mathrm{d} \vx$
& $2\exp(u) - 2$
& $\exp(2u) - 1$\\
Neyman $\chi^2$
& $\int \frac{(p(\vx) - q(\vx))^2}{q(\vx)} \, \mathrm{d} \vx$
& $1 - \exp(-2u)$
& $2 - 2 \exp(-u)$\\
Squared Hellinger
& $\int\left(\sqrt{p(\vx)} - \sqrt{q(\vx)}\right)^2 \, \mathrm{d} \vx$
& $1 - \exp(- \frac{u}{2})$
& $\exp(\frac{u}{2}) - 1$\\
Jensen-Shannon
& $\frac{1}{2} \int p(\vx) \log \frac{2 p(\vx)}{p(\vx)+q(\vx)}
  + q(\vx) \log \frac{2 q(\vx)}{p(\vx) + q(\vx)}\, \mathrm{d} \vx$
& $\log(2) + u - \log(1 + \exp(u))$
& $- \log(2) + \log(1 + \exp(u))$\\
$\alpha$-divergence ($\alpha \notin \{0,1\}$)
& $\frac{1}{\alpha (\alpha-1)} \int
  \left(p(\vx) \left[\left(\frac{q(\vx)}{p(\vx)}\right)^{\alpha}-1\right] - \alpha(q(\vx)-p(\vx))\right) \, \mathrm{d} \vx$
& $\frac{1}{\alpha - 1}(\exp((\alpha - 1) u) - 1)$
& $\frac{1}{\alpha} (\exp(\alpha u) - 1)$
\\
\bottomrule
\end{tabular}
}%
\end{center}
\vspace{-5pt}
\caption{Some examples of $f$-divergences and corresponding $f'(\exp(u))$ and $f^*(f'(\exp(u)))$ functions.
}
\vspace{-5pt}
\label{tab:f-divergences}
\end{table}

\subsection{Discussion on Differentiating Through Langevin Dynamics}\label{app:differetiate-langevin}
In this section, we provide a more detailed discussion on gradient reparametrization that was initially introduced in Section~\ref{sec:challenge-f-gan}. Specifically, we will discuss the possibility of directly extending the $f$-GANs framework by differentiating through the Langevin dynamics. As discussed before, we typically need hundreds of Langevin steps to produce a single sample while we cannot use a sample replay buffer to reduce the number of transition steps (because the initial distribution of the Markov chain cannot depend on the model parameters). During gradient backpropagation, we need the same number of backward steps, where each backward step further involves computing Hessian matrices that are proportional to the parameter and data dimension. This will lead to hundreds of times more memory consumption compared to only using Langevin dynamics for producing samples. More specifically, we provide the following implementation to differentiate through Langevin dynamics in PyTorch:
\vspace{-5pt}
\begin{Verbatim}[numbers=left, xleftmargin=5mm]
def train_energy(model_e, model_dis, optim_e, conjugate_fn, device='cuda'):
    // Initialize Langevin dynamics with random uniform distribution.
    fake_x = torch.rand(batch_size, channel_num, img_size, img_size, device=device)
    fake_x.requires_grad = True
    gaussian_noise = torch.randn(batch_size, 3, 32, 32, device=device)

    for k in range(num_langevin_steps):
        energy_value = model_e(fake_x)
        fake_x_grad = torch.autograd.grad(energy_value.sum(), fake_x, 
                                          create_graph=True)[0]
        gaussian_noise.normal_(0, gaussian_noise_std)
        fake_x = fake_x - step_size * fake_x_grad + gaussian_noise
        fake_x.data.clamp_(0, 1)
        
    fake_dis = model_dis(fake_x)
    energy_loss = - torch.mean(conjugate_fn(fake_dis))
    optim_e.zero_grad()
    energy_loss.backward()
    optim_e.step()
\end{Verbatim}
\vspace{-5pt}
Note that in Line~10, we need to set $\texttt{create\_graph=True}$ in order to compute the second-order derivatives when backpropagating through Langevin dynamics later, which will store all the computation graphs along the Langevin dynamics. By contrast, in $f$-EBMs we only use Langevin dynamics to produce samples and once we get the value of the gradient $\nabla_\vx E_\vtheta(\vx)$, we can discard the computational graph immediately after each transition step ($\texttt{create\_graph=False}$). As a result, with gradient reparametrization, the memory consumption grows linearly as the number of transitions. With modern GPU such as NVIDIA GeForce RTX 2080 Ti and the model architecture in Figure~\ref{fig:architecture}, we will run out of memory when the number of Langevin steps is larger than $5$. Since gradient backpropagation through Langevin dynamics involves computing many Hessian matrices, this approach is also computationally less efficient compared to $f$-EBMs. In our experiments, we set $\texttt{num\_langevin\_steps}=5$ and we observed that this approach cannot produce reasonable images.

\section{Additional Experimental Results for Fitting Univariate Mixture of Gaussians}
\subsection{Parameter Learning Results}\label{app:gaussian-parameter-learning}
\begin{table}[thb]
\begin{center}
\begin{tabular}{lcccc}
\toprule
Objective & $\mu^*$ & $\hat{\mu}$ & $\sigma^*$ & $\hat{\sigma}$\\ \midrule
Contrastive Divergence
& $1.01065$
& $1.01204$
& $1.82895$
& $1.82907$\\ \midrule
Kullback-Leibler
& $1.01065$
& $1.01536$
& $1.82895$
& $1.83024$\\
Reverse KL
& $1.58454$
& $1.58523$
& $1.63106$
& $1.63453$\\
Squared Hellinger
& $1.32024$
& $1.32274$
& $1.73089$
& $1.74710$\\
Jensen Shannon
& $1.30322$
& $1.31669$
& $1.76716$
& $1.76041$\\
Pearson $\chi^2$
& $0.57581$
& $0.56563$
& $1.92172$
& $1.93461$\\
Neyman $\chi^2$
& $1.83037$
& $1.82676$
& $1.51508$
& $1.51598$\\
$\alpha$-divergence ($\alpha=-0.5$)
& $1.74642$
& $1.74332$
& $1.55569$
& $1.55209$
\\
$\alpha$-divergence ($\alpha=-1$)
& $1.82923$
& $1.81979$
& $1.51844$
& $1.52513$\\
$\alpha$-divergence ($\alpha=0.9$)
& $1.07056$
& $1.07237$
& $1.81091$
& $1.81852$\\
\bottomrule
\end{tabular}
\end{center}
\caption{Fitting a quadratic EBM to mixtures of Gaussians. $\mu^*, \sigma^*$ represent the desired optimal solution under a certain discrepancy measure, and $\hat{\mu}, \hat{\sigma}$ represent the learned parameters. The first row is for contrastive divergence (with KL divergence being the underlying objective). The other rows are for $f$-EBMs with various discrepancy measures as the training objectives. 
}
\label{tab:gaussian-result}
\end{table}

\subsection{Density Ratio Estimation Results}\label{app:density-ratio}
\begin{figure*}[h!tb]
\centering
\subfigure[KL]{
\includegraphics[height=1.8in]{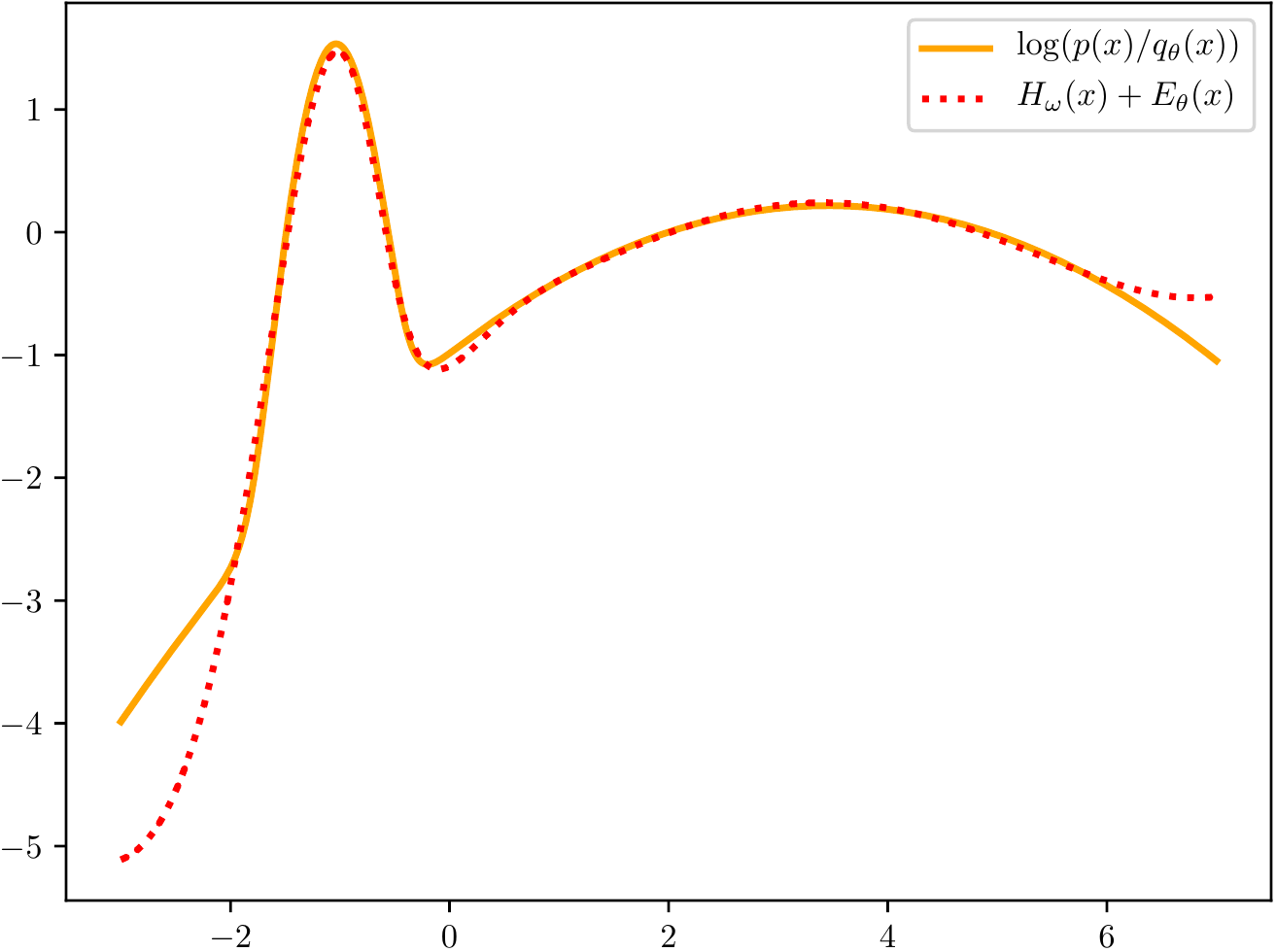}\label{fig:kl_dr}}
\subfigure[Reverse KL]{
\includegraphics[height=1.8in]{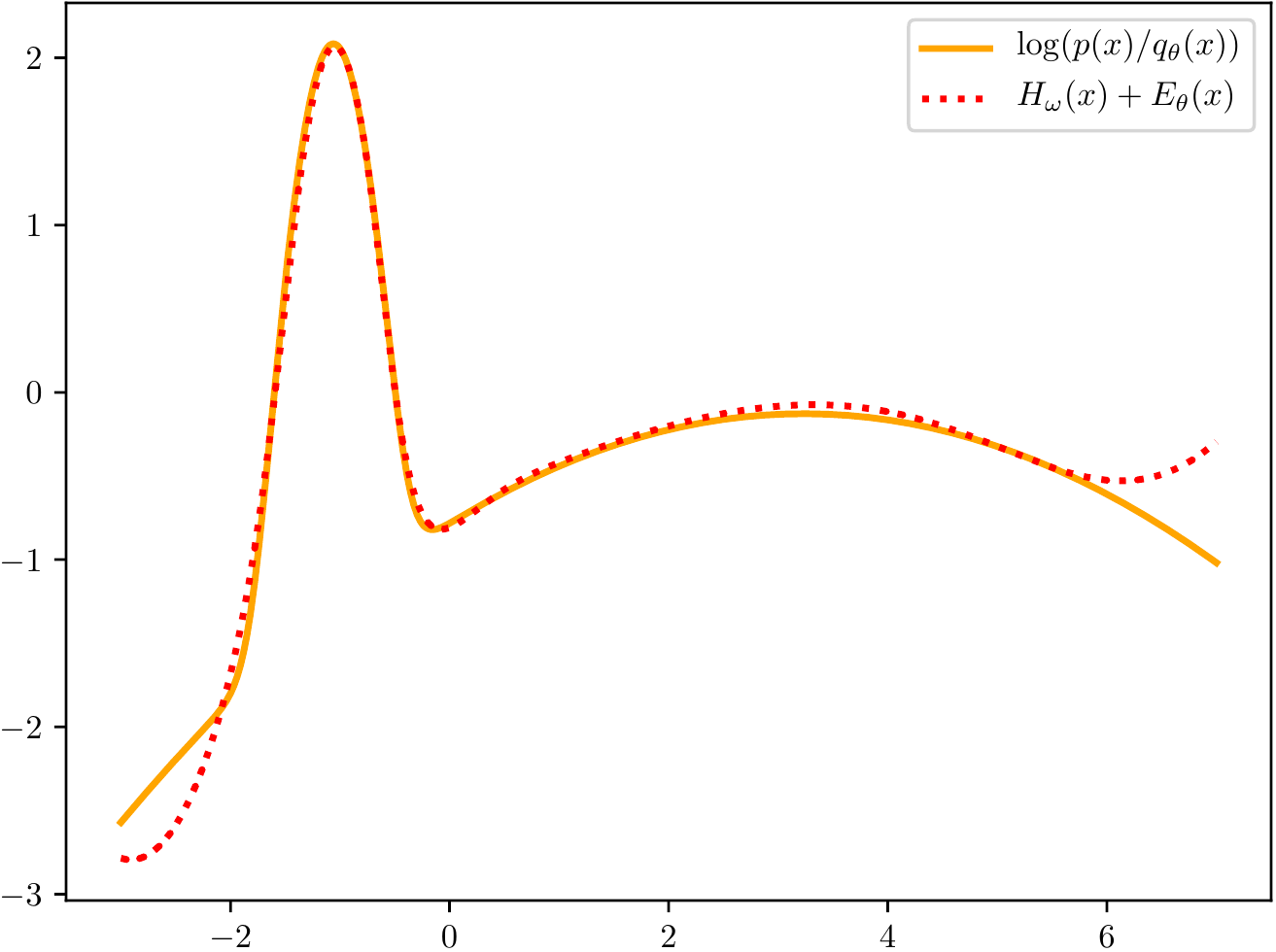}\label{fig:rev-kl_dr}}
\subfigure[Jensen-Shannon]{
\includegraphics[height=1.8in]{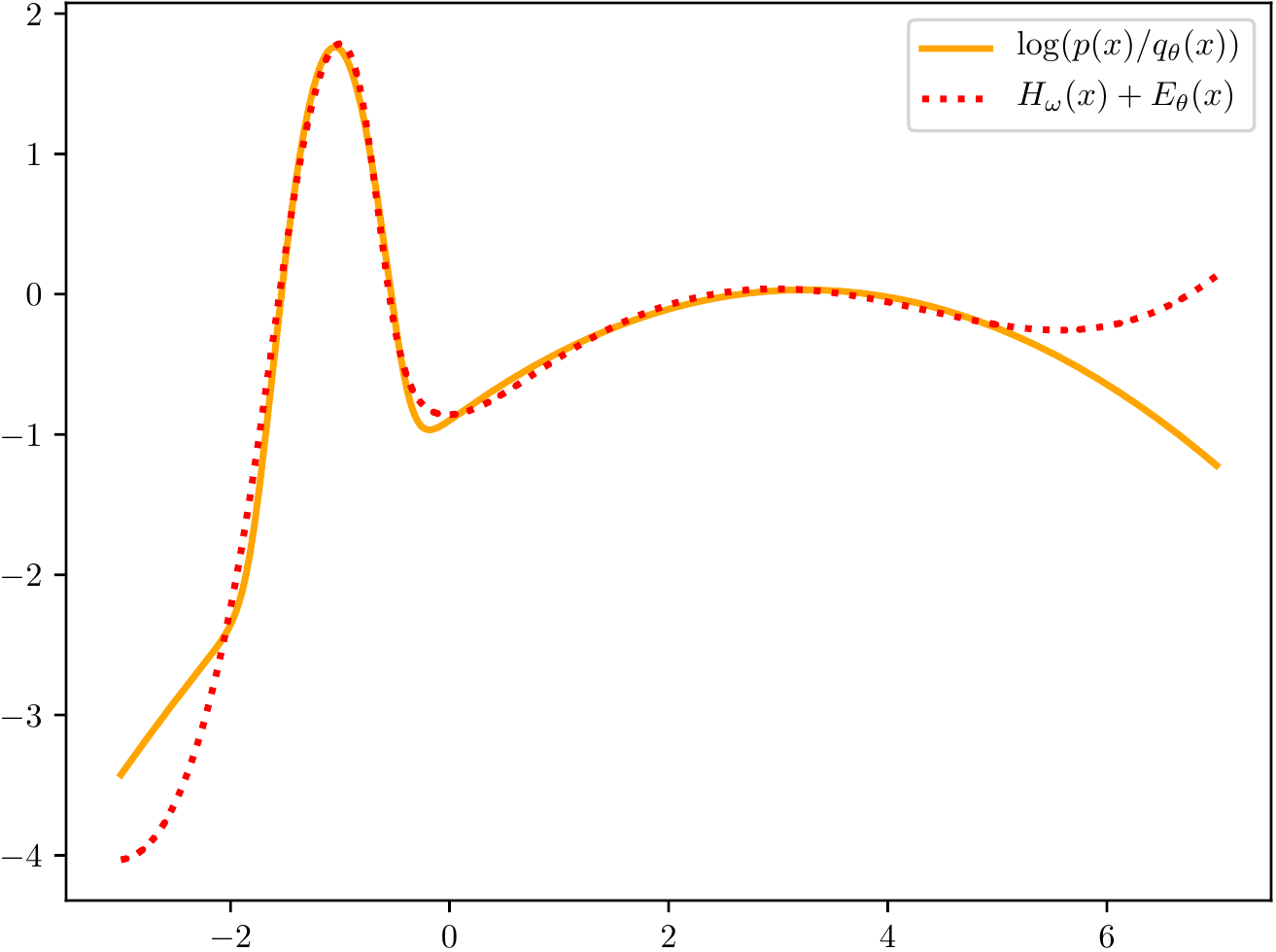}\label{fig:jensen-shannon_dr}}
\subfigure[$\alpha$-Divergence ($\alpha=-1$)]{
\includegraphics[height=1.8in]{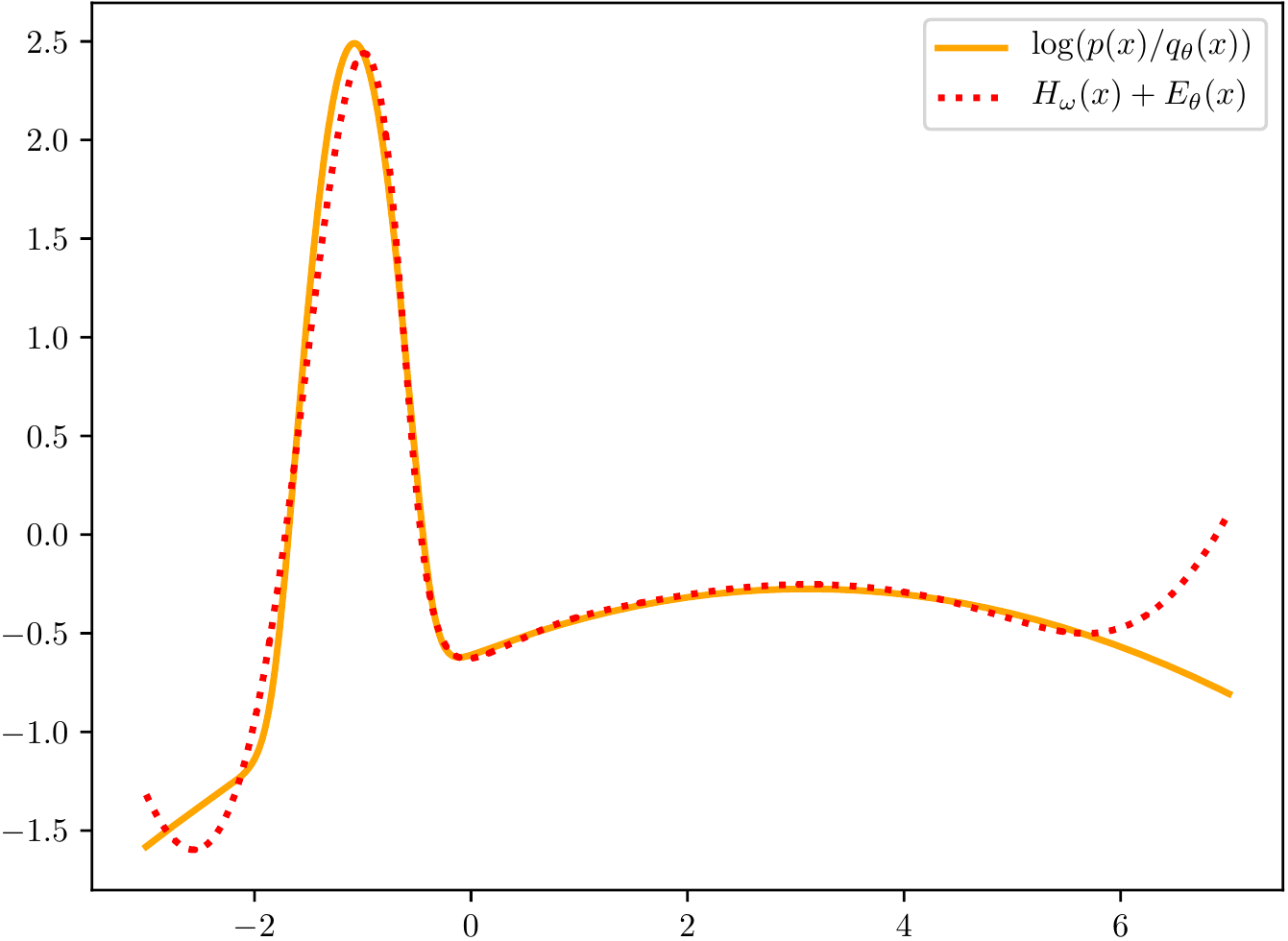}\label{fig:alpha-div_dr}}
\caption{Density ratio estimation results for different $f$-divergences. The orange solid line represents the ground truth log-scale density ratio ($\log(p(x)/q_\theta(x))$) under a certain divergence; The red dashed line represents the estimated log-scale density ratio ($H_\omega(x) + E_\theta(x)$) learned by $f$-EBM. Note that the estimated density ratio is accurate in most areas except in the low density regimes (\emph{e.g.} $(-\infty, -2]$ and $[6, +\infty)$) where very few training data comes from this region.}
\label{fig:divergences_dr}
\end{figure*}

\newpage
\subsection{Optimization Trajectories of Single-Step $f$-EBM}\label{app:optimization-trajectory}
\FloatBarrier
\vfill
\begin{figure*}[h!tb]
\centering
\subfigure[KL]{
\includegraphics[height=2.5in]{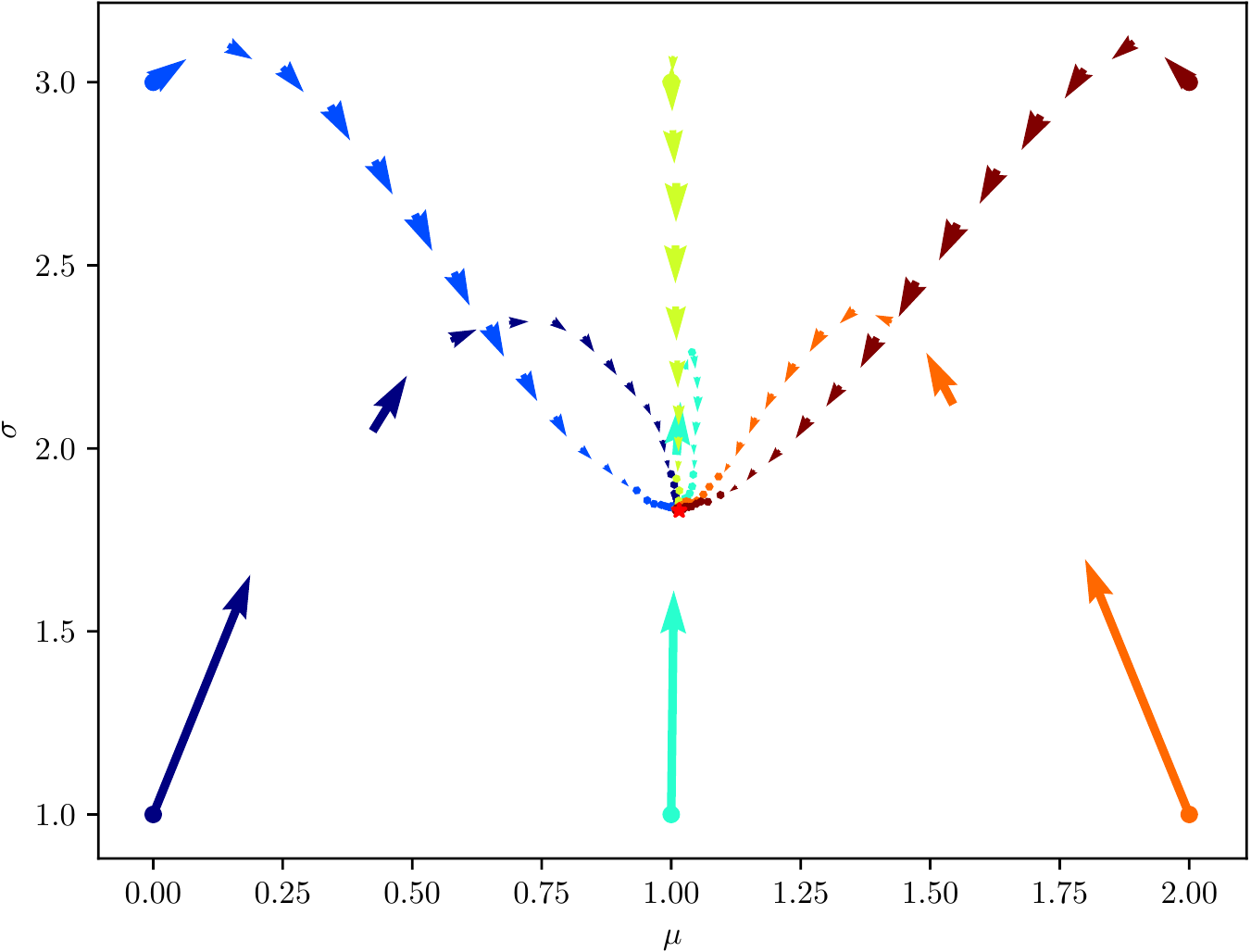}\label{fig:kl_grad}}
\subfigure[Reverse KL]{
\includegraphics[height=2.5in]{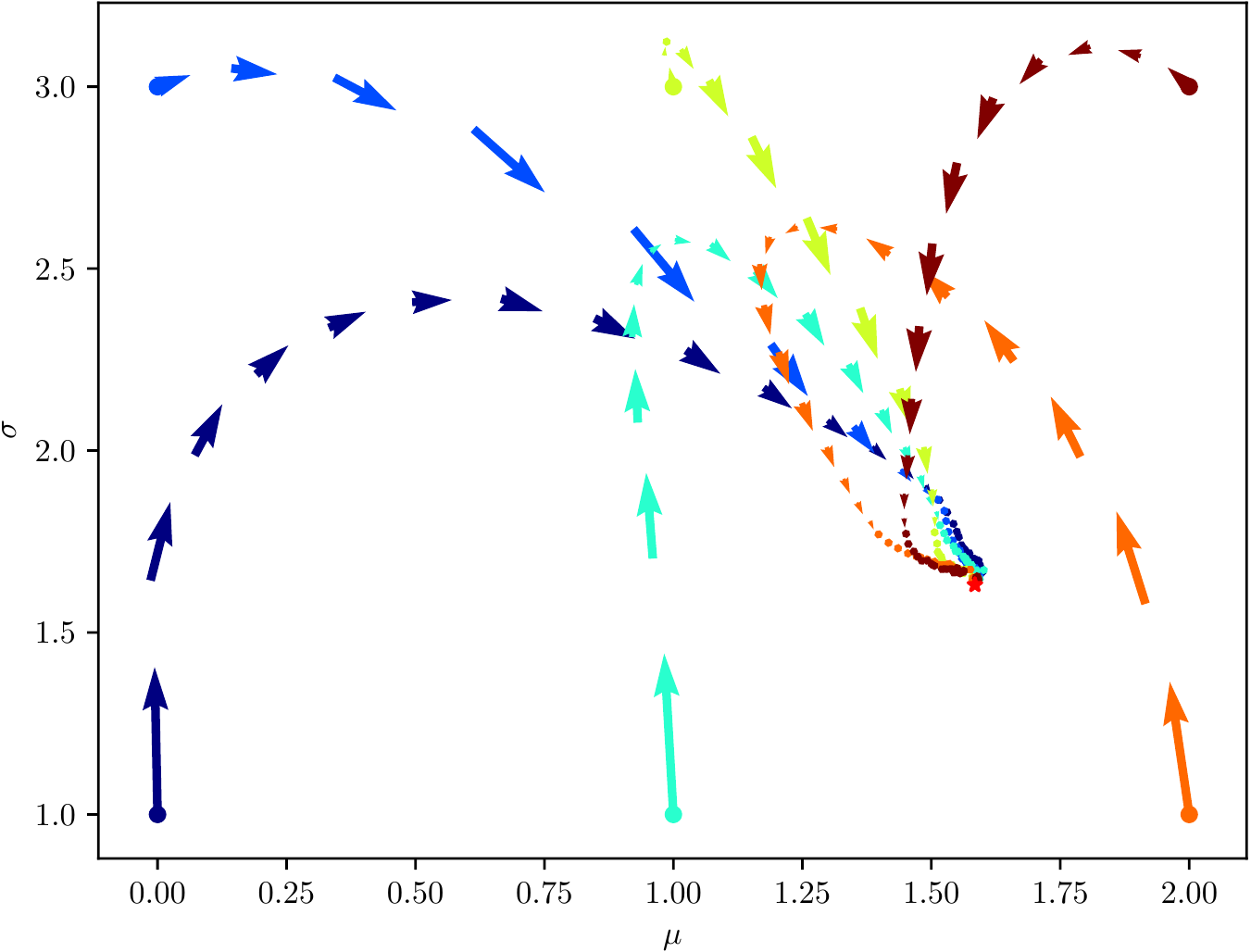}\label{fig:rev-kl_grad}}
\subfigure[Jensen-Shannon]{
\includegraphics[height=2.5in]{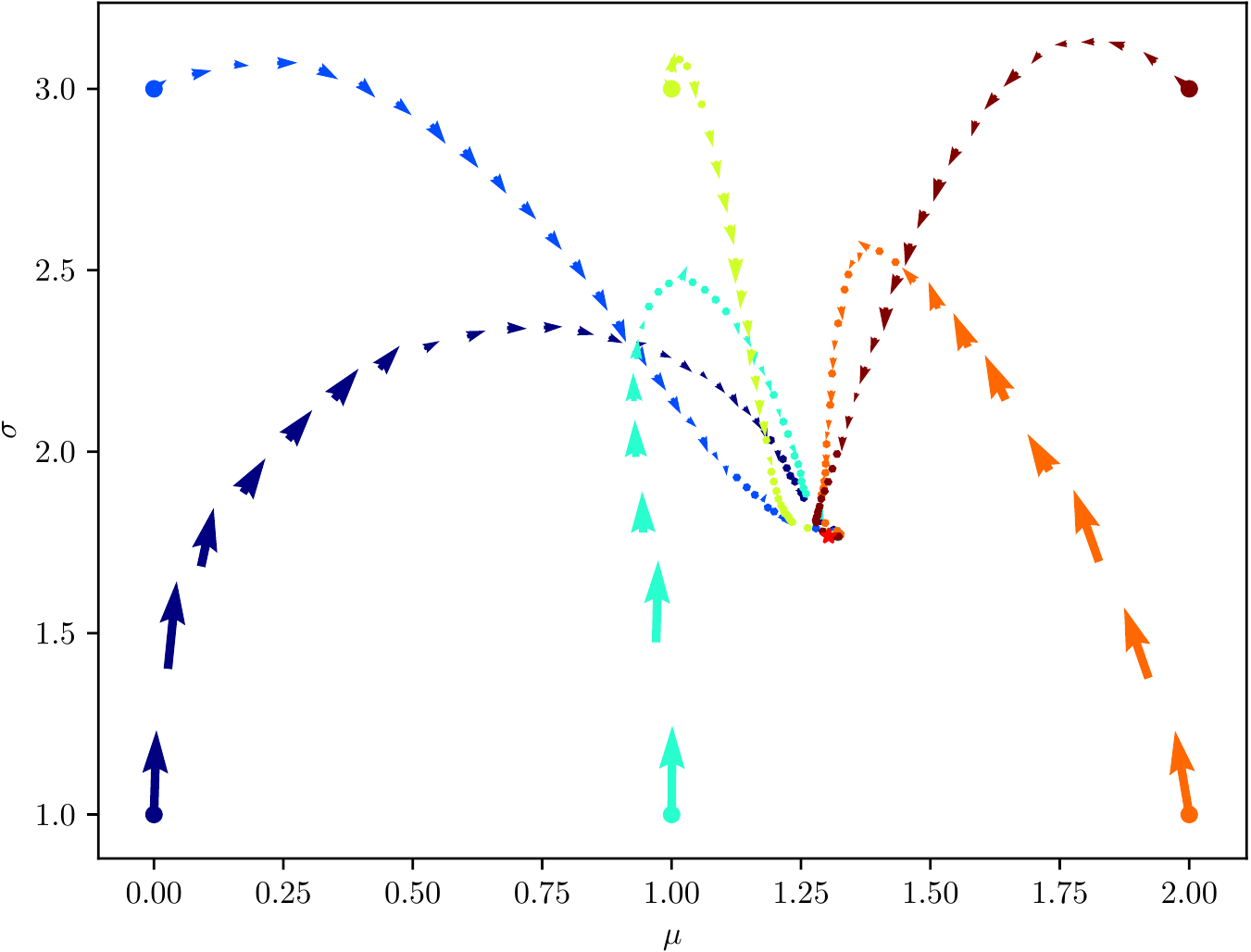}\label{fig:jensen-shannon_grad}}
\subfigure[Squared-Hellinger]{
\includegraphics[height=2.5in]{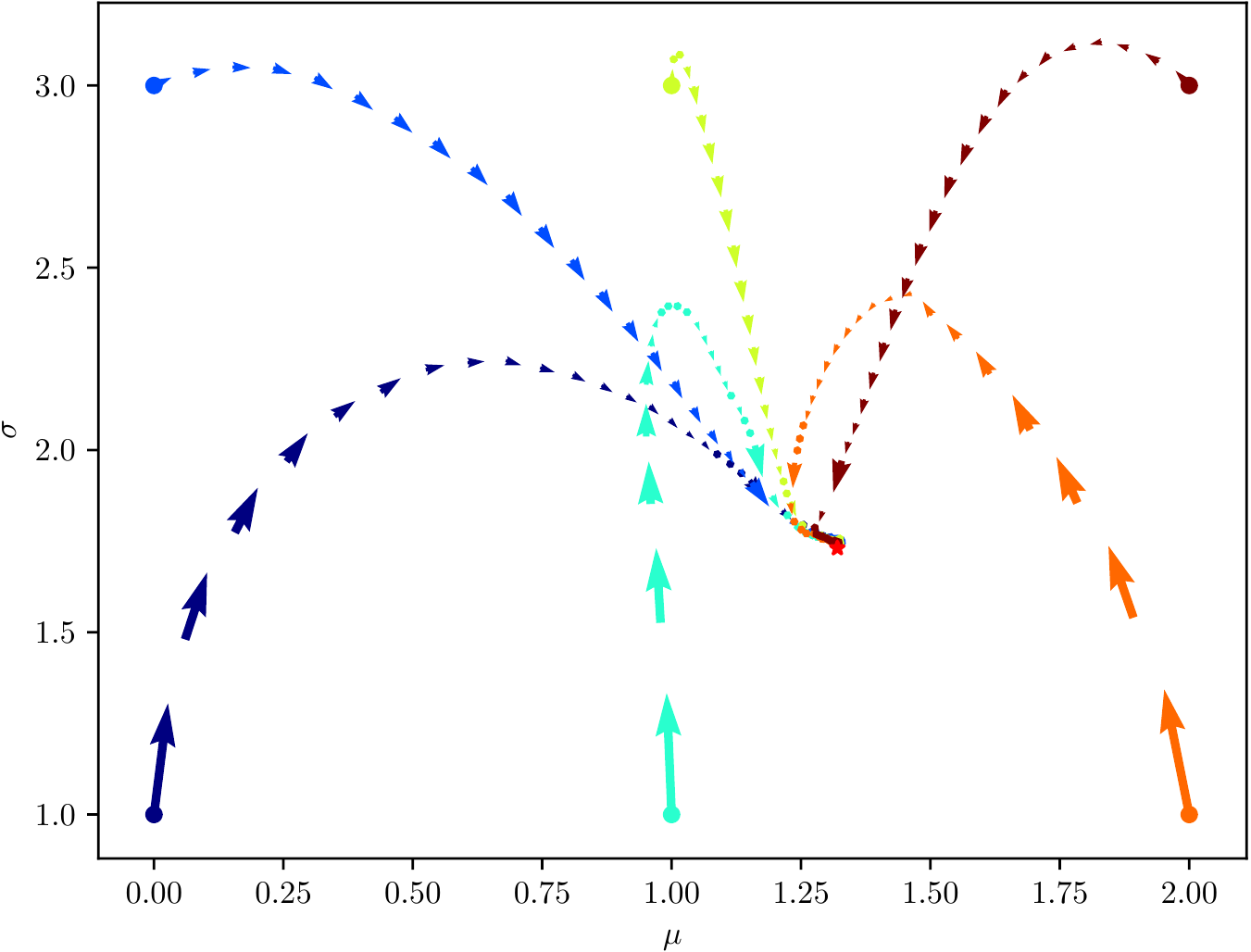}\label{fig:squared-hellinger_grad}}
\caption{Convergence of Single-Step $f$-EBM algorithm for different $f$-divergences. Optimization trajectories in different colors start from different initializations. The length and direction of the arrows represent the scale and the direction of the gradient at a certain point. The red stars that the trajectories converge to represent the desired optimal solutions under corresponding $f$-divergences.}
\label{fig:divergences_grad}
\end{figure*}
\FloatBarrier
\vfill

\newpage
\section{Additional Experimental Details for Modeling Natural Images}
\subsection{Samples from The Baseline Method in Section~\ref{sec:challenges}}\label{app:baseline-samples}
\begin{figure*}[h]
\centering
\subfigure[Reverse KL]{
\includegraphics[height=2in]{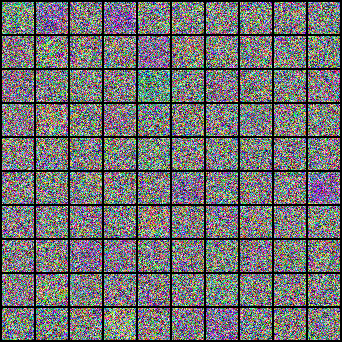}\label{fig:revkl_baseline}}
\subfigure[Jensen Shannon]{
\includegraphics[height=2in]{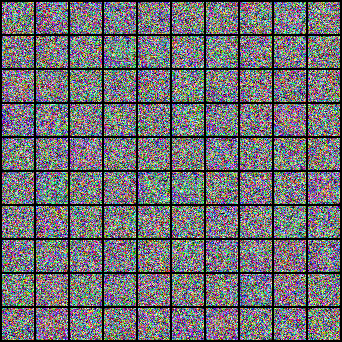}\label{fig:jensen-shannon_baseline}}
\subfigure[Squared Hellinger]{
\includegraphics[height=2in]{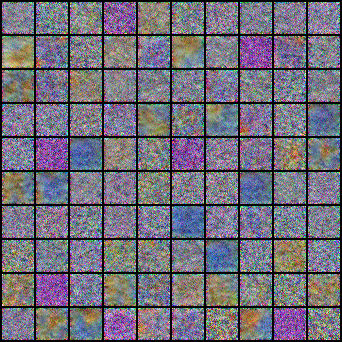}\label{fig:squared_hellinger_baseline}}
\vspace{-10pt}
\caption{Uncurated samples from EBMs trained by the baseline approach described in Section~\ref{sec:challenges} on CIFAR-10 dataset.}
\label{fig:divergences_baseline}
\end{figure*}

\subsection{Uncurated CIFAR-10 Samples from $f$-EBM}\label{app:cifar10-samples}
\afterpage{%
\begin{figure*}[htb]
\centering
\includegraphics[height=4in]{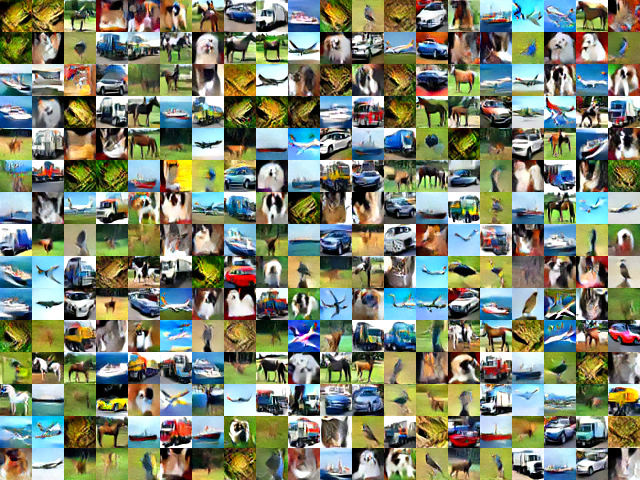}\label{fig:sampling_reverse_kl}
\vspace{-10pt}
\caption{Uncurated CIFAR-10 samples from $f$-EBM under the guidance of Reverse KL.}
\end{figure*}
\clearpage
}

\begin{figure*}[htb]
\centering
\includegraphics[height=4in]{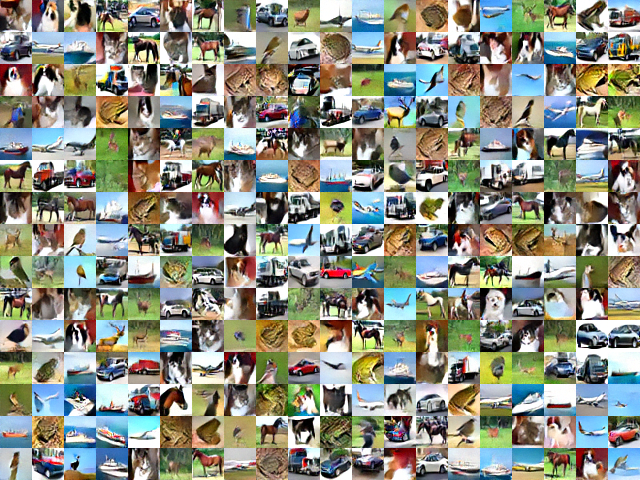}\label{fig:sampling_js}
\vspace{-10pt}
\caption{Uncurated CIFAR-10 samples from $f$-EBM under the guidance of Jensen Shannon.}
\end{figure*}
\begin{figure*}[htb]
\centering
\includegraphics[height=4in]{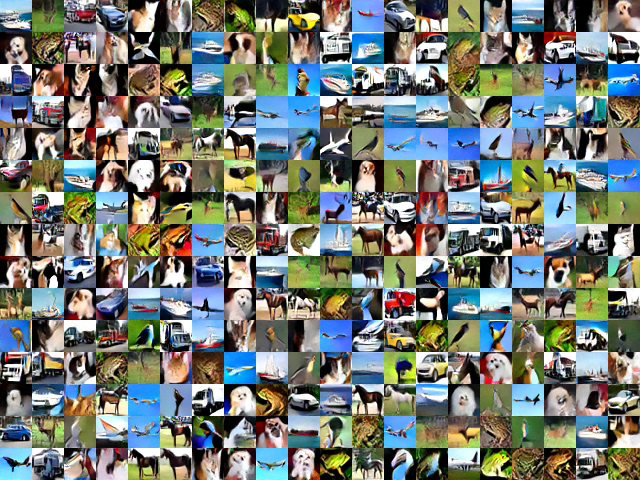}\label{fig:sampling_sh}
\vspace{-10pt}
\caption{Uncurated CIFAR-10 samples from $f$-EBM under the guidance of Squared Hellinger.}
\end{figure*}

\newpage
\subsection{Image Inpainting Results}\label{app:inpainting}
\begin{figure*}[h!]
\centering
\subfigure[Reverse KL]{
\includegraphics[height=2.4in]{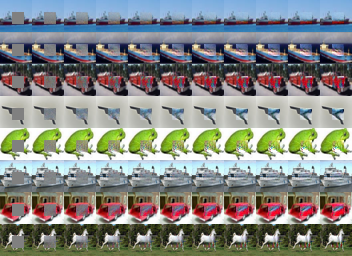}\label{fig:boxcorrupt_reverse_kl}}
~~~~~~
\subfigure[Jensen Shannon]{
\includegraphics[height=2.4in]{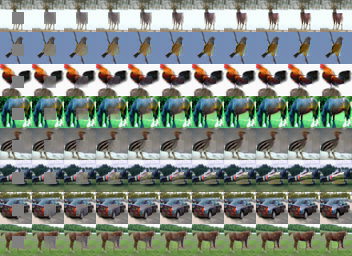}\label{fig:boxcorrupt_js}}
~~~~~~
\subfigure[Squared Hellinger]{
\includegraphics[height=2.4in]{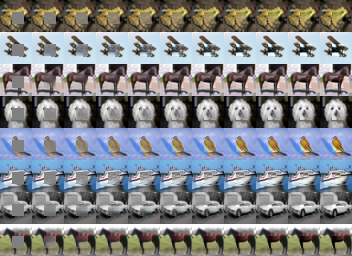}\label{fig:boxcorrupt_squared_hellinger}}
\vspace{-5pt}
\caption{Image inpainting results for $f$-EBM trained with different $f$-divergences on CIFAR-10. We use Langevin dynamics sampling to restore the images which are corrupted by empty boxes.}
\label{fig:image_inpainting}
\end{figure*}

\newpage
\subsection{Image Denoising Results}\label{app:denoising}
\begin{figure*}[h!]
\centering
\subfigure[Reverse KL]{
\includegraphics[height=7.6in]{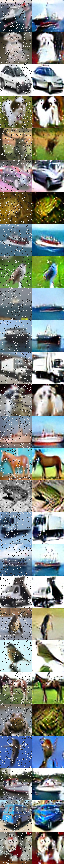}\label{fig:anticorrupt_reverse_kl}}
~~~~~~~~~~~~
\subfigure[Jensen Shannon]{
\includegraphics[height=7.6in]{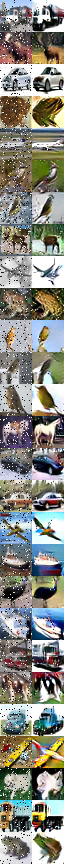}\label{fig:anticorrupt_js}}
~~~~~~~~~~~~
\subfigure[Squared Hellinger]{
\includegraphics[height=7.6in]{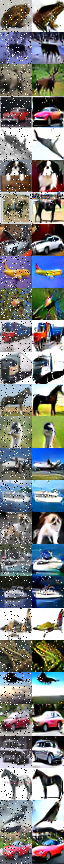}\label{fig:anticorrupt_new_squared_hellinger}}
\caption{Image denoising results for $f$-EBM trained with different $f$-divergences on CIFAR-10. We apply $10\%$ ``salt and pepper'' noise to the images in the test set and use Langevin dynamics sampling to restore the images.}
\label{fig:image_denoising}
\end{figure*}

\newpage
\subsection{Nearest Neighbor Images}\label{app:nearest-neighbor}
\begin{figure*}[h!]
\centering
\subfigure[Reverse KL]{
\includegraphics[height=2.4in]{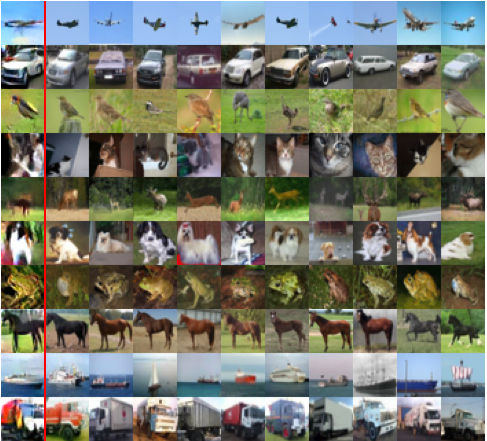}\label{fig:nearest_reverse_kl}}
\\
\subfigure[Jensen Shannon]{
\includegraphics[height=2.4in]{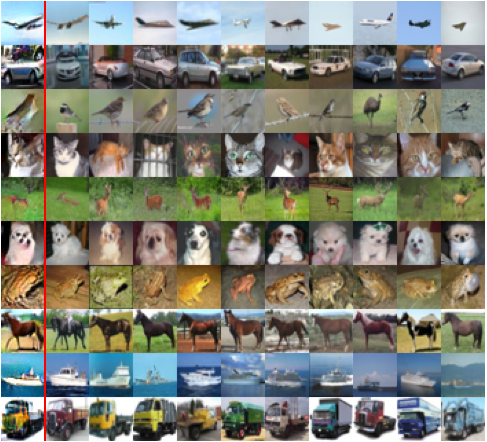}\label{fig:nearest_new_js}}
\\
\subfigure[Squared Hellinger]{
\includegraphics[height=2.4in]{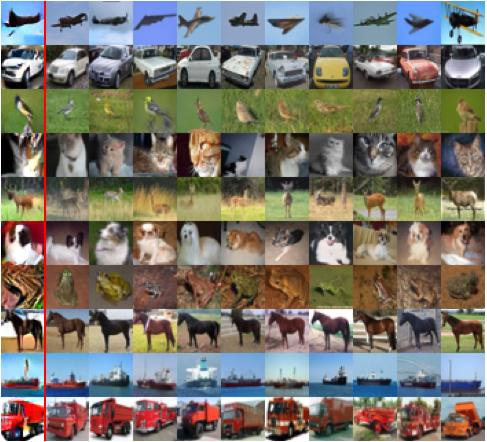}\label{fig:nearest_new_sh}}
\vspace{-5pt}
\caption{Nearest neighbor images according to $l_2$ distance between images. Different rows are for different classes. In each row, the leftmost image (\emph{i.e.}, on the left of the right vertical line) is generated by $f$-EBM, and the other images are nearest neighbors of the generated image in the training set.}
\label{fig:nearest}
\end{figure*}

\subsection{Intermediate Samples of Langevin Dynamics on CIFAR-10}\label{app:langevin-cifar10}
\vspace{10pt}

\begin{figure*}[h!]
\centering
\subfigure[Reverse KL]{
\includegraphics[height=1.2in]{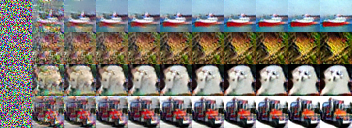}\label{fig:intermediate_sampling_reverse_kl}}
\\
\subfigure[Jensen Shannon]{
\includegraphics[height=1.2in]{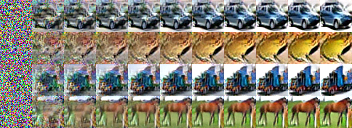}\label{fig:intermediate_sampling_js}}
\\
\subfigure[Squared Hellinger]{
\includegraphics[height=1.2in]{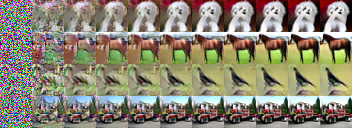}\label{fig:intermediate_sampling_sh}}
\vspace{-5pt}
\caption{Intermediate samples during Langevin dynamics sampling process for $f$-EBM.}
\label{fig:intermediate-samples}
\end{figure*}

\newpage
\subsection{Uncurated CelebA Samples from $f$-EBM}\label{app:celeba-samples}
\begin{figure*}[h!]
\centering
\subfigure[Reverse KL]{
\includegraphics[height=2.5in]{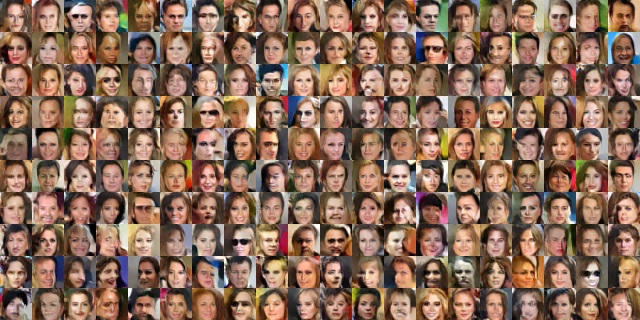}}
\\
\subfigure[Jensen Shannon]{
\includegraphics[height=2.5in]{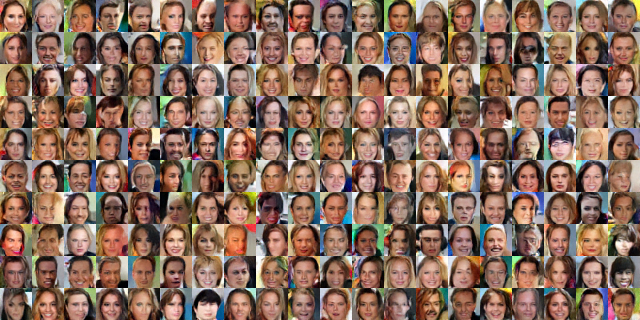}}
\\
\subfigure[Squared Hellinger]{
\includegraphics[height=2.5in]{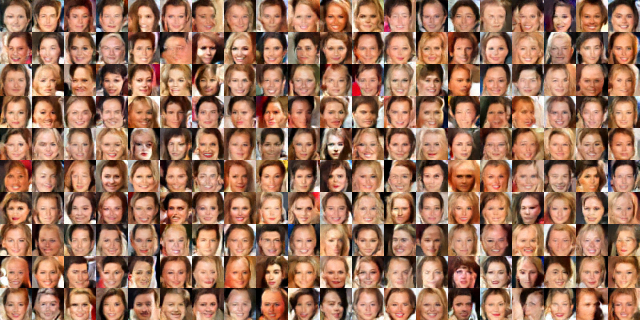}}
\vspace{-5pt}
\caption{Uncurated CelebA samples from $f$-EBM.}
\label{fig:celeba_samples}
\end{figure*}

\subsection{Intermediate Samples of Langevin Dynamics on CelebA}\label{app:langevin-celeba}
\begin{figure*}[h!]
\centering
\subfigure[Reverse KL]{
\includegraphics[height=2.5in]{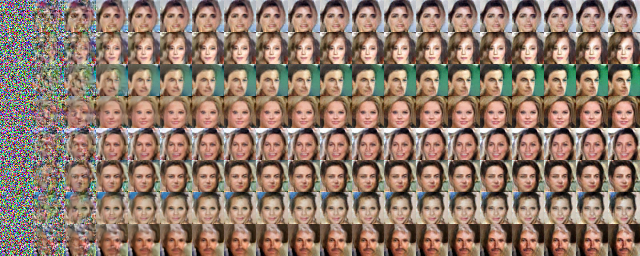}}
\\
\subfigure[Jensen Shannon]{
\includegraphics[height=2.5in]{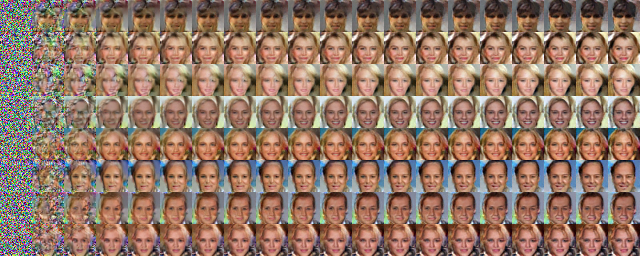}}
\\
\subfigure[Squared Hellinger]{
\includegraphics[height=2.5in]{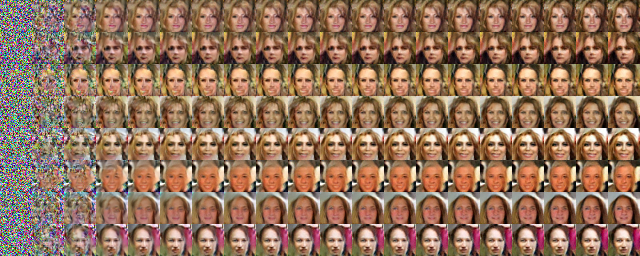}}
\vspace{-5pt}
\caption{Intermediate samples during Langevin dynamics sampling process for $f$-EBM.}
\label{fig:celeba_intermediate}
\end{figure*}

\newpage
\subsection{Architectures and Training Hyperparameters}\label{app:training-details}
\begin{wrapfigure}{r}{0.2\textwidth}
\begin{center}
\vspace{-18pt}
\includegraphics[width=\linewidth]{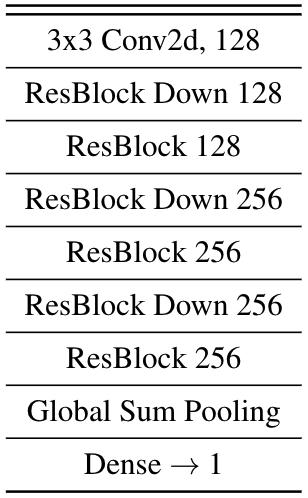}
\vspace{-25pt}
\caption{ResNet architecture for implementing the energy function and the variational function.} 
\label{fig:architecture}
\end{center}
\end{wrapfigure}

For all the experiments on natural images (conditional generation for CIFAR-10 and unconditional generation for CelebA), we use the residual network architecture \cite{he2016deep} in Figure~\ref{fig:architecture} (same as the conditional CIFAR-10 model in \cite{du2019implicit}) to implement both the energy function and the variational function of $f$-EBM. For the contrastive divergence baseline, we also use the same model architecture for fair comparisons. We note that the model architecture is highly relevant to the model performance and we leave the investigation of better architectures to future works.

For the CelebA dataset, we first center-crop the images to $140 \times 140$, then resize them to $32 \times 32$. For both CelebA and CIFAR-10, we rescale the images to $[0,1]$.
Following \cite{du2019implicit}, we apply spectral normalization and $L_2$ regularization (on the outputs of the models) with coefficient 1.0 to improve the stability. We use 60 steps Langevin dynamics together with a sample replay buffer of size 10000 to produce samples in the training phase. In each Langevin step, we use a step size of 10.0 and a random noise with standard deviation of 0.005. We use Adam optimizer with $\beta_1=0.0, \beta_2=0.999$ and learning rate of $3 \times 10^{-4}$ to optimize the parameters of both the energy function and the variational function. In each training iteration, we use a batch of 128 positive images and negative images. These training hyperparameters are used for both $f$-EBMs and the contrastive divergence baseline.

\subsection{Time Complexity and Convergence Speed}
In the experiments we use the same batch size, learning rate and MCMC step number as in CD \cite{du2019implicit}. Since we use single-step minimax optimization for both the energy function and the variational function, each iteration needs the same time for MCMC as \cite{du2019implicit} and twice time for stochastic gradient descent parameter updates. Overall, the time complexity per iteration of $f$-EBM is comparable to that of CD (within a factor of 2). The convergence speed of $f$-EBM is also similar to CD in terms of number of iterations needed (CD: 75K iterations; $f$-EBM with Jensen Shannon: 50K iterations; $f$-EBM with Squared Hellinger: 80K iterations; $f$-EBM with Reverse KL: 70K iterations; $f$-EBM with KL: 75K iterations).

\end{document}